\documentclass[preprint,11pt]{imsart}

\pdfoutput=1

\PassOptionsToPackage{numbers}{natbib}

\RequirePackage[OT1]{fontenc}
\RequirePackage{amsthm,amsmath}
\RequirePackage[numbers]{natbib}
\RequirePackage[colorlinks=true, citecolor=blue, filecolor=black,  urlcolor=blue]{hyperref}
\usepackage{hyperref}

\usepackage[utf8]{inputenc} %
\usepackage[T1]{fontenc}    %
\usepackage{url}            %
\usepackage{booktabs}       %
\usepackage{amsfonts}       %
\usepackage{amssymb}
\usepackage{nicefrac}       %
\usepackage{microtype}      %
\usepackage{graphicx}
\usepackage{epstopdf}
\usepackage{algorithm, algorithmic}
\usepackage{bm}
\usepackage{bbm}
\usepackage{xspace}
\usepackage{tikz}
\usepackage{standalone}
\usepackage{mathtools}
\usepackage{subcaption}
\usepackage[textsize=tiny,textwidth=0.8in]{todonotes}
\usepackage[normalem]{ulem}
\usepackage{amsopn}
\usepackage{cleveref}
\crefname{subsection}{subsection}{subsections}

\usepackage[centering,letterpaper,margin=1in]{geometry}
\usetikzlibrary{shapes.misc}
\tikzset{cross/.style={cross out, very thick, draw=black, minimum size=2*(#1-\pgflinewidth), inner sep=0pt, outer sep=0pt},
cross/.default={5pt}}

\ifpdf
  \DeclareGraphicsExtensions{.eps,.pdf,.png,.jpg}
\else
  \DeclareGraphicsExtensions{.eps}
\fi

\usepackage{enumitem}
\setlist[enumerate]{leftmargin=.5in}
\setlist[itemize]{leftmargin=.5in}

\newtheorem{theorem}{Theorem}
\newtheorem{proposition}[theorem]{Proposition}
\newtheorem{lemma}[theorem]{Lemma}
\newtheorem{corollary}[theorem]{Corollary}
\newtheorem{definition}[theorem]{Definition}
\newtheorem{example}[theorem]{Example}
\newtheorem{remarkx}[theorem]{Remark}
\newtheorem{dataassumption}[theorem]{Data Assumption}
\newtheorem{definitiono}[theorem]{Optimization Problem}
\numberwithin{theorem}{section}

\newenvironment{remark}
  {\pushQED{\qed}\remarkx}
  {\popQED\endremarkx}

\newcommand{\R}{\mathbb{R}}
\newcommand{\Loss}{\textsf{L}}
\newcommand{\Joss}{\textsf{J}}

\newcommand{\LCLIPN}{\Loss^{N}_{\mathsf{clip}}}
\newcommand{\LCONDN}{\Loss^{N}_{\mathsf{cond}}}
\newcommand{\LCOND}{\Loss_{\mathsf{cond}}}
\newcommand{\LCONDM}{\Loss_{\mathsf{cond,mmd}}}
\newcommand{\LJOINT}{\Loss_{\mathsf{joint}}}

\newcommand{\LJOINTM}{\Loss_{\mathsf{joint,mmd}}}
\newcommand{\JCOND}{\Joss_{\mathsf{cond}}}
\newcommand{\JCONDM}{\Joss_{\mathsf{cond,mmd}}}
\newcommand{\JCONDD}{\Joss_{\mathsf{cond,D}}}
\newcommand{\LCONDD}{\Loss_{\mathsf{cond,D}}}
\newcommand{\JJOINT}{\Joss_{\mathsf{joint}}}
\newcommand{\JJOINTD}{\Joss_{\mathsf{joint,D}}}
\newcommand{\tcn}{\theta^{N}_{\mathsf{clip}}}
\newcommand{\tcondn}{\theta^{N}_{\mathsf{cond}}}
\newcommand{\tcond}{\theta_{\mathsf{cond}}}
\newcommand{\tjoint}{\theta_{\mathsf{joint}}}
\newcommand{\tjointm}{\theta_{\mathsf{joint,mmd}}}
\newcommand{\JFINE}{\Joss_{\mathsf{fine}}}
\newcommand{\LFINE}{\Loss_{\mathsf{fine}}}
\newcommand{\tfine}{\vartheta_{\mathsf{fine}}}
\newcommand{\LMNIST}{\Loss_{\mathsf{mnist}}}

\newcommand{\Nt}{N_{\mathsf{test}}}

\newcommand{\munew}{\widehat{\mu}}
\newcommand{\phinew}{\widehat{\phi}}
\newcommand{\munef}{\widetilde{\mu}}
\newcommand{\phinef}{\widetilde{\phi}}
\newcommand{\mufinetuning}{\widehat{\mu}}
\newcommand{\xhat}{\widehat{\x}}
\newcommand{\yhat}{\widehat{\y}}
\newcommand{\nunew}{\widehat{\pmeas}}
\newcommand{\nunef}{\widetilde{\pmeas}}

\newcommand{\dkl}{\mathsf{D}_{\mathsf{kl}}}
\newcommand{\Div}{\mathsf{D}}
\newcommand{\Tr}{\textrm{Tr}}

\newcommand{\x}{u}
\newcommand{\y}{v}
\newcommand{\rr}{w}
\newcommand{\z}{z}
\newcommand{\g}{g}
\newcommand{\h}{h}
\newcommand{\normg}{\bar{\g}}

\newcommand{\X}{U}
\newcommand{\Y}{V}
\newcommand{\pmodel}{\rho}
\newcommand{\pmeas}{\nu}
\newcommand{\mean}{m}
\newcommand{\cov}{\mathcal{C}}
\newcommand{\cR}{\mathcal{W}}
\newcommand{\cU}{\mathcal{\X}}
\newcommand{\cV}{\mathcal{\Y}}
\newcommand{\cX}{\mathcal{\X}}
\newcommand{\cY}{\mathcal{\Y}}
\newcommand{\bT}{\mathbb{T}}
\newcommand{\bR}{\mathbb{R}}
\newcommand{\bS}{\mathbb{S}}

\newcommand{\dx}{n_\x}
\newcommand{\dy}{n_\y}
\newcommand{\de}{n_e}

\DeclareMathOperator*{\argmax}{arg\,max}
\DeclareMathOperator*{\argmin}{arg\,min}

\begin{document}

\begin{frontmatter}

\title{A Mathematical Perspective On Contrastive Learning}

\runtitle{Contrastive Learning}

\begin{aug}
\author[a1]{\fnms{Ricardo} \snm{Baptista}%
\ead[label=e1]{ricarsb@amazon.com}},
\author[a1,a2]{\fnms{Andrew} \snm{Stuart}%
\ead[label=e2]{andrxstu@amazon.com}},
\and
\author[a1]{\fnms{Son} \snm{Tran}%
\ead[label=e3]{sontran@amazon.com}}
\runauthor{Baptista, Stuart, Tran}
\address[a1]{Stores Foundational AI, Amazon, Palo Alto CA 94301 and Pasadena CA 91125\\
\printead*{e1}, \printead*{e2}, \printead*{e3}}
\address[a2]{Computing and Mathematical Sciences, California Institute of Technology, Pasadena CA 91125\\}
\end{aug}

\begin{abstract}
Multimodal contrastive learning is a methodology for linking different data modalities; the canonical example is linking image and text data. The methodology is typically framed as the identification of a set of encoders, one for each modality, that align representations within a common latent space. In this work, we focus on the bimodal setting and interpret contrastive learning as the optimization of (parameterized) encoders that define conditional probability distributions, for each modality conditioned on the other, consistent with the available data. This provides a framework for multimodal algorithms such as crossmodal retrieval, which identifies the mode of one of these conditional distributions, and crossmodal classification, which is similar to retrieval but includes a fine-tuning step to make it task specific.

The framework we adopt also gives rise to crossmodal generative models. This probabilistic perspective suggests two natural generalizations of contrastive learning: the introduction of novel probabilistic loss functions, and the use of alternative metrics for measuring alignment in the common latent space. We study these generalizations of the classical approach in the multivariate Gaussian setting. In this 
context we view the latent space identification as a low-rank matrix approximation problem. This allows us to characterize the capabilities of loss functions and alignment metrics to approximate natural statistics, such as conditional means and covariances; doing so yields novel variants on contrastive learning algorithms for specific mode-seeking and for generative tasks. The framework we introduce is also studied through numerical experiments on multivariate Gaussians, the labeled MNIST dataset, %
and on a data assimilation application arising in oceanography.

\end{abstract}


\end{frontmatter}

\section{Introduction} \label{sec:introduction} 

Contrastive learning is a computational methodology for identifying mechanisms by which different data modalities, derived from a common underlying reality, can communicate with one another; the canonical example is linking image and text data. The goal of this paper is to formulate the problem of contrastive learning in the language of probability measures and to use this framing both to shed light on existing algorithms and to suggest novel variants of them. The subject is illustrated with explicit theory for Gaussians;
by numerical experiments for Gaussians; by demonstrating how %
image classification fits within our general framework; and by a novel application of the
methodology to a data assimilation problem arising in the ocean sciences. 
In Subsection \ref{ssec:Context} we set our work in context, giving two illustrative examples in Subsection \ref{ssec:IE}.
Subsection \ref{ssec:CPO} describes our contributions and provides an overview of the paper. A literature review may be found in Subsection \ref{ssec:LR}.

\subsection{Context} \label{ssec:Context}
The recent advances in artificial intelligence have been driven to a large extent by rapidly increasing acquisition of data, by innovations in the design of novel function classes for approximation (architectures), by novel training algorithms (optimizers) to identify suitable candidates from these function classes, and by advances in computer hardware and software. A fundamental challenge in the use of data is to combine information from multiple sources, often known as modalities, such as language, audio, image and video,  that are not necessarily represented in the same spaces. Multimodal learning addresses this fundamental problem. A successful and widely adopted methodology in this arena is \emph{contrastive learning.} In the bimodal context, contrastive learning is a way of aligning two different data modalities. This paper focuses on the study of contrastive learning in the bimodal setting.

Consider the two different modalities as elements $u$ and $v$, from spaces  $\cU, \cV$.
Contrastive learning is based on observing data pairs from $\cU \times \cV$ and aims to find a common latent space in
which the pairs can be aligned. In contrast to supervised learning, which aims to find a map from $\cU$ to $\cV$ or vice versa, linking the pairs, contrastive learning focuses on learning about the conditional distributions on $u|v$ and on $v|u.$ A useful way of conceptualizing the problem is to consider
$u$ and $v$ to be distinct noisy and indirect measurements of an element $\rr$ in a third space $\cR$:
\begin{subequations}
    \begin{align}
        u & =f_u(\rr,\eta_u), \label{eq:I_ucond}\\
        v & =f_v(\rr,\eta_v). \label{eq:I_vcond}
    \end{align}
\end{subequations}
Here $f_u(\cdot,\eta_u)$ (respectively $f_v(\cdot,\eta_v)$) captures the measurement process leading to an element in $\cU$ (resp. $\cV$) and
$\eta_u$ (resp. $\eta_v$) represents the measurement noise that may enter this process. 
Assuming that $\rr$ is a random variable in $\cR$, this set-up implies a joint distribution $\mu(\x,\y)$ on $\mathcal{\X} \times \mathcal{\Y}.$
In the bimodal setting, contrastive learning is based on observing data pairs from this joint distribution and aims to find a common latent space in
which the conditional distributions for $u|v$ and $v|u$ can be readily computed. The space $\cR$ and the probability measure on it will play 
no explicit role in our developments of this subject, but it remains a useful conceptual underpinning of contrastive learning 
as a source for the implied nontrivial conditional structure existing between the two modes with joint distribution $\mu.$

\subsection{Illustrative Examples}
\label{ssec:IE}

\paragraph{Images and Text} Arguably, the canonical example of bimodal information is that of image-text pairs \cite{radford2021learning}. Allowing these two modalities to communicate is at the heart of problems such as document retrieval \cite{miech2020end, jia2021scaling}, text-guided image generation \cite{rombach2022high, nichol2021glide} and multimodal representation learning~\cite{chen2020simple, grill2020bootstrap, chen2021exploring}. To place this problem in our general framework, we consider the setting where $\x$ is a pixelated colour (RGB) image,  represented as a flattened vector in $\cU:=\R^{\dx}$ where $\dx=3p^2$, with $p$ being the number of pixels in each of the two image dimensions and the multiplier $3$ accounting for the three RGB channels.
On the other hand text may be represented as a sequence over a finite alphabet $\mathfrak{A}$. Choice of the alphabet is critical to empirical success
and byte-pair encoding \cite{gage1994new} is commonly adopted. In practice a finite length $N$ is imposed on the sequence and so $\cV=\{f: [[1,N]] \to \mathfrak{A}\}.$\footnote{We employ the notation $[[{1,N]]}:=\{1, \cdots, N\}.$}
If we view pixelated image $u \in \cU$ and text $v \in \cV$ as, respectively, photographic and textual descriptions of a common real scene $\rr \in \cR$,
then we expect an implied joint distribution on image-text pairs in $\cU \times \cV$, with non-trivial conditional dependencies.

This set-up can be generalized in numerous directions. Perhaps most important is to highlight that multimodal variants, with more than two modalities, arise naturally and can be formulated, analyzed and studied similarly to the bimodal setting in this paper. In particular we can consider, for example, video and audio data, derived from the same common real scene. Or we have may have access to %
textual data not present in the underlying real scene, but that is implied by or complementary to it.

\paragraph{Eulerian and Lagrangian Visualization of Fluid Flow}

A common problem in oceanography is to simultaneously use direct observation of ocean currents (the \emph{Eulerian} picture \cite{bennett1992inverse}) and indirect observations through the transport of objects moving in those currents (the \emph{Lagrangian} picture \cite{ide2002lagrangian}); see \cite{cotter2009bayesian}. It is therefore of interest to ask how to align these two different modalities. The problem of recovering Eulerian velocity fields from Lagrangian information is often referred to as \emph{Lagranian data assimilation} \cite{kuznetsov2003method}.

To formulate the problem as one in the framework of contrastive learning we consider an idealized setting of flow in a periodic geometry.
Let $\bT^d$ denote the $d-$dimensional torus, and consider a velocity field 
$\rr \in \cR:=C^1(\bT^d,\bR^d).$
Eulerian observations of an element in $\cR$ are defined by $f_\x(\cdot, \eta_\x): \cR \to \bR^{d \times J_\x}$ where
$$f_\x(\rr,\eta_u)=\{\rr(x_j)+\eta_u^{j}\}_{j=1}^{J_\x},$$ 
defined through a specified set of observation points $x_j \in \bT^d,$ $j=1,\cdots, J_u,$ and subject to i.i.d.\thinspace additive noise $\eta_\x^{j}.$
Lagrangian observations are defined by considering trajectories $z \in C^1([0,T]; \bR^d)$ governed by the ordinary differential equation
$$\dot{z}=\rr(z), \; z(0)=z_0,$$
and defining $f_v(\cdot, \eta_u): \cR \to \bR^{d \times J_\y}$ by
$$f_\y(\rr,\eta_v)=\{z(t_j)+\eta_v^{j}\}_{j=1}^{J_\y};$$
thus, map $f_\y$ is defined through a specified set of observation times $t_j \in [0,T],$ $j=1,\cdots, J_\y$ and i.i.d.\thinspace additive noise $\eta_\y^{j}.$
The probability measure on $\cR$ can be defined, for example, by a Gaussian random field and the additive noises $\eta_u, \eta_v$ can
also be chosen as Gaussian. In this Gaussian setting $\x \in \cU = \R^{d \times J_\x}$ given by (\ref{eq:I_ucond}) is also Gaussian; and it is correlated to non-Gaussian variable $\y \in \cV = \R^{d \times J_\x}$ given by
(\ref{eq:I_vcond}). %
This creates a joint distribution $\mu(\x,\y)$ on the space $\cU \times \cV$ with non-trivial conditional structure.

This set-up can be generalized in a number of ways, including observing multiple Lagrangian trajectories,
making the noise structure more complex, for example by adding Brownian noise to the equations for trajectory $z$,
and by considering time-dependent velocity fields. Indeed, in Subsection \ref{ssec:fluid_flow}, we will consider an example
in which the velocity field is time-dependent and $f_\x$ encodes linear functionals, different from pointwise evaluations, of the velocity fields.

\subsection{Contributions and Paper Overview}
\label{ssec:CPO}

We focus on the bimodal setting in this paper, but the reader will readily generalize to multimodal settings with three or more modalities.
We make the following contributions to our overarching goal, namely the formulation of contrastive learning in a mathematical framework, shedding light on existing algorithms, enabling analysis of the algorithm in the Gaussian setting, and suggesting novel variants of those algorithms:
\begin{enumerate}[label=(C\arabic*)]    
    \item We formulate contrastive learning as determination of an underlying joint distribution on the two data modalities, defined through a change of measure, a \emph{tilting}, from the product of the marginals of each modality.
    \item We introduce new \emph{probabilistic loss functions}, including matching the joint, either conditional, or the sum of two conditionals; standard contrastive learning corresponds to the last case.
     \item We introduce new tiltings which subsume the standard contrastive learning case, based on cosine similarity, as a special case.
    \item We analyze the resulting classes of new contrastive learning methodologies in the Gaussian setting, formulating latent space identification in terms of low-rank approximation; we shed light on the capabilities of different approaches in terms of their ability to define point estimators and as generative models.
    \item We demonstrate that the basic contrastive learning methodology can be applied to problems arising in science and engineering, such as
    Lagrangian data assimilation; and we show that the generalized contrastive learning methodology applies to data science applications including retrieval and MNIST digit classification. 
\end{enumerate}

In \Cref{sec:CL} we define contrastive learning in the bimodal setting, and consider the limit of infinite data, addressing contribution (C1). \Cref{sec:PP} adopts a probabilistic interpretation of the bimodal contrastive learning problem; we introduce generalizations
of the standard problem, both through the form of tilting and the form of loss (i.e., objective) function  defining the learning problem, 
addressing contributions (C2) and (C3). \Cref{sec:RAC} formulates retrieval and classification using the framework developed
in the two preceding sections, laying the groundwork for later numerical experiments, illustrating parts of contribution (C5). In \Cref{sec:GS} we analyze this generalized class of contrastive learning problems in the Gaussian setting, shedding light
on the capabilities of different approaches---contribution (C4).
The supplementary meterial in Section \ref{sec:experiments} is devoted to contribution (C5): we provide 
numerical experiments which: (i) validate the Gaussian theory developed in this paper; (ii) which showcase, in the context of image-text crossmodal problems arising from the MNIST dataset, the generalized methodology developed in this paper;  (iii) show that the Lagrangian data assimilation problem can be solved
in a purely data-driven fashion using the cross-modal approaches studied in this paper.

\subsection{Literature Review}
\label{ssec:LR}

Contrastive learning is often referred to as \emph{self-supervised learning}, distinguishing it from supervised learning~\cite{bengio2017deep} and from unsupervised learning~\cite{von2007tutorial}. 
Contrastive self-supervised learning approaches have enabled the comparison of data modalities without explicitly labeled data, but using only pairs of related samples. In the context of language-image understanding,~\cite{jia2021scaling, radford2021learning, ramesh2022hierarchical} showed that contrastive learning approaches can learn representations for visual and text data obtained from the internet. Moreover, the pretrained representations are immediately useful for downstream tasks, such as image retrieval and classification, without additional model training. Recently, these pretrained models have been applied directly, without task-specific training data, for various applications, including embedding text prompts for image generation in diffusion models~\cite{rombach2022high} and to embed text and/or image inputs in multimodal language models~\cite{liu2024visual}. Beyond their use in language, contrastive learning is applied in computer vision tasks based on other modalities including the alignment of text and video~\cite{fang2021clip2video, luo2022clip4clip}, text and audio~\cite{elizalde2023clap}, and 3D scene generation~\cite{wang2022clip}. It has also proven useful for image classification, object detection~\cite{joseph2021open}, or semantic segmentation~\cite{kirillov2023segment}. Scaling studies have investigated the generalization behavior of contrastive models with increasing numbers of parameters~\cite{Cherti_2023_CVPR}. While contrastive learning aims to learn low-dimensional encodings, these can also be used to learn a mapping from the embedding space back to the original data. One such approach was considered in~\cite{yu2022coca} to do zero-shot image-captioning using CLIP.

Several objectives for contrastive representation learning have been proposed to match related samples, often referred to as positive samples, and to bring unrelated samples, often referred to as negative samples, farther apart in the latent space. These include objectives based on the Euclidean norm between the mapped embeddings~\cite{chopra2005learning}, or a metric between the sample covariance of sample pairs~\cite{zbontar2021barlow}. 
Several modifications have been proposed to the objective function in CLIP to improve computational efficiency: one such is SigClip~\cite{zhai2023sigmoid}, which replaces the softmax with a sigmoid layer and doesn't need two passes over the data to evaluate the loss function; another is  ClipLITE that only takes one pass over the training data to compute negative alignments. When supervised datasets are not available, data augmentation approaches generate multiple representations of the same image. For example Barlow twins~\cite{zbontar2021barlow}, `Bootstrap Your Own Latent'~\cite{grill2020bootstrap}, and SimCLR uses distorted representations of the same image. Rather than starting from a joint distribution of data from two modalities, self-supervised learning starts an unlabeled dataset from one distribution and does data augmentation to generate new samples. Recently, other generalizations of the loss for contrastive learning include learning the tilting from the recovery of a cost matrix using inverse optimal transport~\cite{shi2024ot}. Historically, InfoNCE~\cite{wu2020mutual} also showed how to identify embeddings by maximizing mutual information (MI), which upper bounds the contrastive loss underlying CLIP. However,~\cite{tschannen2019mutual} showed that maximizing MI directly is not correlated with downstream predictive accuracy when the learned embeddings are used for classification. The paper \cite{yue2023hyperbolic} also used a new alignment metric motivated by hyperbolic geometry that allows encoders to map into a normalized sphere, rather than only on the surface. 

From a theoretical perspective, previous analysis has studied the optimizers for  contrastive self-supervised learning problems. In particular,~\cite{ziyinshapes} studied loss landscapes and showed in linear-Gaussian settings that the learned representations may collapse to only capture low-dimensional subspaces. Later,~\cite{chen2020simple} showed that it is necessary to use large batch sizes to learn useful representations and prevent collapse. Recent work has shown that the choice of contrastive loss is crucial to optimize both alignment and uniformity of the resulting distribution of normalized features on the hypersphere. In terms of generalization,~\cite{haochen2021provable} proved that contrastive and self-supervised spectral losses can lead to generalization  guarantees. Concerning statistical guarantees, \cite{chen2022we} considers the bias in self-supervised learning from small mini-batches and proposes an augmented variable technique that is unbiased.

In this work our analysis is focused on the solution of contrastive learning in the Gaussian setting. Here, the standard contrastive learning methodology %
corresponds to matching conditionals of a joint multivariate Gaussian distribution. In the context of Bayesian inverse problems, these Gaussian conditional approximations have been studied in~\cite{spantini2015, hangcheng2024}. 

\subsection{Notation} \label{ssec:notation}

The formulation, analysis and numerical experiments in this paper are grounded in the metric space of probability measures 
on product space $\cU \otimes \cV.$ We will denote by $\mu$
the measure on this space from which paired data points, the basis of training, are derived; and we will denote by $\nu(\cdot; \theta)$
the family of probability models, on the same space, from which a learned approximation of $\mu$ is to be found. We will denote the $\x$ marginal and $\x|\y$ conditional measures of $\mu$ as $\mu_{\x}$ and $\mu_{\x|\y}$, respectively, with analogous notation when the roles of $\x$ and $\y$ are swapped, and when applied to measure $\nu$ rather than $\mu.$ We will also consider various objective functions, constructed to measure closeness of $\nu$, parameterized by $\theta$, to $\mu$, namely
$\JCOND$,
$\JCONDD$,
$\JJOINT$ and
$\JJOINTD$. Here $\Div$ refers to a chosen divergence; when $\Div$ is not specified, it is considered to be the Kullback-Leibler (KL) divergence. Subtracting constants independent of $\theta$ from these objectives, and in two cases making the specific choice of divergence $\Div$ to be the maximum mean discrepancy rather than the KL divergence, leads to loss functions
$\LCOND$,
$\LCONDM$,
$\LJOINT$ and
$\LJOINTM$.
The preceding objectives and loss functions are all defined in the population limit (infinite data). Our starting point, however, is the
standard contrastive loss function $\LCLIPN$, where $N$ denotes the size of the empirical data set. This is related, by a $\theta-$independent, but $N-$dependent, constant to $\LCONDN$. And $\LCONDN$ has a well-defined continuum limit, leading to $\LCOND$ and thence to $\JCOND$ and the zoo of population level objectives and losses listed above. Finally we will discuss fine-tuning of classification methods, leading to
$\LFINE$ and $\LFINE^K.$

We let  $|\cdot|$ and $\langle \cdot, \cdot \rangle$ denote the standard Euclidean norm and inner-product on $\R^{\de}$.
Given positive definite matrix $A$, we define weighted version of the standard Euclidean inner-product
by $\langle w, w' \rangle_{A}=\langle w, A^{-1} w' \rangle$; the induced weighted Euclidean norm is
denoted by $|\cdot|_{A}.$ We will denote by $A_r$ the optimal rank-r approximation of the matrix $A$ with respect to the Frobenius-norm---the low-rank truncation of the singular value decomposition of $A$. Lastly, we highlight the previously deployed notation
$[[{1,N]]}:=\{1, \cdots, N\}.$

\section{Contrastive Learning}
\label{sec:CL}

Contrastive methodologies learn to associate two random variables $\x \in \mathcal{\X},\y \in \mathcal{\Y}$, living in typically different spaces $\mathcal{\X},\mathcal{\Y}$. We assume that both spaces are equipped with a sigma-algebra, allowing probability measures to be
defined on them and on their product space. The methodology proceeds by learning two encoders $\g_\x \colon \mathcal{\X} \rightarrow \R^{\de}$ and $\g_\y \colon \mathcal{\Y} \rightarrow \R^{\de}$ that map to a common latent space of user-specified dimension $\de$; typically $\de$ is much smaller than the size of $\mathcal{\X}$ and $\mathcal{\Y}$. The training of the encoders is governed by minimizing a contrastive learning objective that rewards data pairs defined by the joint distribution of the two random variables and penalizes unrelated data pairs.

In standard implementations of contrastive learning, the outputs of the encoders are normalized by the $L^2$ norm of their output so that only the alignment of the outputs is compared, rather than their magnitude. Thus, the encoders used to align the two modalities are
\begin{equation}
\label{eq:normalized}
\normg_\x(\x) = \frac{\g_\x(\x)}{\left|\g_\x(\x)\right|}, \qquad \normg_\y(\y) = \frac{\g_\y(\y)}{\left|\g_\y(\y)\right|},
\end{equation}
where $|\cdot|$ denotes the Euclidean norm on $\R^{\de}$ so that $\normg_\x\colon \mathcal{\X} \to \bS^{\de-1}$ and $\normg_\y\colon \mathcal{\Y} \to \bS^{\de-1}$.
For example, the encoders may be constructed using a text or vision transformer for text and image data, respectively. Moreover, they may involve pretrained models that are composed with projections (e.g., linear maps), which are then learned so as to embed their outputs into a space of the same common dimension where they are aligned. 
We make the following standard assumption about the data used to train encoders that reflect the desired alignment.

\begin{dataassumption} \label{as:paired_data} The available data comprises pairs $\{(\x^i,\y^i)\}_{i=1}^N$ drawn i.i.d.\thinspace from a joint distribution $\mu(\x,\y)$ on $\mathcal{\X} \times \mathcal{\Y}.$
\end{dataassumption}

The aim of contrastive learning is to maximize the alignment between related data pairs, in the common latent space, and to minimize 
the alignment between unrelated pairs. The related pairs are draws from the joint distribution $\mu(\x,\y)$; unrelated pairs may be 
considered to be drawn independently from the two marginal distributions for $\x$ and $\y$, which we denote by $\mu_\x(\x)$ 
and $\mu_\y(\y)$, respectively. 
Given $N$ samples from the joint distribution, as in \Cref{as:paired_data}, we have the following empirical measures for the joint and  marginal  distributions
\begin{equation}
\label{eq:empirical}   
\mu^N \coloneqq \frac{1}{N}\sum_{\ell=1}^N \delta_{(\x^\ell,\y^\ell)}, \quad
\mu_\x^N \coloneqq \frac{1}{N}\sum_{\ell=1}^N \delta_{\x^\ell}, \quad \mu_\y^N \coloneqq \frac{1}{N}\sum_{\ell=1}^N \delta_{\y^\ell}. 
\end{equation}
We also define
\begin{equation}
\label{eq:empirical2}   
\mathcal{\X}^N:=\{\x^i\}_{i=1}^N \subset \mathcal {\X}, \quad \mathcal{\Y}^N:=\{\y^i\}_{i=1}^N \subset \mathcal {\Y}.
\end{equation}

In Subsection \ref{sec:learning} we describe the standard contrastive learning objective defined by the bimodal data given in~\Cref{as:paired_data}.
Subsection \ref{ssec:pop} describes a population learning objective, obtained in the limit $N \to \infty.$

\subsection{Discrete Learning Problem} \label{sec:learning}

Having chosen an embedding dimension $\de$, we introduce parameter $\theta = (\theta_\x,\theta_\y) \in \R^p$ defining the
parameterized encoder models $\normg_\x(\cdot;\theta_\x)$ and $\normg_\y(\cdot;\theta_\y)$, respectively. 
We define the conditional probability of data $\x \in \mathcal{\X}^N$, given $\y^i \in \mathcal{\Y}^N$, by
\begin{subequations} \label{eq:softmax_probability0}
\begin{align}
p(\x|\y^i;\theta)  \coloneqq  & \frac{\exp\bigl(\langle \normg_\x(\x;\theta_\x), \normg_\y(\y^i;\theta_\y) \rangle /\tau)}{\sum_{j=1}^N \exp\bigl(\langle \normg_\x(\x^j;\theta_\x), \normg_\y(\y^i;\theta_\y) \rangle /\tau\bigr)},%
\end{align}
\end{subequations}
where $\tau > 0$ is a hyperparameter, often referred to as a \emph{temperature}.
By symmetry, we may also define the conditional probability of data $\y \in \mathcal{\Y}^N$, given $\x^k \in \mathcal{\X}^N$, by
\begin{subequations} \label{eq:softmax_probability2}
\begin{align}
p(\y|\x^k;\theta) \coloneqq  & \frac{\exp\bigl(\langle \normg_\x(\x^k;\theta_\x), \normg_\y(\y;\theta_\y) \rangle /\tau)}{\sum_{j=1}^N \exp\bigl(\langle \normg_\x(\x^k;\theta_\x), \normg_\y(\y^j;\theta_\y) \rangle /\tau\bigr)},%
\end{align}
\end{subequations}

\begin{remark}
\label{rem:softmax}    
The specific construction of probabilities using exponentiation and normalization is often referred to as the $\texttt{softmax}$ operation.
\end{remark}

\begin{remark}
\label{r:N}
We note that $Np(\cdot|\y^i;\theta)$ may also be viewed as a density with respect to the discrete measure $\mu_\x^N$ on $ \mathcal{\X}^N$: 
reweighting $\mu_\x^N$ by $Np(\cdot|\y^i;\theta)$ gives another probability measure on $\mathcal{\X}^N.$ Likewise $Np(\cdot|\x^k;\theta)$ may be viewed as a density with respect to the discrete measure $\mu_\y^N$ on $ \mathcal{\Y}^N$.
\end{remark}

The idea behind contrastive learning is to choose parameter $\theta$ so that, when summed over the data set, both $p(\x^i|\y^i;\theta)$ and $p(\y^i|\x^i;\theta)$ are maximized.  The log of the conditional probabilities, also known as \emph{logits}, are commonly used to define an objective function for contrastive learning that maximizes the likelihood of the paired data $\{(\x^i,\y^i)\}_{i=1}^N$, defined by \Cref{as:paired_data}, under the models for the conditional probabilities of $\x|\y$ and $\y|\x$ defined by \eqref{eq:softmax_probability0} and \eqref{eq:softmax_probability2}. This line of reasoning leads 
to the \emph{cross-entropy loss}
\begin{equation}
\label{eq:LN}
\LCLIPN(\theta) \coloneqq -\frac{1}{2}\left[\frac{1}{N} \sum_{i=1}^N \log p(\x^i|\y^i;\theta) + \log p(\y^i|\x^i;\theta) \right].
\end{equation}
The optimal parameter $\tcn$ minimizes this loss. Rewriting in terms of expectations we have the following optimization problem:
\begin{definitiono}
\label{do:LN2}
\begin{subequations}
\label{eq:LN2}
\begin{align}
\LCLIPN(\theta) & = -\frac{1}{2}\mathbb{E}_{(\x,\y) \sim \mu^N} \Bigl[\log p(\x|\y;\theta) + \log p(\y|\x;\theta) \Bigr],\\
\tcn & \in \argmin_{\theta \in \R^p} \, \LCLIPN(\theta).
\end{align}
\end{subequations}
\end{definitiono}

We now make several remarks which help in the interpretation of the  contrastive learning optimization problem \eqref{eq:LN2} for $\theta.$
We first note that the definition of the conditional probabilities uses \emph{cosine similarity}:

\begin{remark} For normalized encoders, i.e., $|\normg_\x(\x)| = |\normg_\y(\y)| = 1,$ the inner product satisfies
$$\langle \normg_\x(\x), \normg_\y(\y) \rangle = |\normg_\x(\x)||\normg_\y(\y)|\cos(\alpha_{\x\y}) = \cos(\alpha_{\x\y}) \in [-1,1],$$ where $\alpha_{\x\y}$ is the angle between the embedding vectors. Hence, the inner product is known as a cosine similarity.
 The cosine similarity is invariant to rescaling of $g_\x$ and $g_\y$ since this does not affect  $\bigl(\normg_\x(\x), \normg_\y(\y)\bigr)$. In addition, the cosine similarity is invariant to joint rotations of the two encoders: for any orthonormal matrix $U \in \R^{\de \times \de}$ it follows that
 $$\langle U \normg_\x(\x), U \normg_\y(\y) \rangle = \langle \normg_\x(\x), U^\top U \normg_\y(\y) \rangle = \langle \normg_\x(\x), \normg_\y(\y) \rangle.$$ 
\end{remark}

We have mentioned previously that contrastive learning has the interpretation of choosing encoders that reward similarity for paired data points
from the two modalities and penalizes unpaired data. The following remark makes this explicit.

\begin{remark} \label{r:wgwan}
Recall that, from \Cref{as:paired_data}, $\{(\x^i,\y^i)\}_{i=1}^N$ denote i.i.d. samples from $\mu.$
Now consider the unpaired data set $\{(\x^i,\y^j)\}_{i=1, j=1, i \ne j}^N$; this may be viewed as constituting i.i.d. samples
from $\mu_\x \otimes \mu_\y$. Now observe that, using (\ref{eq:softmax_probability0}b) and  (\ref{eq:softmax_probability2}b),
\begin{align*}
-\LCLIPN(\theta) & =  \frac{1}{2}\left[\frac{1}{N} \sum_{i=1}^N \log p(\x^i|\y^i;\theta) + \log p(\y^i|\x^i;\theta) \right]\\
&= \frac{1}{N} \sum_{i=1}^N  \langle \normg_\x(\x^i;\theta_\x), \normg_y(\y^i;\theta_y) \rangle/\tau  - \log(N)\\
&\qquad - \frac{1}{2N}\sum_{i=1}^N\log\left(\frac{1}{N}\sum_{j=1}^N \exp\Bigl(\langle \normg_\x(\x^j;\theta_\x),  \normg_\y(\y^i;\theta_\y)\rangle/\tau\Bigr) \right) \\
&\qquad \qquad - \frac{1}{2N}\sum_{i=1}^N\log\left(\frac{1}{N}\sum_{j=1}^N \exp\Bigl(\langle \normg_\x(\x^i;\theta_\x), \normg_\y(\y^j;\theta_\y)\rangle/\tau\Bigr) \right).
\end{align*}
The Optimization Problem \ref{do:LN2} is solved by maximizing $-\LCLIPN(\cdot).$ This is achieved by maximizing the cosine similarities of paired
data in the first (single) sum over $i$ alone; and by simultaneously minimizing the cosine similarities of unpaired
data in the  second and third  (double) sums, over $i$ and $j$ together. This is the process we have been refering to, and will refer to, as alignment of
the two modalities.
\end{remark}

In the next section we discuss a population level version of the loss function. The following observation will help interpret this development.
To this end it is helpful to recall~\Cref{r:N}.

\subsection{Population Level Learning Problem} 
\label{ssec:pop}

The appropriate analogs of \eqref{eq:softmax_probability0} and \eqref{eq:softmax_probability2} in the population limit are given by
\begin{subequations}
\label{eq:continuum_probability}
\begin{align}
\pmodel(\x|\y;\theta) &= \frac{\exp\bigl(\langle \normg_\x(\x;\theta_\x), \normg_\y(\y;\theta_\y) \rangle /\tau\bigr)}{\mathbb{E}_{\x' \sim \mu_\x}\exp\bigl(\langle \normg_\x(\x';\theta_\x),  \normg_\y(\y;\theta_y) \rangle / \tau \bigr)},\\
\pmodel(\y|\x;\theta) &= \frac{\exp\bigl(\langle \normg_\x(\x;\theta_\x), \normg_\y(\y;\theta_\y) \rangle /\tau\bigr)}{\mathbb{E}_{\y' \sim \mu_\y}\exp\bigl(\langle \normg_\x(\x;\theta_\x),  \normg_\y(\y';\theta_y) \rangle / \tau \bigr)}.
\end{align}
\end{subequations}
Here (\ref{eq:continuum_probability}a) defines a density with respect to measure  $\mu_\x$ on $\mathcal{\X}$
and (\ref{eq:continuum_probability}b) defines a density with respect to measure  $\mu_\y$ on $\mathcal{\Y}.$ 
Thus, $\pmodel(\x|\y;\theta)\mu_\x(d\x)$ defines a probability measure on $\mathcal{\X}$ and
$\pmodel(\y|\x;\theta)\mu_\x(d\y)$ defines a probability measure on $\mathcal{\Y}.$
The appropriate analog of \Cref{do:LN2} is then given by: 
\begin{definitiono}
\label{do:L}
\begin{subequations}
\begin{align*}
\LCOND(\theta) & =  -\frac{1}{2}\mathbb{E}_{(\x,\y) \sim \mu} \Bigl[\log \pmodel(\x|\y;\theta) + \log \pmodel(\y|\x;\theta) \Bigr],\\
\tcond & \in \argmin_{\theta \in \R^p} \, \LCOND(\theta).
\end{align*}
\end{subequations}
\end{definitiono}

Following the line of reasoning in \Cref{r:wgwan}, the objective function $\LCOND(\theta)$ may be written as
\begin{align} 
    -\LCOND(\theta)
    &= \mathbb{E}_{(\x,\y) \sim \mu}\Bigl[\langle \normg_\x(\x;\theta_\x), \normg_\y(\y;\theta_\y) \rangle / \tau \Bigr] \notag \\
     &\qquad -\frac{1}{2} \mathbb{E}_{\y \sim \mu_\y}\log \mathbb{E}_{\x' \sim \mu_\x}\Bigl[\exp(\langle \normg_\x(\x';\theta_\x), \normg_\y(\y;\theta_\y)\rangle / \tau) \Bigr] \label{eq:consistent_loss}\\
    &\qquad\qquad -\frac{1}{2} \mathbb{E}_{\x \sim \mu_\x}\log \mathbb{E}_{\y' \sim \mu_\y}\Bigl[\exp(\langle \normg_\x(\x;\theta_\x),  \normg_\y(\y';\theta_\y)\rangle / \tau) \Bigr] \notag 
\end{align}
Replacing the expectations under $\mu, \mu_\x$ and $\mu_\y$  in $\LCOND$ by their empirical versions \eqref{eq:empirical} results in the following empirical risk function
\begin{align}
-\LCONDN(\theta) & = \frac{1}{N} \sum_{i=1}^N \langle \normg_\x(\x^i;\theta_\x), \normg_\y(\y^i;\theta_\y) \rangle / \tau \nonumber \\
&\qquad -\frac{1}{2N}\sum_{i=1}^N\log\left(\frac{1}{N}\sum_{j=1}^N \exp(\langle \normg_\x(\x^j;\theta_\x),  \normg_\y(\y^i;\theta_\y)\rangle / \tau) \right) \nonumber \\ 
&\qquad \qquad  -\frac{1}{2N}\sum_{i=1}^N\log\left(\frac{1}{N}\sum_{j=1}^N \exp(\langle \normg_\x(\x^i;\theta_\x), \normg_\y(\y^j;\theta_\y)\rangle / \tau) \right). \label{eq:lcondn}
\end{align}
The following is a direct consequence of \Cref{r:wgwan}:
\begin{theorem}
\label{t:lcondn}
The objective functions $\LCLIPN(\cdot)$ and $\LCONDN(\cdot)$ differ by $\log(N):$
$$\LCLIPN(\theta) = \LCONDN(\theta) +\log(N).$$
If $\tcondn$ minimizes $\LCONDN(\theta)$ and $\tcn$ solves Optimization Problem~\ref{do:LN2}, then $\tcondn=\tcn$.
\end{theorem}

In what follows we work with $\LCOND(\cdot)$, and hence its empirical counterpart $\LCONDN(\cdot)$. The latter is well-defined
in the population loss limit; working in the population loss limit clarifies understanding and is adopted throughout the remainder of
the paper. The presence of the constant $\log(N)$ in Theorem \ref{t:lcondn}, 
which leads to a divergence in the population loss limit of $\LCLIPN(\cdot)$,
relates to the fact that the loss $\LCLIPN(\cdot)$ is not defined via densities with respect to a reference probability measure, but via probabilities. Working with densities with respect to a reference probability measure enables seamless passage between population and empirical representations of the problem.

\section{Probabilistic Perspective}
\label{sec:PP}

In Subsection \ref{ssec:J} we formulate contrastive learning in terms of a minimization problem over probability measures on the joint
space $\mathcal{\X} \otimes \mathcal{\Y}$. This suggests several natural generalizations of contrastive learning; the first
class of such generalizations follow from using different probabilistic loss functions for the joint distribution and the second class from considering different measures of alignment of the two data modalities in the latent space. These two classes of generalizations are
considered in Subsections~\ref{ssec:G} and \ref{ssec:T}, respectively. Although we work at the population level, it is important that we are constrained by formulations which can be deployed in the empirical setting
by replacing the measures $\mu, \mu_\x$ and $\mu_\y$ with their empirical counterparts given in~\eqref{eq:empirical}.  We will discuss empiricalization explicitly in several cases to highlight this issue.

\subsection{Contrastive Learning in Terms of the Joint Distribution}
\label{ssec:J}

In this subsection, we relate the solution identified by contrastive learning to a joint distribution that is learned over the product space $\mathcal{\X} \times \mathcal{\Y}$. Let $\pmeas(d\x,d\y;\theta)$ be a probability measure for the joint random variable $(\x, \y) \in \mathcal{\X} \times \mathcal{\Y}$ defined by
\begin{equation} \label{eq:joint_probability_model}
\pmeas(d\x,d\y;\theta) =  \pmodel(\x,\y;\theta)  \mu_\x(d\x)\mu_\y(d\y),
\end{equation}
where
\begin{subequations}
\label{eq:rz}
    \begin{align}
    \pmodel(\x,\y;\theta) &=  \frac{1}{Z} \exp\Bigl(\langle \normg_\x(\x;\theta_\x), \normg_\y(\y;\theta_\y) \rangle / \tau \Bigr),\\
        Z &= \int_{\mathcal{\X} \times \mathcal{\Y}} \exp(\langle \normg_\x(\x;\theta_\x),  \normg_\y(\y;\theta_\y) \rangle / \tau) \mu_\x(d\x)\mu_\y(d\y).
    \end{align}
\end{subequations}
We refer to the change of measure $\rho$ as a \emph{tilting}; this tilting links the product of marginals of the 
data-generating distribution $\mu_\x, \mu_\y$
to a ($\theta-$parameterized) joint distribution $\nu=\nu(\cdot\,;\theta).$ The joint distribution defines a coupling of $\x$ and $\y$ to encode dependence 
between the two random variables. We show that the contrastive learning objective from the previous section can be reformulated in terms of a loss function which aligns $\pmeas$ with $\mu.$  We refer to the specific choice of $\rho$ here as an  \textit{exponential tilting}.

In the following, we let $\pmeas_{\x|\y}$ (resp. $\pmeas_{\y|\x}$) denote the conditional measure for $\x|\y$ (resp. $\y|\x$) under the joint measure $\pmeas$ in~\eqref{eq:pmeas1}. We then have
\begin{equation}
\label{eq:pmeas1}
\pmeas_{\x|\y}(d\x|\y;\theta)=\pmodel(\x|\y;\theta)\mu_\x(d\x),
\end{equation}
where $\pmodel(\cdot|\y;\theta)$ is defined in equation 
(\ref{eq:continuum_probability}a); likewise 
\begin{equation}
\label{eq:pmeas2}
\pmeas_{\y|\x}(d\y|\x;\theta)=\pmodel(\y|\x;\theta)\mu_\y(d\y),
\end{equation}
where $\pmodel(\cdot|\x;\theta)$ 
is defined in equation (\ref{eq:continuum_probability}b). We let $\mu_{\x|\y}$ (resp. $\mu_{\y|\x}$) denote the conditional measure for $\x|\y$ (resp. $\y|\x$) under the data generating measure $\mu$.

With this notation we may now define an optimization problem which
is equivalent to \Cref{do:L}. It casts the problem as  minimization of a sum of distances between probability measures;
in so doing it suggests paths for the generalization of contrastive learning.

\begin{definitiono}
\label{do:J}
\begin{subequations}
\begin{align}
\JCOND(\theta) & = \frac{1}{2}\mathbb{E}_{\y \sim \mu_\y}\left[\dkl(\mu_{\x|\y}(\cdot|\y)||\pmeas_{\x|\y}(\cdot|\y;\theta))\right] + \frac{1}{2}\mathbb{E}_{\x \sim \mu_\x}\left[\dkl(\mu_{\y|\x}(\cdot|\x)||\pmeas_{\y|\x}(\cdot|\x;\theta))\right], \nonumber \\
\theta^* & \in \argmin_{\theta \in \R^p} \, \JCOND(\theta). \nonumber
\end{align}
\end{subequations}
\end{definitiono}

\begin{theorem} \label{thm:clip_minimization} Let the conditionals of $\mu(d\x,d\y)$ satisfy the finite relative entropy conditions:
\begin{align*}
\mathbb{E}_{\y \sim \mu_\y} \dkl(\mu_{\x|\y}(\cdot|\y)||\mu_\x) &< \infty \\ 
\mathbb{E}_{\x \sim \mu_\x} \dkl(\mu_{\y|\x}(\cdot|\x)||\mu_\y) &< \infty.
\end{align*}
Then, $\JCOND(\theta) = \LCOND(\theta)+C$ where $C$ depends only on the relative entropy of the $\mu$ conditionals, and not on $\theta.$ Consequently, the Optimization Problems \ref{do:L} and \ref{do:J} coincide: the minimizer of $\JCOND$ and the minimizer of $\LCOND$ satisfy $\theta^*=\tcond$.
\end{theorem}

Proof of the theorem may be found in Appendix~\ref{app:A}.

\begin{remark}
\label{rem:J2L} The preceding theorem does not use any properties of the specific change of measure (tilting) defined by \eqref{eq:rz}. Thus,
the theorem generalizes to the other tiltings of the product measure $\mu_\x(d\x)\mu_\y(d\y)$ that we discuss in Subsection~\ref{ssec:T}.
\end{remark}

\subsection{Generalized Probabilistic Loss Functions}
\label{ssec:G}

In the previous subsection we have shown that the population level formulation of the standard constrastive learning problem can be recast
as learning a joint distribution, from a parameterized class, that best matches the true joint distribution of the data. Matching 
is determined by a sum of two terms measuring Kullback-Liebler divergences between the conditionals of the model and true joint distributions. 
This idea may be generalized in a number of directions, outlined in the next two subsubsections. 

\subsubsection{Generalized Conditional Losses}

One natural direction to generalize \Cref{do:J} is to replace the Kullback-Liebler divergence $\dkl$ by an arbitrary divergence, 
or metric, $\Div_\cU$ (resp. $\Div_\cV$) in the term leading to the $\x-$conditional (resp. $\y-$conditional). %
Furthermore, since the two conditionals for $\x$ and $\y$ may have different representations, the two terms in the loss function may not have a similar scale; indeed this observation applies to the original \Cref{do:J}. It may then be desirable to add scalar parameters $\lambda_\x,\lambda_\y \in \R_{+}$ that balance the two terms. These two generalizations result in the following minimization problem, in which $\Div$ denotes the pair $(\Div_\cU, \Div_\cV):$
\begin{definitiono}
\label{do:J1}
\begin{subequations}
\begin{align*}
\JCONDD(\theta;\lambda_\x,\lambda_\y) &= \frac{\lambda_\x}{2}\mathbb{E}_{\y \sim \mu_\y}\left[\Div_\cU(\mu_{\x|\y}(\cdot|\y)||\pmeas_{\x|\y}(\cdot|\y;\theta))\right] + \frac{\lambda_\y}{2}\mathbb{E}_{\x \sim \mu_\x}\left[\Div_\cV(\mu_{\y|\x}(\cdot|\x)||\pmeas_{\y|\x}(\cdot|\x;\theta))\right],\\
\tcond(\lambda_\x,\lambda_\y;\Div) & \in \argmin_{\theta \in \R^p} \, \JCONDD(\theta;\lambda_\x,\lambda_\y).
\end{align*}
\end{subequations}
\end{definitiono}

Note in particular that $\tcond\bigl(1,1;(\dkl,\dkl)\bigr)=\tcond$. The ratio of hyperparameters $\lambda_u/\lambda_v$ needs to be tuned to balance the two contributions to the objective function.
Furthermore, by choosing $\lambda_v=0$ (resp. $\lambda_u=0$) we obtain an objective function tuned to choose the joint distribution simply
to match the $\x|\y$ (resp. $\y|\x$) conditional and not the sum of both; this is desirable in applications where the desired 
downstream tasks focus on only one of the two conditionals. An example of this, covered in Subsection \ref{ssec:C}, is the fine-tuning
of classifiers based on contrastive learning. Indeed, in Subsection~\ref{ssec:lecun} we show that the original MNIST digit classification
algorithm can be viewed as a generalization of the standard CLIP methodology, using such a one-sided loss.

\begin{remark}
\label{rem:refback}
Optimization problem~\ref{do:J1} is well-defined for any divergence pair $\Div$. But a critical and practical issue is whether it defines an optimization problem for $\theta$ which is actionable given only samples from $\mu.$ This is possible when $\Div_\cU(\cdot||\cdot),\Div_\cV(\cdot||\cdot)$
are both chosen to be the forward Kullback-Liebler divergence $=\dkl(\cdot||\cdot)$; but it is not possible for the reverse Kullback-Liebler divergence,
which requires knowledge of the change of measure $r \coloneqq d\mu/(d\mu_\x \otimes d\mu_\y).$ Nor is it possible for the $\chi^2-$divergence or for the Hellinger or TV metrics, all of which also require knowledge of $r$. However, the optimization problem is still actionable for some choices other than the Kullback-Liebler divergence. The energy distances, of which maximum mean discrepancy (MMD) is a special case, provide a useful class of examples. We  illustrate this with Example~\ref{ex:mmd_conditionals} in~\Cref{app:mmd}.
\end{remark}

\subsubsection{Generalized Joint Losses}

Now we propose a somewhat different class of objectives, based around matching the joint distribution rather than conditionals. 
We assume that $\mu$ has density $r$ with respect to $\mu_\x \otimes \mu_\y$, so that
\begin{equation}
\label{eq:comr}
\mu(d\x,d\y)=r(\x,\y)\mu_\x(d\x)\mu_\y(d\y).
\end{equation}
Recalling $\nu(\cdot\,;\theta)$ given by \eqref{eq:joint_probability_model} we may now consider 
\begin{definitiono}
\label{do:J12}
\begin{subequations}
\begin{align*}
\JJOINT(\theta) &=  \dkl(\mu||\pmeas(\cdot,\cdot;\theta)),\\
\tjoint & \in \argmin_{\theta \in \R^p} \, \JJOINT(\theta).
\end{align*}
\end{subequations}
\end{definitiono}
Now note that, recalling function $\rho$ is defined by \eqref{eq:rz}, we are attempting to model $r(\cdot,\cdot)$ by parameterized function
$\rho(\cdot,\cdot;\theta).$ Minimization of $\JJOINT$ over $\theta$ reflects this goal as the following explicit calculation shows:
\begin{align}
\label{eq:dklj2}
\dkl(\mu||\pmeas(\cdot,\cdot;\theta)) &= \mathbb{E}_{(\x,\y) \sim \mu}[\log r(\x,\y) - \log \pmodel(\x,\y;\theta)].
\end{align}
The first term does not involve $\theta$ and may be ignored for the purposes of minimization to determine the optimal $\theta^*$. Thus we see that Optimization Problem~\ref{do:J12} may be formulated as:
\begin{definitiono}
\label{do:J2}
\begin{subequations}
\begin{align*}
-\LJOINT(\theta) &= \mathbb{E}_{(\x,\y) \sim \mu}[\langle \normg_\x(\x;\theta_\x), \normg_\y(\y;\theta_\y) \rangle/\tau] - \log \mathbb{E}_{(\x,\y) \sim \mu_\x \otimes \mu_\y}[\exp(\langle \normg_\x(\x;\theta_\x), \normg_\y(\y;\theta_\y) \rangle/\tau)],\\
\tjoint & \in \argmin_{\theta \in \R^p} \, \LJOINT(\theta).
\end{align*}
\end{subequations}
\end{definitiono}

Indeed we have proved:
\begin{theorem}
    \label{thm:addedT}
Assume that $\mathbb{E}_{(\x,\y) \sim \mu}\bigl(\log r(\x,\y)\bigr)<\infty.$ Then, $\JJOINT(\theta) = \LJOINT(\theta)+C$ where $C$ does not depend on $\theta.$ Consequently, the Optimization Problems \ref{do:J12} and \ref{do:J2} coincide. 
\end{theorem}

\begin{remark} Evaluating the joint loss in Optimization Problem~\ref{do:J2} has computational advantages over the commonly used loss $\LCOND$ for contrastive learning, defined in Optimization Problem~\ref{do:L}. The joint loss $\LJOINT$ is efficiently evaluated  by computing the cosine similarities for one batch of data from the joint distribution $\mu$ and one batch from the tensor product of marginal distributions $\mu_\x \otimes \mu_\y$. The contrastive loss $\LCOND$, on the other hand, requires a negative batch for each sample from one of the marginal distributions to evaluate the second and third terms in Optimization Problem~\ref{do:L}.
\end{remark}

The following theorem connects the two optimization problems and is proved in Appendix \ref{app:A}:
\begin{theorem} \label{thm:equivalent_objectives} For all $\theta \in \R^p$, the objectives in Optimization Problems \ref{do:L} and \ref{do:J2} are related by 
the inequality
\begin{equation*}
\LCOND(\theta) \leq \LJOINT(\theta).
\end{equation*}
Thus,
$$\LCOND(\tcond) \le \LCOND(\tjoint) \le \LJOINT(\tjoint).$$
\end{theorem}

Thus, solution of the joint minimization problem provides an upper bound for the conditional optimization problem.
Finally we note that it is also possible to generalize Optimization Problem~\ref{do:J12} by using a different divergence, or metric:
\begin{align}
\label{eq:dklj3}
\theta^* = \argmin_{\theta \in \R^p} \, \Div(\mu||\pmeas(\cdot,\cdot;\theta)).
\end{align}
The comments in Remark \ref{rem:refback}, concerning actionable loss functions, apply also to this loss 
function. One example  of an actionable loss function is to take $\Div$ to be the maximum mean discrepancy 
$\Div_{\mathsf{mmd}}$ as defined in Example~\ref{ex:mmd_conditionals}. This results in the following optimization problem:
\begin{definitiono}
\label{do:J3}
\begin{subequations}
\begin{align*}
\LJOINTM(\theta) &= - 2\mathbb{E}_{(x,y) \sim \mu \otimes \pmeas(\cdot;\theta)} k(x,y) + \mathbb{E}_{(y,y') \sim \pmeas(\cdot;\theta) \otimes \pmeas(\cdot;\theta)} k(y,y') \\
\tjointm & \in \argmin_{\theta \in \R^p} \, \LJOINTM(\theta).
\end{align*}
\end{subequations}
\end{definitiono}

\subsection{Generalized Tilting} 
\label{ssec:T}

Recall that contrastive learning may be thought of as learning a probability measure $\nu$, given by \eqref{eq:joint_probability_model}, 
\eqref{eq:rz}, so that it is close to the data generating distribution $\mu.$ The change of measure $\rho$ defined in \eqref{eq:rz}
is referred to as a tilting. In this subsection we develop variants on this standard tilting. We concentrate on two variants both of which
we will return to in the Gaussian setting studied in~\Cref{sec:GS}. However, the reader 
will readily identify numerous other generalizations based on different choices of parameterized function $\rho$ in the expression \eqref{eq:joint_probability_model}.

\subsubsection{Unnormalized Encoders}
Recall that contrastive learning introduces encoders $g_\x, g_v$ on the two modalities, but defines loss functions through use of the cosine distance based on the normalized encoders given in \eqref{eq:normalized}. However it is possible to simply drop the constraint that the embedding vectors are normalized. Making this change in \eqref{eq:joint_probability_model}, \eqref{eq:rz} leads to a new model class of joint distributions with the form
\begin{subequations}
\label{eq:T1}
    \begin{align}
    \pmeas(d\x,d\y;\theta) &=  \pmodel(\x,\y;\theta)  \mu_\x(d\x)\mu_\y(d\y),\\
    \pmodel(\x,\y;\theta) &=  \frac{1}{Z} \exp\Bigl(\langle \g_\x(\x;\theta_\x), \g_\y(\y;\theta_\y) \rangle / \tau \Bigr),\\
        Z &= \int_{\mathcal{\X} \times \mathcal{\Y}} \exp(\langle \g_\x(\x;\theta_\x),  \g_\y(\y;\theta_\y) \rangle / \tau) \mu_\x(d\x)\mu_\y(d\y).
    \end{align}
\end{subequations}
These unnormalized encoders provide our first example of a generalized tilting.

\subsubsection{$L^2-$Distance}
Now note that, if the embedding vectors in \eqref{eq:joint_probability_model}, \eqref{eq:rz} are normalized, 
then the density function of the joint distribution \eqref{eq:rz} can be written as
$$\pmodel(\x,\y;\theta)  \propto \frac{1}{Z}  \exp\Bigl(-\frac{1}{2\tau}\left|\normg_\x(\x;\theta_\x) - \normg_\y(\y;\theta_\y)\right|^2\Bigr)\mu_\x(d\x)\mu_\y(d\y).$$
Using this form for the tilting, but in the unnormalized setting, leads to the %
model:
\begin{subequations}
\label{eq:rz2}
    \begin{align}
    \pmeas(d\x,d\y;\theta) &=  \pmodel(\x,\y;\theta)  \mu_\x(d\x)\mu_\y(d\y),\\
    \pmodel(\x,\y;\theta)  &=  \frac{1}{Z}  \exp\Bigl(-\frac{1}{2\tau}\left|\g_\x(\x;\theta_\x) - \g_\y(\y;\theta_\y)\right|^2\Bigr)\mu_\x(d\x)\mu_\y(d\y),\\
        Z &= \int_{\cX \times \cY} \exp \Bigl(-\frac{1}{2\tau}\left|\g_\x(\x;\theta_\x) - \g_\y(\y;\theta_\y)\right|^2\Bigr) \mu_\x(d\x)\mu_\y(d\y).
    \end{align}
\end{subequations}
These unnormalized encoders provide our second example of a generalized tilting.

\subsubsection{The General Setting}
\label{sssec:gs}

Both of the preceding modified tiltings can be used within all of the generalized probabilistic loss functions described in Subsection~\ref{ssec:G};
in particular, Optimization Problems \ref{do:J1} and \ref{do:J12} are well-defined for any tilting of the form \eqref{eq:joint_probability_model}.
In this context, the following observation will be useful to us in what follows.

\begin{remark}
\label{rem:J2L2} Theorem \ref{thm:clip_minimization} may be generalized to Optimization Problem \ref{do:J1} in the case where $\Div_\cU$ and $\Div_\cV$ are both chosen to be $\dkl.$ The analogous asymmetric loss function and minimization problem is given by 
\begin{definitiono} \label{opt:KLonesided}
\begin{subequations}
    \begin{align*} 
        -\LCOND(\theta;\lambda_\x,\lambda_\y) &= \mathbb{E}_{(\x,\y) \sim \mu}\left[\frac{\lambda_\x}{2} \log \rho(\x|\y;\theta) + \frac{\lambda_\y}{2}\log \rho(\y|\x;\theta)\right]. \\
        \tcond(\lambda_\x,\lambda_\y;\dkl) &\in \argmin_{\theta \in \R^p} \LCOND(\theta;\lambda_\x,\lambda_\y).
    \end{align*}
\end{subequations}
\end{definitiono}
Following the same proof of Theorem \ref{thm:clip_minimization}, the minimizers of $\JCOND(\cdot;\lambda_\x,\lambda_\y)$ and $\LCOND(\cdot;\lambda_\x,\lambda_\y)$ coincide for all $\lambda_\x,\lambda_\y \in \R_+$. The reader will be able to identify
a similar generalization of Theorem~\ref{thm:addedT}.
\end{remark}

\section{Retrieval and Classification}
\label{sec:RAC}
Two information science applications that make use of contrastive learning are retrieval and classification.
In Subsection~\ref{ssec:R} we describe the problem of retrieval, framing it using the perspective on contrastive learning that we have developed over the two preceding sections; Subsection~\ref{ssec:C} is devoted to a similar treatment of classification.

\subsection{Retrieval}
\label{ssec:R}

In this subsection we formalize the use of contrastive learning as a building block in the retrieval task.
The methodology we describe in this subsection proceeds with $\theta=(\theta_\x,\theta_\y)$ fixed at the optimal value found
through the contrastive learning procedure described in Subsection~\ref{ssec:J}, or its subsequent variants detailed in Subsections \ref{ssec:G} and \ref{ssec:T}. Given one realization from a data modality, \emph{crossmodal retrieval} identifies relevant items, from the prescribed marginal distribution of the  other modality, that are aligned with it. In practice, this is accomplished by computing the cosine similarity of the given realization  with samples from the other data modality and returning the elements (or possibly a single element) with highest similarity. The complete procedure is described in Algorithm~\ref{alg:retrieval}. This procedure is referred to as \emph{zero-shot retrieval} as it does not require any additional training of parameter $\theta$ after learning the encoders.

\begin{algorithm}
\caption{Crossmodal Retrieval \label{alg:retrieval}}
\begin{algorithmic}[1]
\STATE \textbf{Input}: Input $\y \in \mathcal{\Y}$, data samples $\mathcal{\X}^N = \{\x^i\}_{i=1}^N$,  number of similar items $K$ 
\STATE Compute cosine similarities $s_{i} \coloneqq \langle \normg_\x(\x^i), \normg_\y(\y) \rangle$ for each sample $\x^i$ %
\STATE Identify $K$ distinct indices $\sigma^* \in \mathbb{N}^K$ with largest similarities, i.e., $\sigma^* \in \argmax_{\sigma \in [[1,K]]} \sum_{k=1}^K s_{\sigma(i)}$
\STATE \textbf{Output}: Samples $\x^{\sigma^*(k)}$ for $k = 1,\dots,K$
\end{algorithmic}
\end{algorithm}

\begin{example}
\label{ex:tti}
    Take $\mathcal{\X}$ to denote images and  $\mathcal{\Y}$ to denote text prompts. In this setting $\y \in \mathcal{\Y}$ represents a text prompt, typically not in the training data set $\mathcal{\Y}^N$; and $\mathcal{\X}^N$ is the collection of $N$ images in the training dataset. Crossmodal retrieval seeks the top $K$ images from $\mathcal{\X}^N$  that are most aligned, in the precise sense described by Algorithm
    \ref{alg:retrieval}, to the given prompt $\y.$ The embedding of all images in $\mathcal{\X}^N$ is precomputed offline.
    Given any specific text prompt $\y$ encountered, the online computational cost is then $\mathcal{O}(\de N)$. 
\end{example}

Here we show how it is performed in practice: with empirical measures computed from 
the discrete data set in Data Assumption~\ref{as:paired_data}. A generalization of  retrieval to the population level setting of measures with continuous densities is found in~\Cref{sec:retrieval-population}.

In practice $\mu_\x$ may not have density with respect to Lebesgue measure and, more fundamentally,
is only available through samples: the marginal distribution for $\x$ is specified by the equally weighted empirical data measure $\mu_{\x}^N$ in~\eqref{eq:empirical} supported on $\mathcal{\X}^N$ in~\eqref{eq:empirical2}. Then, the learned conditional distribution for $\x|\y$ is given by
\begin{equation} \label{eq:condu2}
\pmeas_{\x|\y}^N(d\x|\y;\theta) \propto \exp\Bigl(\langle \normg_\x(\x;\theta_\x), \normg_\y(\y;\theta_\y) \rangle / \tau \Bigr) \mu_\x^N(d\x).
\end{equation}
We may still give an interpretation as a Bayesian inverse problem: we take as $\mu^N_\x$ from \eqref{eq:empirical} as the prior and the function $\rho(\cdot|\y)$ defined in (\ref{eq:continuum_probability}a) as the likelihood (up to a constant of proportionality).
However, the conditional distribution $\pmeas_{\x|\y}^N$, which is the resulting posterior in the Bayesian interpretation, 
is only supported at $N$ points $\mathcal{\X}^N = \{\x^i\}_{i=1}^N$. 
The measure parameterized by $\y$ can be written as the weighted empirical measure
$$\pmeas_{\x|\y}^N(\cdot|\y;\theta) \coloneqq \sum_{i=1}^N w_i(\y;\theta) \delta_{\x^i}.$$
The weights $w_i(\y;\theta) \in [0,1]$ sum to $1$ for any input $\y$. They are defined as
$$w_i(\y;\theta) \coloneqq \frac{\omega_i(\y;\theta)}{\sum_{\ell=1}^N \omega_\ell(\y;\theta)}, \qquad \omega_i(\y;\theta) = \exp\Bigl(\langle \normg_\x(\x^i;\theta), \normg_\y(\y;\theta) \rangle/\tau\Bigr).$$
Continuing the interpretation as a Bayesian inverse problem, the mode or MAP point in this context is defined as the point in $\mathcal{\X}^N$, which maximizes the density of the posterior $\pmeas_{\x|\y}^N$
with respect to the empirical measure $\mu_\x^N$; because $\mu_\x^N$ comprises equally weighted points, this is equivalent to maximizing the likelihood over the data set. Given that the weights ${w}_i$ are proportional to the density, on the data set, we have the following theorem. The proof is found in Appendix~\ref{app:A}.
\begin{theorem} \label{thm:retrieval}
The retrieval process in Algorithm~\ref{alg:retrieval} with $K = 1$ finds a  mode of the empirical conditional distribution
$\pmeas_{\x|\y}^N.$ For each $\y \in \mathcal{Y}$, retrieval computes
\begin{equation} \label{eq:retrieval_discrete}
    \argmax_{i=1,\dots,N} \, \langle \normg_\x(\x^i;\theta), \normg_\y(\y;\theta) \rangle. %
\end{equation}
\end{theorem}

\subsection{Classification}
\label{ssec:C}

In this subsection we formalize the use of contrastive learning as a building block in the classification task. Whilst finding an image from text is the canonical retrieval task, as explained in Example \ref{ex:tti}, the canonical classification task is the assignation of a text classifier to an image. The complete procedure is described in Algorithm~\ref{alg:discrete_classifier} (recall Remark \ref{rem:softmax} for definition of the {\tt softmax} operation)
and a key point to appreciate is that the labels are not necessarily taken from the training data.
\begin{algorithm}
\caption{Crossmodal Classifier  \label{alg:discrete_classifier}}
\begin{algorithmic}[1]
\STATE \textbf{Input}: Input $\x \in \mathcal{\X}$, labels $\mathcal{C} = \{\y^i\}_{i=1}^K$ %
\STATE Compute cosine similarities $s_{i} \coloneqq \langle \normg_\x(\x), \normg_\y(\y^i) \rangle$ for each label $\y^i$ and $i = 1,\dots,K$ 
\STATE \textbf{Output}: Most likely label $i^\ast = \arg\max_{i=1,\dots,K} s_i$ and (optional)  probabilities $\texttt{softmax}(s_i)$ for each label 
\end{algorithmic}
\end{algorithm}

\begin{example} \label{ex:itt}
Take $\mathcal{\X}$ to denote the space of images and  $\mathcal{\Y}$ to denote text prompts. Although image-to-text classification might, in principle, be achieved by simply reversing the roles of text and image in Example~\ref{ex:tti} for retrieval, this typically results in poor performance because the set of all text prompts in the data set is typically inadequate to classify arbitrary new images. Moreover, often interest is focused on classifying an image using a small set of labels---for example diagnosing whether tissue in medical images is healthy or not. 
To solve classification problems of this type, an initial pre-training step of contrastive learning is performed using a large population of text and image pairs
$\mathcal{\X}^N$ and $\mathcal{\Y}^N$ as defined in \eqref{eq:empirical2};
typically $\mathcal{C}$, or at least all of $\mathcal{C}$, is not in $\mathcal{\Y}^N.$ Then, 
fine-tuning specializes the classifier to increase its accuracy as a predictor of a label in $\mathcal{C}$, given an image, by updating
encoder parameters on the basis of a new (and typically smaller) set of labels and images, that are different from those considered during pre-training, but now containing $\mathcal{C}$.
\end{example}

When \emph{pretrained} encoders are used to classify inputs among a finite set without additional fine-tuning,
Algorithm~\ref{alg:discrete_classifier} is referred to as \emph{zero-shot classification}. We will concentrate, henceforth,
on using \emph{fine-tuning}, motivated by the preceding example, but not working in the specific context of text-image data. %
Next, we describe actionable algorithms for classification, based on discrete data with a finite number of labels. A description of the population level problem with continuous densities is presented  in~\Cref{sec:classification-population}.

Our starting point is to assume that we have a pretrained model, via access to the data in Data Assumption \ref{as:paired_data}, which is used
to determine the parameters of the encoder $(\theta_\x,\theta_\y)$. We first describe \emph{zero-shot classification}: how to use the encoders to perform classification over a new distribution of labels, using this pretrained model. We then update the parameters of the encoder for $\y$, and learn a new reference marginal distribution, by use of a second data set of size $M \ll N$---the \emph{fine-tuning process}, with
goal being improved bespoke classification. This second data set is defined in:

\begin{dataassumption} \label{as:paired_data2} The fine-tuning data comprises pairs $\{(\xhat^i,\yhat^i)\}_{i=1}^M$ drawn i.i.d.\thinspace from a joint distribution $\mufinetuning(\x,\y)$ on $\mathcal{\X} \times \mathcal{\Y}.$
\end{dataassumption}

To describe zero-shot classification we first define finite set 
$$\mathcal{C} = \{\yhat^i\}_{i=1}^K \subseteq \{\yhat^i\}_{i=1}^M.$$ 
The set $\mathcal{C}$ comprises a set of \emph{labels} which we wish to assign in the classification task. Without loss of generality we
have ordered them to be the first $K$ members of the entire fine-tuning data set, marginalized on $\mathcal{\Y}.$
We then define the empirical measures  
\begin{equation}
\label{eq:empiricalK}   
\mufinetuning^K \coloneqq \frac{1}{K}\sum_{\ell=1}^K \delta_{(\x^\ell,\y^\ell)}, \quad
\mufinetuning_\x^K \coloneqq \frac{1}{K}\sum_{\ell=1}^K \delta_{\xhat^\ell}, \quad \mufinetuning_\y^K \coloneqq \frac{1}{K}\sum_{\ell=1}^K \delta_{\yhat^\ell}. 
\end{equation}
We may use $\mufinetuning_\y^K$ as reference measure to define a conditional measure from a pair of pretrained encoders with parameter
$\theta=(\theta_\x,\theta_\y)$; this leads to the weighted empirical measure
\begin{align}
\nunew_{\y|\x}^K(\y|\x;\theta) &\propto \exp(\langle \normg_\x(\x;\theta_\x), \normg_\y(\y;\theta_\y) \rangle / \tau)\mufinetuning_\y^K(d\y) = \sum_{i=1}^K w_i(\x;\theta) \delta_{\y^i}.
\end{align}
Here, $w_i(\x;\theta) \in [0,1]$ are normalized weights, defined for any sample $\x$, and given by
$$w_i(\x;\theta) \coloneqq \frac{\omega_i(\x;\theta)}{\sum_{\ell=1}^K \omega_\ell(\x;\theta)}, \qquad \omega_i(\x;\theta) = \exp(\langle \normg_\x(\x;\theta_\x), \normg_\y(\y^i;\theta_\y) \rangle/\tau).$$
The most likely label for each input $\x$ is then given by the mode 
$$\y^*(\x) \in \argmax_{\y^i \in \mathcal{C}} \nunew_{\y|\x}^K(\y^i|\x;\theta) = \argmax_{i=1,\dots,K} w_i(\x).$$

\begin{example} In the Bayesian interpretation, where we view the titling proportional to $\pmodel(\x,\y;\theta)$ as a likelihood function, classification over an equally weighted marginal distribution of labels is a MAP estimator. It can also be interpreted as the restricted maximum likelihood estimator
$$\arg\max_{\y^i \in \mathcal{C}} \pmodel(\y^i|\x;\theta) = \argmax_{\y^i \in \mathcal{C}} \, \langle \normg_\x(\x;\theta_\x), \normg_\y(\y^i;\theta_\y) \rangle .$$
\end{example}

As in the population loss setting one may fine-tune the parameters of the $\y$ encoder and redefine the prior to align the classification algorithm
with the fine-tuning data set given in \Cref{as:paired_data2}. To this end we now fix the pretrained parameters $\theta_\x$ of the $\x$ encoder. 
We then define a new prior $\munef_{\y}^K$ defined by parameter $\theta_\phi \coloneqq F=(F_1, \cdots, F_K) \in \R^K$:
\begin{equation}
    \munef_{\y}^K(d\y;\theta_\phi) = \sum_{i=1}^K \frac{ \exp(F_i)}{\sum_{j=1}^K \exp(F_j)} \delta_{\yhat^i}.
\end{equation}
Thus $\munef_\y^K$ is a reweighting of $\munew_\y^K$, analogous to the re-weighting of $\munew_\y$, to define $\munef_\y$, defined at the population level. Then, the conditional measure for $\y|\x$ depending on parameter $\vartheta = (\theta_\y,F)$ is given by the weighted empirical measure
\begin{equation*}
\nunef_{\y|\x}^K(\y|\x;\vartheta) \propto \exp(\langle \normg_\x(\x;\theta_\x), \normg_\y(\y;\theta_\y) \rangle /\tau)\munef_\y^K(d\y;\theta_\phi)=\sum_{i=1}^K w_i(\x;\theta) \delta_{\y^i}, 
\end{equation*}
where $w_i(\x;\theta) \in [0,1]$ are normalized weights defined as 
\begin{align}
w_i(\x;\theta) = \frac{\omega_i(\x;\theta)}{\sum_{l = 1}^M \omega_l(\x;\theta)}, \qquad \omega_i(\x;\theta) = \exp\bigl( \langle \normg_\x(\x;\theta_\x), \normg_\y(\y^i;\theta_\y) \rangle / \tau + F_i \bigr).
\end{align}
We wish to adjust $\theta_v$, along with $F$, to improve performance on the specified classifiers $\mathcal{C}$.
But, to classify an input $\x$ among a finite set of labels from $\mathcal{C}$, we only need the action of the encoder at $\y^i \in \mathcal{C}$. Thus, let the matrix $G \in \R^{\de \times K}$ %
be a set of weights where each column $G_{i} = \normg_\y(\y^i;\theta_\y) \in \R^{\de}$ contains the evaluation of the encoder for label $v^i$. 
We let $e_\y$ be the  unit-vector corresponding to the label $\y$ (i.e., one-hot encoding)
and now redefine $\vartheta=(G,F)$, noting that we only need to optimize over $(G,F)$ not $(\theta_\y,F).$  We now fine-tune $\vartheta$ using the 
dataset $\mufinetuning$ prescribed in Data Assumption~\ref{as:paired_data2} by minimizing the discrete version of the one-sided loss in~\eqref{do:J1}\footnote{Throughout the paper $\langle \cdot\,. \,\cdot \rangle$ denotes the inner-product on $\R^{\de};$ note that the inner-product here is on $\R^K.$}:
\begin{subequations}  
\label{op:finetuning0}
\begin{align}
    -\LFINE^K(\vartheta) &= \mathbb{E}_{(\x,\y) \sim \mufinetuning^K}\Bigl[\bigl\langle e_{\y}, \bigl(G^\top \normg_\x(\x;\theta_\x) / \tau + F\bigr) \bigr\rangle_{\R^K} \Bigr]\\
    &\quad\quad\quad - \mathbb{E}_{\x \sim \mufinetuning_\x^K} \log \mathbb{E}_{\y \sim \mufinetuning_\y^K} \exp\Bigl[\bigl\langle e_{\y}, \bigl(G^\top \normg_\x(\x;\theta_\x) / \tau + F\bigr) \bigr\rangle_{\R^K}\Bigr], \nonumber \\
    \tfine &\in \argmin \LFINE^K(\vartheta).
\end{align}
\end{subequations}
Given the learned parameter of the fine-tuned model, the classifier for each input $\y$ is defined as the mode of the conditional distribution for $v$ given $u$: 
\footnote{The optimization problem defined by equations \eqref{op:finetuning0} has population level analog Optimization Problem \ref{op:finetuning};
likewise \eqref{eq:mode_labelsA0} has population level analog~\eqref{eq:mode_labelsA}.}
\begin{equation}
\label{eq:mode_labelsA0}
\y^*(\x) \in \argmax_{\y^i \in \mathcal{C}} \nunef_{\y|\x}^K(\y|\x;\tfine) = \argmin_{i=1,\dots,K} \bigl(G_i^\top \normg_\x(\x;\theta_\x) + F_i \bigr).
\end{equation}

\section{Gaussian Setting}
\label{sec:GS}

In this section we study the contrastive learning problem when the data distribution is a multivariate Gaussian. This enables explicit insights into the capabilities of the standard approach to contrastive learning and the suggested variants on it proposed in~\Cref{sec:PP}; these insights are obtained via theory and via straightforward numerical experiments. The main theoretical highlights are contained in Corollaries \ref{cor:matching_means}, \ref{cor:matching_meansandcov} and \ref{corr:marginals_joint_loss}. And these results are supported by further theory concerning low-rank 
approximation and by numerical studies that take our understanding beyond the theory.

We let $(\x,\y) \in \mathcal{\X} \times  \mathcal{\Y} \coloneqq 
\R^{\dx} \times \R^{\dy}$ be a multivariate centered Gaussian random variable with distribution $\mu=\mathcal{N}(0,\cov)$. We assume that the covariance matrix $\cov$ has block form
$$\cov = \begin{bmatrix} \cov_{\x\x} & \cov_{\x\y} \\ \cov_{\y\x} & \cov_{\y\y} \end{bmatrix},$$
where $\cov$ is strictly positive-definite. Matrices $\cov_{\x\x}$ and $\cov_{\y\y}$ are also then necessarily strictly positive-definite and hence invertible. The two marginal distributions for $\x$ and $\y$ are given by
\begin{align} \label{eq:Gaussian_marginals}
    \mu_{\x}(\cdot) = \mathcal{N}(0,  \cov_{\x\x}), \qquad 
    \mu_{\y}(\cdot) = \mathcal{N}(0, \cov_{\y\y});
\end{align}
moreover, the two conditional distributions for $\x|\y$ and $\y|\x$ are given by
\begin{subequations} \label{eq:Gaussian_conditionals}
\begin{align} 
    \mu_{\x|\y}(\cdot|\y) &= \mathcal{N}(\cov_{\x\y}\cov_{\y\y}^{-1}\y, \cov_{\x|\y}), \qquad \cov_{\x|\y} \coloneqq \cov_{\x\x} - \cov_{\x\y}\cov_{\y\y}^{-1}\cov_{\y\x}, \label{eq:true_conditional_xy}\\
    \mu_{\y|\x}(\cdot|\x) &= \mathcal{N}(\cov_{\y\x}\cov_{\x\x}^{-1}\x, \cov_{\y|\x}), \qquad \cov_{\y|\x} \coloneqq \cov_{\y\y} - \cov_{\y\x}\cov_{\x\x}^{-1}\cov_{\x\y}.\label{eq:true_conditional_yx}
\end{align}
\end{subequations}
In~\Cref{sec:PP} we formulate the contrastive approach to learning in terms of finding a representation of the joint distribution as
a change of measure (tilting) from  a reference measure defined as the independent product of the
marginals in \eqref{eq:Gaussian_marginals}. The key question we focus on in this section
is the ability of the learned joint distribution to accurately
replicate the conditionals in~\eqref{eq:Gaussian_conditionals}. 
To enable explicit analysis of this question we employ linear encoders and log-quadratic tiltings so that the learned joint
distribution is also Gaussian.

We introduce embedding (latent space) dimension $\de$ and seek tiltings based on the linear encoders
\begin{equation}
    \label{eq:LE}
\g_\x(\x) = G \x, \qquad \g_\y(\y) = H \y,
\end{equation}
with $G \in \R^{\de \times \dx}$ and $H \in \R^{\de \times \dy}$ and $\de \leq \min(\dx,\dy)$\footnote{It is straightforward to generalize the analysis to a setting with non-zero mean $\mean$ in $\mu.$ To consistently approximate the conditional means in this setting, one may consider affine encoders of the form $\g_\x(\x) = G\x + g,$ and $\g_\y(\y) = H\y + h$ with $g \in \R^{\dx}$ and $h \in \R^{\dy}$.}. We will work primarily in the settings of the tiltings defined in \eqref{eq:T1} and \eqref{eq:rz2}, making an exception to include normalization of the encoders only in some specific numerical experiments. We will also take the temperature parameter $\tau = 1$, since this can be absorbed into the matrices $G,H$. To determine parameters $G,H$ we will study both Optimization Problems \ref{opt:KLonesided} and \ref{do:J12} for the generalized conditional and the joint losses, respectively.

In~\Cref{sec:Gaussian_cosine} we study the standard Optimization Problem~\ref{do:J}, i.e., the particular symmetric case of Optimization Problem~\ref{opt:KLonesided}. We show we can match the conditional means of both distributions in~\eqref{eq:Gaussian_conditionals}, i
f the latent space has high enough dimension (Corollary~\ref{cor:matching_means}); otherwise we
identify the best low-rank approximation of the conditional means. However this methodology fails to represent the conditional covariances. \Cref{sec:Gaussian_quadraticform} uses a one-sided choice in Optimization Problem \ref{opt:KLonesided}, matching only the conditional of one modality on the other, not both. We show that in this setting we can match the mean and covariance of the one-sided conditional distribution if the latent space has high enough dimension (Corollary~\ref{cor:matching_meansandcov}); and we identify the optimization problem for the best low-rank approximation. In Subsection~\ref{ssec:gjoint}, we employ the Optimization Problem \ref{do:J12} based on matching the joint distribution rather than conditionals, and explicitly identify the family of minimizers which attain the minimum value of the objective function; we show that the joint loss leads to better
representation of the marginals (Corollary~\ref{corr:marginals_joint_loss}). \Cref{sec:Gaussian_visualization} is devoted to
numerical illustrations of the proven theoretical results, together with empirical extensions of the theory to new settings; these new settings include the addition of normalization constraints on the encoders in the setting of \Cref{sec:Gaussian_cosine}, loss matching both conditionals in  the setting of \Cref{sec:Gaussian_quadraticform} and the study of the loss function minimizing divergence of the joint distributions from \Cref{ssec:gjoint}.

\subsection{Cosine Distance: Conditional Loss} \label{sec:Gaussian_cosine}

Consider Optimization Problem \ref{do:J} with model class $\nu$ defined by the linear encoders in~\eqref{eq:LE}:
\begin{equation} \label{eq:joint_probability_modelI}
\pmeas(d\x,d\y;\theta) = \frac{1}{Z} \exp\Bigl(\langle G\x, H\y \rangle \Bigr) \mu_\x(d\x)\mu_\y(d\y).
\end{equation}
We may define parameter $\theta \coloneqq (G, H).$ Note, however, that $\pmeas$ is invariant to re-scaling of the parameters $G \mapsto \alpha U G$ and $H \mapsto \frac{1}{\alpha} U H$ for any scalar $\alpha > 0$ and any orthonormal matrix $U \in \R^{\de \times \de}$. 
To avoid this invariance we redefine $\theta=A \in \R^{\dx \times \dy}$, where $A \coloneqq G^\top H$.

For the Gaussian marginal distributions $\mu_\x,\mu_\y$ in~\eqref{eq:Gaussian_marginals}, $\nu$ is a multivariate Gaussian distribution $\pmeas = \mathcal{N}(\mean_\theta,\cov_\theta)$ where the mean $\mean_\theta$ and inverse covariance matrix $\cov_\theta^{-1}$ have the block-form 
\begin{equation} \label{eq:Gaussianmodel_moments}
    m_\theta = \begin{bmatrix} 0 \\ 0 \end{bmatrix}, \qquad \cov_\theta^{-1} = \begin{bmatrix} \cov_{uu}^{-1} & -A \\ -A^\top & \cov_{vv}^{-1} \end{bmatrix}.
\end{equation}
Moreover, the learnable model $\nu$ has the conditional distributions
\begin{subequations}
\begin{align}    
\label{eq:ConditionalModelI_Gaussian}
\pmeas_{\x|\y}(d\x|\y;A) &= \frac{1}{Z}\exp\Bigl(\langle \x, A \y \rangle\Bigr)\mu_\x(d\x) = \mathcal{N}(d\x;\cov_{\x\x} A\y, \cov_{\x\x}),\\
\pmeas_{\y|\x}(d\y|\x;A) &= \frac{1}{Z}\exp\Bigl(\langle  \x, A \y \rangle\Bigr)\mu_\y(d\y) = \mathcal{N}(d\y;\cov_{\y\y} A^\top \x, \cov_{\y\y}).
\end{align}
\end{subequations} 
By Remark \ref{rem:J2L2}, the relevant optimization problem 
over $\theta \coloneqq A$ is to minimize the loss function
\begin{align}
\LCOND(A) &= -\mathbb{E}_{(\x,\y) \sim \mu} \Bigl[ \langle \x, A \y \rangle  \Bigr] \notag \\
&\quad +\frac{1}{2} \mathbb{E}_{\x \sim \mu_\x} \log \mathbb{E}_{\y \sim \mu_\y}\Bigl[ \exp(\langle  \x, A \y \rangle)\Bigr] + \frac{1}{2}\mathbb{E}_{\y \sim \mu_\y} \log \mathbb{E}_{\x \sim \mu_\x}\Bigl[\exp(\langle  \x, A \y \rangle)\Bigr]. \label{eq:objective_linear_encoder}
\end{align}

The following result provides a closed form solution for the minimizer of this loss function. The proof of the theorem may be found in Appendix~\ref{app:GaussianProofs}.
Recall from~\Cref{ssec:notation} that
for any matrix $\Sigma$, $\Sigma_r$ denotes the rank$-r$ truncation of its singular value decomposition.

\begin{theorem} \label{thm:Gaussian_solution_cosine}
The minimizer of $\LCOND(A)$ over all matrices of size $\dx \times \dy$ 
is
\begin{equation} \label{eq:Gaussian_cond_cosine_minimizer_unconstrained}
A^\ast =
\cov_{\x\x}^{-1}\cov_{\x\y}\cov_{\y\y}^{-1}. \end{equation}
The minimizer of $\LCOND(A)$ over the set of rank-r matrices $\mathcal{A}_r = \{A \in \R^{\dx \times \dy}, \; \text{rank}(A) \leq r\}$ for $0 < r \leq \min(\dx,\dy)$ is given by
\begin{equation}
\label{eq:Gaussian_cond_cosine_minimizer_constrained}
A^*(r) = 
\cov_{\x\x}^{-1/2} (\cov_{\x\x}^{-1/2}\cov_{\x\y}\cov_{\y\y}^{-1/2})_r \cov_{\y\y}^{-1/2}.
\end{equation}
\end{theorem}

The following corollary is also proved in
Appendix~\ref{app:GaussianProofs}:

\begin{corollary} \label{cor:matching_means}
If the parameter in the learnable model $\nu$~\eqref{eq:joint_probability_modelI} is chosen to be $A= G^\top H = A^\ast$, Theorem~\ref{thm:Gaussian_solution_cosine} 
results in the approximate conditional distributions:
\begin{subequations}
    \begin{align}
    \pmeas_{\x|\y}(\x|\y;A^\ast) &= \mathcal{N}(\cov_{\x\y}\cov_{\y\y}^{-1}\y, \cov_{\x\x}) \label{eq:cosine_optconditional_xy_unconstrained} \\
    \pmeas_{\y|\x}(\y|\x;A^\ast) &= \mathcal{N}(\cov_{\y\x\y} \cov_{\x\x}^{-1}\x, \cov_{\y\y}), \label{eq:cosine_optconditional_yx_unconstrained}
    \end{align}
\end{subequations} 
Thus, the conditional means of $\pmeas$ in~\eqref{eq:cosine_optconditional_xy_unconstrained} and~\eqref{eq:cosine_optconditional_yx_unconstrained}  match those of the data distribution $\mu$ in~\eqref{eq:true_conditional_xy} and~\eqref{eq:true_conditional_yx}, respectively; however, the conditional variances of $\pmeas$ are strictly larger, in the cone of positive definite matrices, than those of the data distribution $\mu$, unless the data distribution is in product form with respect to $\x$ and $\y$ when they coincide.
\end{corollary}

This corollary shows that the conditional loss will not be consistent if used in a generative sense for sampling from approximate conditional distributions. Sections~\ref{sec:Gaussian_quadraticform} and~\ref{ssec:gjoint} consider a different alignment metric and loss function, respectively, and show that by doing so we can better characterize the approximated covariances. In particular,
Corollary~\ref{cor:matching_meansandcov} shows that it is possible to exactly match one of the conditional covariances. And
Corollary~\ref{corr:marginals_joint_loss} shows that using the joint loss rather than the conditional loss yields a closer approximation to the marginal covariances of the data distribution. 
For further discussion of the material in this subsection see Remarks \ref{rem:G1} and \ref{rem:G2}; there we discuss an alternative formulation of the loss function, and the setting with empirical data.

\subsection{Positive Quadratic Form: Conditional Loss} \label{sec:Gaussian_quadraticform}

Corollary \ref{cor:matching_means} shows that, with learnable measure of the form~\eqref{eq:joint_probability_modelI}, it is only
possible to match the conditional means, not the conditional covariances. To address this we now
consider the following generalization:
\begin{equation} \label{eq:Gaussian_joint_quadratic}
\pmeas(d\x,d\y;\theta) = \frac{1}{Z} \exp\Bigl(-\frac12|G\x-H\y|^2 \Bigr) \mu_\x(\x)\mu_\y(\y).
\end{equation}
The parameters of this model are $\theta=(G,H).$  The measure $\nu$ is a multivariate Gaussian distribution that depends on three matrix products $A = G^\top H \in \R^{\dx \times \dy}$, $B = G^\top G \in \R^{\dx \times \dx}$ and $C = H^\top H \in \R^{\dy \times \dy}$. %
Moreover, the joint measure has the form $\pmeas(\x,\y;\theta) = \mathcal{N}(m_\theta,\cov_\theta)$, where the mean and inverse covariance matrix have the block form
\begin{equation}
m_\theta = \begin{bmatrix} 0 \\ 0 \end{bmatrix}, \qquad \cov_\theta^{-1} = \begin{bmatrix} B + \cov_{\x\x}^{-1} & -A \\ 
-A^\top & C + \cov_{\y\y}^{-1} \end{bmatrix}.
\end{equation}
The parameterized model has the conditional distributions
\begin{subequations} \label{eq:Gaussian_conditionals_quadratic}
    \begin{align}
        \pmeas_{\x|\y}(\x|\y;\theta) &= \frac{1}{Z} \exp\Bigl(-\frac12|G\x-H\y|^2 \Bigr) \mu_\x(d\x) = \mathcal{N}\bigl(d\x;(B + \cov_{\x\x}^{-1})^{-1} A \y, (B + \cov_{\x\x}^{-1})^{-1}\bigr), \\
        \pmeas_{\y|\x}(\y|\x;\theta) &=  \frac{1}{Z} \exp\Bigl(-\frac12|G\x-H\y|^2 \Bigr) \mu_\y(d\y) = \mathcal{N}\bigl(d\y;(C + \cov_{\y\y}^{-1})^{-1} A^\top \x, (C + \cov_{\y\y}^{-1})^{-1}\bigr).
    \end{align}
\end{subequations}
The additional degrees of freedom introduced by using \eqref{eq:Gaussian_joint_quadratic} in place of \eqref{eq:joint_probability_modelI}
enable us to match the mean \emph{and} covariance of either one of the conditional distributions.
To this end we study the optimal solution when the objective only aims to match the conditional distribution for $\x|\y$. Analogous results may be derived for the $\y|\x$ conditional. To achieve this matching we employ Optimization Problem~\ref{opt:KLonesided} with $(\lambda_{\x},\lambda_{\y})=(2,0)$ to obtain the following objective that is minimized to determine the parameters $\theta$:
\begin{equation*}
\LCOND(\theta;2,0) = \mathbb{E}_{\y \sim \mu_\y} \left[-\mathbb{E}_{\x \sim \mu_{\x|\y}(\cdot|\y)}\Bigl(-\frac{1}{2}|G \x - H \y |^2 \Bigr) + \log \mathbb{E}_{\x \sim \mu_\x} \exp\Bigl(-\frac{1}{2}|G \x - H \y |^2\Bigr)\right].%
\end{equation*}
Expanding the quadratic forms, the objective simplifies to the form
\begin{equation*}
\LCOND(\theta;2,0) = \mathbb{E}_{(\x,\y) \sim \mu} \Bigl[\frac{1}{2}|G \x|^2 - \langle u, G^\top H v \rangle \Bigr] + \mathbb{E}_{\y \sim \mu_\y}\Bigl[\log \mathbb{E}_{\x \sim \mu_\x}\exp\Bigl(-\frac{1}{2}|G\x|^2 + \langle \x, G^\top H \y \rangle \Bigr)\Bigr], %
\end{equation*}
which  depends only on the matrix products $A=G^\top H$ and $B=G^\top G$.
Thus, we can consider optimization over parameters $(A,B)$ and hence define
\begin{align*} 
\LCOND(A,B;2,0) \coloneqq \mathbb{E}_{(\x,\y) \sim \mu} \Bigl[\frac{1}{2}\langle \x, B \x \rangle - \langle 
\x, A \y \rangle \Bigr] + \mathbb{E}_{\y \sim \mu_\y} \Bigl[ \log \mathbb{E}_{\x \sim \mu_\x} \exp\Bigl(-\frac{1}{2}\langle \x, B \x \rangle + \langle \x, A \y \rangle \Bigr)\Bigr].%
\end{align*}

Recall, again, that in Subsection~\ref{ssec:notation} we introduce the following notation:
for any matrix $\Sigma$, $\Sigma_r$ denotes the rank$-r$ truncation of its singular value decomposition.
The following theorem presents a closed form for the optimal matrix pair $(A,B)$  without and with rank constraints arising from using an embedding dimension $\de<\min(\dx,\dy)$. The proof may be found in Appendix~\ref{app:GaussianProofs}.

\begin{theorem} \label{thm:Gaussian_quadraticform_onesided}  
The minimizer of $\LCOND(A,B;2,0)$ over all matrices $A$ of size $\dx \times \dy$ and matrices $B$ of size $\dx \times \dx$ is 
\begin{subequations}
    \begin{align}
    A^\ast &= \cov_{\x|\y}^{-1}\cov_{\x\y}\cov_{\y\y}^{-1} \\ %
    B^\ast &= \cov_{\x\x}^{-1}\cov_{\x\y}\cov_{\y|\x}^{-1}\cov_{\y\x}\cov_{\x\x}^{-1}. \label{eq:optB_onesided}
    \end{align}
\end{subequations}
The minimizer of $\LCOND(A,B;2,0)$  
over the rank-constrained sets of matrices
$\mathcal{A}_r = \{A \in \R^{\dx \times \dy}: \text{rank}(A) \leq r\}$ and $\mathcal{B}_r = \{B \in \R^{\dx \times \dx}: \text{rank}(B) \leq r\}$ for $0 < r \leq \min(\dx,\dy)$ is given by %
\begin{subequations} \label{eq:opt_rank_constrained_onesided}
    \begin{align}
    A^*(r) &= (B^*(r) + \cov_{\x\x}^{-1})^{1/2} ((B^*(r) + \cov_{\x\x}^{-1})^{1/2}\cov_{\x\y}\cov_{\y\y}^{-1/2})_r \cov_{\y\y}^{-1/2}, \\
    B^*(r) &= \argmin_{B \in \mathcal{B}_r} \Tr((B + \cov_{\x\x}^{-1})\cov_{\x|\y}) + \log\left|(B + \cov_{\x\x}^{-1})\cov_{\x|\y}\right| \, + \label{eq:opt_onesided_Bproblem} \\
    &\qquad\qquad\quad \left\|\left((B + \cov_{\x\x}^{-1})^{1/2}\cov_{\x\y}\cov_{\y\y}^{-1/2}\right)_r - (B + \cov_{\x\x}^{-1})^{1/2}\cov_{\x\y}\cov_{\y\y}^{-1/2}\right\|_{F}^2 \nonumber.
    \end{align}
\end{subequations}
\end{theorem}

\begin{corollary} \label{cor:matching_meansandcov}
If the parameters $(A,B)$ in the learnable model $\nu$ given by \eqref{eq:Gaussian_joint_quadratic} are chosen to be $(A^\ast,B^\ast)$, the minimizers of $\LCOND(A,B;2,0)$, then the conditional distribution $\nu_{\x|\y}$, which only depends on the parameters $(A,B)$, is given by
\begin{equation}
    \pmeas_{\x|\y}(\x|\y;A^\ast,B^\ast) = \mathcal{N}(\cov_{\x\y}\cov_{\y\y}^{-1}\y, \cov_{\x\x} - \cov_{\x\y}\cov_{\y\y}^{-1}\cov_{\y\x}) \label{eq:cosine_optconditional_xy_unconstrained2},
\end{equation}
which exactly matches the conditional of the data distribution $\mu_{\x|\y}$.
\end{corollary}

The proof establishing this corollary is contained in Appendix~\ref{app:GaussianProofs}. See Remark \ref{rem:G3} for discussion of how to link matrices $A,B$ to $G,H$.

\subsection{Cosine Distance: Joint Loss}
\label{ssec:gjoint}

In this subsection we return to the proposed model form of learnable measure $\nu$ given by~\eqref{eq:joint_probability_modelI}, but
we consider the  Optimization Problem~\ref{do:J2}; recall that the loss function appearing therein is defined by minimizing the KL divergence between the learnable joint measure $\nu$ and the data measure $\mu$, rather than through conditionals. The relevant optimization problem over the parameter $A=G^\top H$ is to minimize the objective function
\begin{equation} \label{eq:JointLossGaussian}
  \LJOINT(A) = -\mathbb{E}_{(\x,\y) \sim \mu}[\langle \x, A \y\rangle] + \log \mathbb{E}_{(\x,\y) \sim \mu_{\x} \otimes \mu_\y}[\exp(\langle \x, A \y \rangle)].
\end{equation}
In what follows the following function, and its properties, will be useful in evaluating the properties of minimizers of $\LJOINT$.
\begin{definition} \label{def:sv_shrinkage}
Define $h\colon(0,1] \to \R^+$ by $h(\sigma)=\sigma^{-1}\Bigl(\frac12(1+4\sigma^2)^{\frac12}-\frac12\Bigr).$
\end{definition}
We note that $\h(\sigma) \in [0,\sigma)$, that $\lim_{\sigma \to 0}\sigma^{-1}h(\sigma)=1$ and that $h(1)=\frac12(\sqrt{5}-1) \in (0,1).$
The following theorem describes the form of the minimizer of the joint loss in closed form; proof may be found in~\Cref{app:GaussianProofs}. Once again, recall that in~\Cref{ssec:notation} we introduce the following notation:
for any matrix $\Sigma$, $\Sigma_r$ denotes the rank-r truncation of its singular value decomposition.

\begin{theorem} \label{thm:Gaussian_joint} 
Let $U\Sigma V^\top$ be the singular value decomposition of $\cov_{\x\x}^{-1/2}\cov_{\x\y}\cov_{\y\y}^{-1/2}$ where $\Sigma$ is a diagonal matrix of size $\min(\dx,\dy) \times \min(\dx,\dy)$. The minimizer of $\LJOINT(A)$ over all matrices of size $\dx \times \dy$ is given by 
\begin{equation} \label{eq:JointMin}
A^\ast %
= \cov_{\x\x}^{-1/2} U h(\Sigma) V^\top \cov_{\y\y}^{-1/2},
\end{equation}
where $h$ in Definition~\ref{def:sv_shrinkage} is applied elementwise to the diagonal entries of $\Sigma$.  %
Furthermore, the minimizer of $\LJOINT(A)$ over the rank-constrained set of matrices $\mathcal{A}_r = \{A \in \R^{\dx \times \dy}: \text{rank}(A) \leq r\}$ for $0 < r \leq \min(\dx,\dy)$ is given by
\begin{equation} \label{eq:JointMin_rankconstrained}
A^*(r) = \cov_{\x\x}^{-1/2} (U h(\Sigma) V^\top)_r \cov_{\y\y}^{-1/2}.
\end{equation}
\end{theorem}

The following corollary shows that the approximation to the joint data distribution resulting from  Theorem~\ref{thm:Gaussian_joint} has marginal distributions that are closer to the marginals of the data distribution $\mu$ than the approximation implied by minimizing the conditional losses in \Cref{sec:Gaussian_cosine}.
Without loss of generality, we study the approximation to marginal for $\x$.
An analogous result may be stated for the marginal on $\y$, by symmetry. The proof of the corollary can be found in Appendix~\ref{app:GaussianProofs}.

\begin{corollary} \label{corr:marginals_joint_loss}
Recall that the true marginal distribution $\mu_\x$ is the centred Gaussian with covariance $\mathcal{N}(0, \cov_{\x\x}).$ Consider the marginal distribution
implied with the model form \eqref{eq:joint_probability_modelI} under the (conditional) loss function \eqref{eq:objective_linear_encoder},
denoted $\nu_{\x}(d\x;A^*_{\mathsf{cond}})$, and under the (joint) loss function 
\eqref{eq:JointLossGaussian}, denoted  $\nu_{\x}(d\x;A^*_{\mathsf{joint}}).$ Then
\begin{subequations} 
    \begin{align} 
    \nu_{\x}(d\x;A^*_{\mathsf{cond}}) &= \mathcal{N}(0, \cov_{\x\x}^{1/2}U\bigl(I_{\dx} - \Sigma^2)^{-1}U^\top\cov_{\x\x}^{1/2}\bigr), \label{eq:marginal_dist_approx_cond}\\
    \nu_{\x}(d\x;A^*_{\mathsf{joint}}) &= \mathcal{N}(0, \cov_{\x\x}^{1/2}U\bigl(I_{\dx} - h(\Sigma)^2)^{-1}U^\top\cov_{\x\x}^{1/2}\bigr), \label{eq:marginal_dist_approx_joint}
    \end{align}
\end{subequations}
where $U$ and $\Sigma$ are comprised of the left singular vectors and singular values of $\cov_{\x\x}^{-1/2}\cov_{\x\y}\cov_{\y\y}^{-1/2}$. 
In view of the stated properties of function $h$, this shows that minimizing the joint loss results in a marginal distribution on $\x$ which is closer to the
true marginal distribution than that obtained by minimizing the conditional loss.
\end{corollary}

\subsection{Numerical Illustrations} \label{sec:Gaussian_visualization}

In this section, we consider a two-dimensional Gaussian target distribution $\mu$ and visualize the approximations that result from various contrastive learning problems. The target distribution we consider is $\mu = \mathcal{N}(\mean,\cov)$ with mean $m$ and covariance matrix $\cov$ given by $$\mean = \begin{bmatrix} 0 \\ 0 \end{bmatrix}, \qquad \cov = \begin{bmatrix} 1.5 & 1 \\ 1 & 1.5\end{bmatrix}.$$ 
Here $u,v \in \R$ are the two components of the vector in $\R^2$ governed by Gaussian $\mu.$ Because of the low dimensionality of the example the latent space
also has dimension $\de=1.$ (In Section \ref{sec:experiments} we will consider Gaussian numerical examples where the embedding dimension is smaller than the dimension of the spaces in which $u$ and $v$ lie.)
In Figure~\ref{fig:visualize-data} we plot the reference measure $\mu_{\x} \otimes \mu_\y$, used to define various contrastive learning models through tilting, and the target distribution $\mu$.

\begin{figure}[!ht]
\centering
\includegraphics[width=0.7\textwidth]{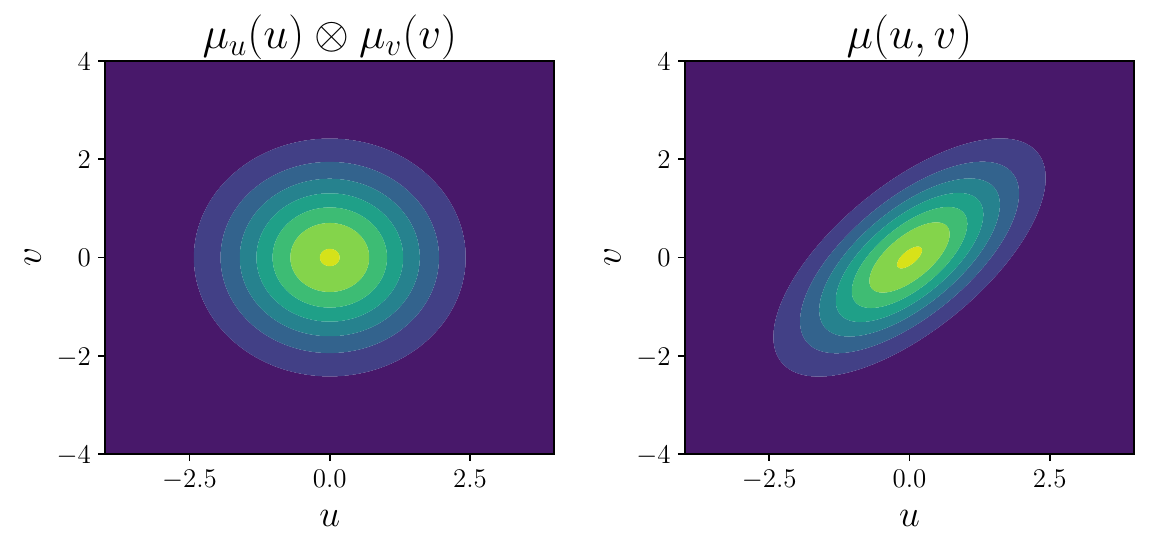}
\vspace{-0.3cm}
\caption{Two-dimensional densities for the reference Gaussian distribution given by the product of marginals $\mu_{\x} \otimes \mu_\y$ (\emph{left}) and the target Gaussian distribution $\mu$ (\emph{right}). CLIP aims to learn a model that tilts the reference distribution (on the left) to match the target distribution (on the right). 
\label{fig:visualize-data}}
\end{figure}

First, we investigate the approximate distributions that result from the learning problems considered in Sections~\ref{sec:Gaussian_cosine}-\ref{ssec:gjoint} for various alignment metrics and objective functions. Figures~\ref{fig:visualize-marg_cosine_cond}-\ref{fig:visualize-marg_cosine_joint} plot the conditional and marginal densities of the models $\nu(\cdot,\cdot;\theta^*)$ corresponding to the optimal parameter $\theta^*$ for the learning problems arising from: (i) the cosine distance with the two-sided conditional loss; (ii) the positive quadratic form with the one-sided conditional loss;  and (iii) the cosine distance with the joint loss, respectively. As expected from Theorem~\ref{thm:Gaussian_solution_cosine}, we observe that the cosine alignment correctly identifies the two conditional means for $\mu_{\x|\y}$ and $\mu_{\y|\x}$ in Figure~\ref{fig:visualize-marg_cosine_cond}. However, the model structure produces Gaussian conditionals whose covariances follow the marginal covariance of the reference distribution, which have larger variance in all directions than the true conditional covariances. With the additional parameter in the positive quadratic form, the model correctly captures both the conditional means and variance for the $\x|\y$ variable in Figure~\ref{fig:visualize-marg_quadratic_cond}, but the one-sided loss does not accurately describe the moments for the $\y|\x$ variable; see Theorem~\ref{thm:Gaussian_quadraticform_onesided}. Lastly, using the cosine alignment metric with the joint loss results in a closer match to the true marginal covariances of the data distribution, as expected from Theorem~\ref{thm:Gaussian_joint}.

In Figure~\ref{fig:visualize-twodim_PDFs2}, we plot the resulting approximation to the data distribution. As compared to the target $\mu$, we observe that using the cosine distance with the conditional loss results in a joint distribution with inflated variance along both variables. With the positive quadratic form, the variance for $u|v$ is reduced, while the $v|u$ conditional is not correctly specified when using the one-sided loss. Lastly, the joint loss results in the closest variance in both variables to the joint distribution.

\begin{figure}[!ht]
\centering
\includegraphics[width=0.9\textwidth]{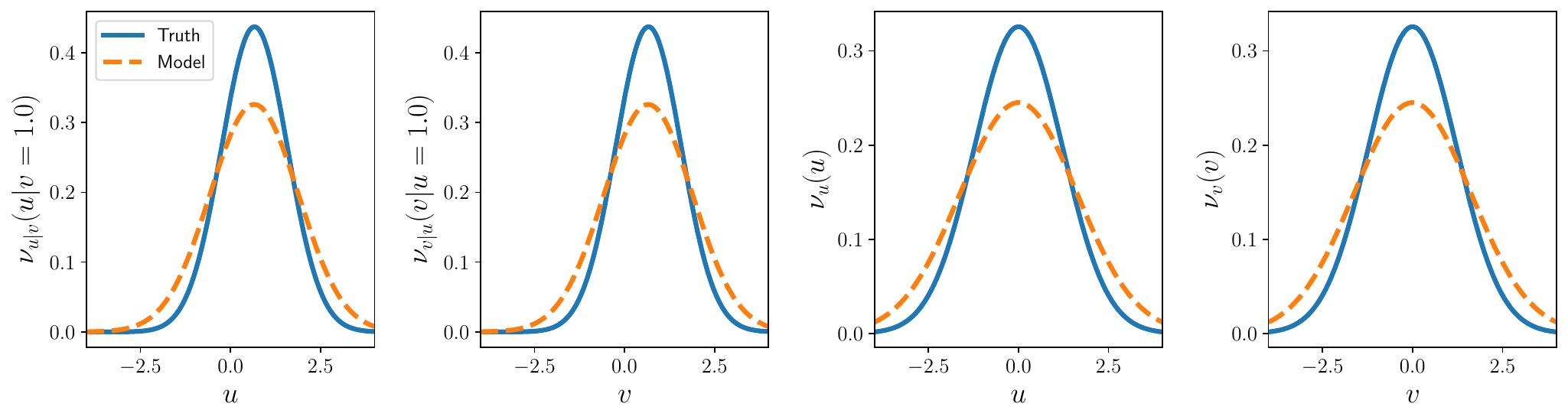}
\vspace{-0.3cm}
\caption{Densities for the conditional distributions $\nu_{\x|\y
},\nu_{\y|\x}$ and the marginals $\nu_\x,\nu_\y$ resulting from the cosine distance with the two-sided conditional loss. The conditional means are matched, but not the variance; see Theorem~\ref{thm:Gaussian_solution_cosine}.}
\label{fig:visualize-marg_cosine_cond}
\end{figure}
\begin{figure}[!ht]
\centering
\includegraphics[width=0.9\textwidth]{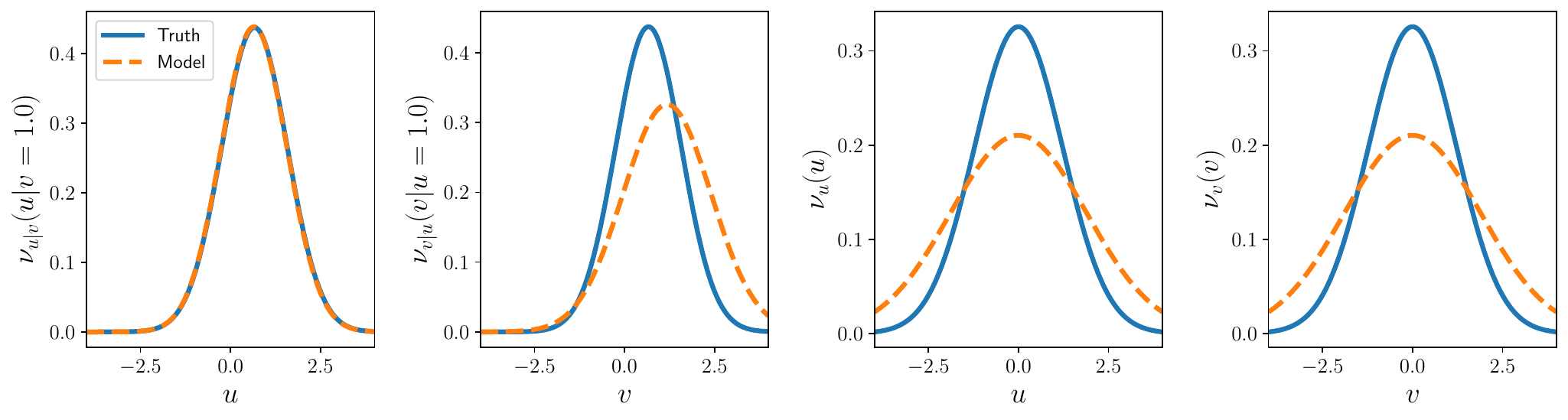}
\vspace{-0.3cm}
\caption{Densities for the conditional distributions $\nu_{\x|\y
},\nu_{\y|\x}$ and the marginals $\nu_\x,\nu_\y$ resulting from the cosine distance with the joint loss. The mean and variance for the $\x|\y$ conditional is matched, but not for $\y|\x$; see Theorem~\ref{thm:Gaussian_quadraticform_onesided}.}
\label{fig:visualize-marg_quadratic_cond}
\end{figure}
\begin{figure}[!ht]
\centering
\includegraphics[width=0.9\textwidth]{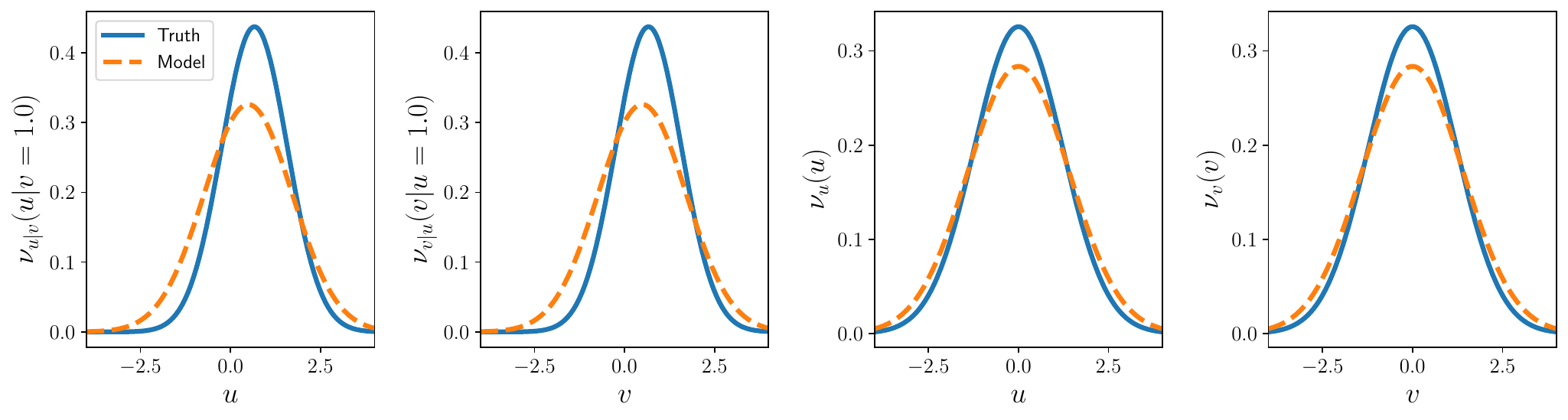}
\vspace{-0.3cm}
\caption{Densities for the marginals $\nu_\x,\nu_\y$ and the conditional distributions $\nu_{\x|\y
},\nu_{\y|\x}$ resulting from the cosine distance with the joint loss. The marginal variances are better approximated, but a bias is introduced in the conditional means; see Theorem~\ref{thm:Gaussian_joint}.}
\label{fig:visualize-marg_cosine_joint}
\end{figure}

\begin{figure}[!ht]
\centering
\includegraphics[width=0.32\textwidth]{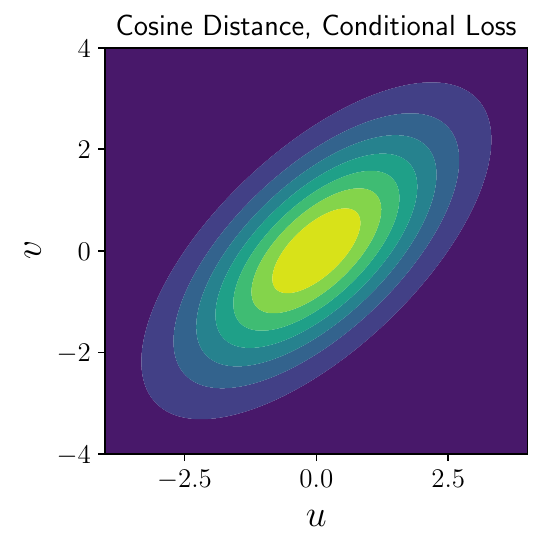}
\includegraphics[width=0.32\textwidth]{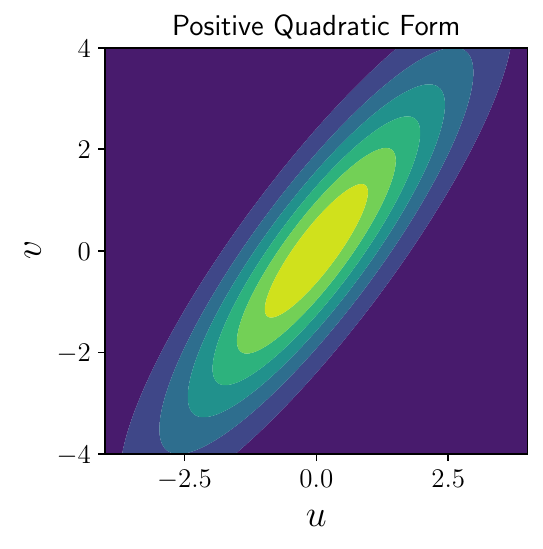}
\includegraphics[width=0.32\textwidth]{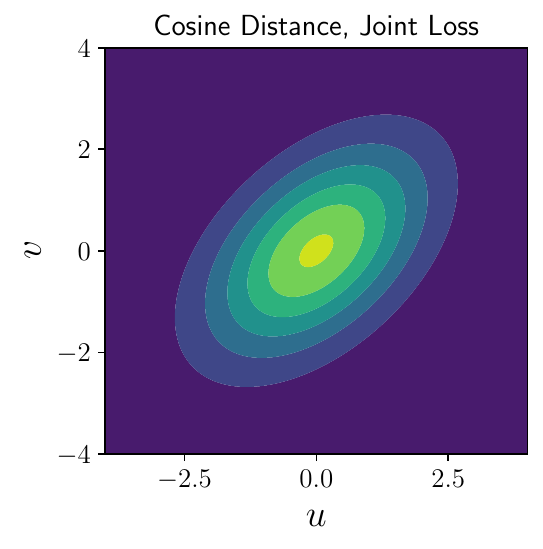}
\vspace{-0.3cm}
\caption{Two-dimensional densities for the distributions $\nu$ learned with various alignment metrics and objective functions: cosine distance with the conditional loss from \Cref{sec:Gaussian_cosine} (\emph{left}), the positive quadratic form with the conditional loss for matching the $u|v$ conditional from \Cref{sec:Gaussian_quadraticform} (\emph{middle}) and the cosine distance with the joint loss from Section~\ref{ssec:gjoint} (\emph{right}). These densities provide different approximations of the joint distribution on the right of Figure~\ref{fig:visualize-data}. Unsurprisingly, the joint loss provides the closest approximation to the  target distribution. 
\label{fig:visualize-twodim_PDFs2}}
\end{figure}

Lastly we go beyond the Gaussian setting and consider the effect of normalizing the encoders. We consider the approximation resulting from normalized encoders that are defined by parameters $\theta_\x \coloneqq g \in \R$ and $\theta_\y \coloneqq h \in \R$:
$$\normg_\x(\x;\theta_\x) = \frac{g\x}{|g\x|}, \qquad \normg_\y(\y;\theta_\y) = \frac{h\y}{|h\y|}.$$
Note that $\normg_\x(\x;\theta_\x),  \normg_\y(\y;\theta_\y) \in \bS^0:=\{\pm 1\}.$
 Due to the nonlinearity of $\normg_\x(\cdot\,;\theta_\x)$ and $\normg_\y(\cdot\,;\theta_\y)$ 
the resulting learned joint and conditional distributions are non-Gaussian; this is
illustrated in  Figure~\ref{fig:twodim_dist_comparisons_normalized}. In this case, the model captures the positive correlation of the $(\x,\y)$ variables, which is observed in the approximation of the joint distribution. The expressiveness of the normalized models tilts the marginal distributions by weighting the probability mass to the left and right side of the origin. As a result of the symmetry of the joint target distribution and the low-dimensional embedding, we note that the marginal distributions for this example are captured exactly with the learned model, but we do not present this result because this property is not guaranteed in arbitrary dimensions.

\begin{figure}[!ht]
\centering
\includegraphics[width=0.5\textwidth]{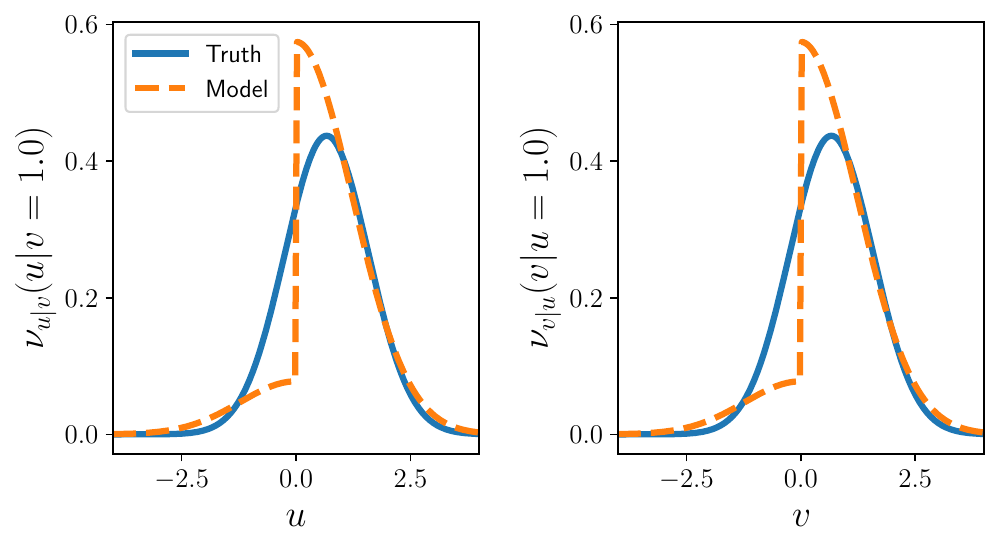}
\includegraphics[width=0.28\textwidth]{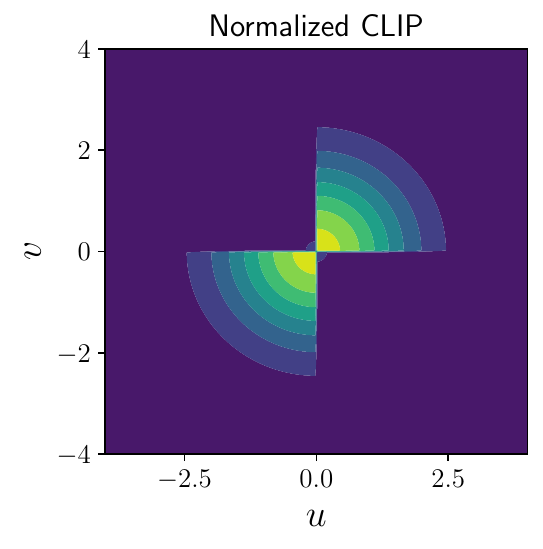}
\vspace{-0.3cm}
\caption{One-dimensional densities for the conditional distributions $\nu_{\x|\y},\nu_{\y|\x}$ (\emph{left}) and the two-dimensional joint distribution $\nu$ (\emph{right}) that is learned with the cosine metric and the normalized encoder models.~\label{fig:twodim_dist_comparisons_normalized}}
\end{figure}

\section{Numerical Experiments} \label{sec:experiments} 

In this section we present numerical experiments that complement the material presented in the previous sections.
Subsection \ref{ssec:61} is devoted to a high dimensional Gaussian example where an exactly computable true joint distribution
may be used to evaluate sensitivity of CLIP to parameters such as embedding dimension, size of the data set and batch size used in training.
In Subsection \ref{ssec:lecun} we show how classification with the MNIST data set may be formulated as a specific generalization
of the basic contrastive learning set-up, by changing the embeddings and working with a one-sided loss.
Subsection \ref{ssec:fluid_flow} demonstrates how a problem in Lagrangian data assimilation may be addressed by use of contrastive learning.

\subsection{High-dimensional Gaussian} \label{ssec:61}

In this section we consider a pair of data modalities that are different linear projections of a Gaussian process, and hence themselves
follow a multivariate joint distribution. As in \Cref{sec:Gaussian_cosine}, the conditionals of this distribution are known analytically, allowing us to evaluate the learned model approximations. Let $D=(0,1)$ and $C$ be a trace-class covariance operator on $\cR:=L^2(D;\R).$ Denote the eigenpairs of $C$ by
$(\psi_j,\lambda_j)$ with non-negative eigenvalues ordered to be decreasing in $j \in \{1,2,3,\dots \}.$ 
From this we may define draws from $w \sim \mathcal{N}(0,C)$  via the Karhunen-Lo\'{e}ve decomposition
\begin{equation}
    \label{eq:KLo}
\rr(x) = \sum_{j=1}^\infty \sqrt{\lambda_j} \xi_j \psi_j(x).
\end{equation}
Let $\Delta$ denote the Laplacian equipped with homogeneous Neumann boundary conditions on $D$ and viewed as acting on functions in $H^2(D;\R)$ with mean zero over $D$. We assume that $C = (-\Delta + \tau^2 I)^{-\alpha}$.
Then, for all positive integers $j$, $\psi_{j}(x) = \cos(j\pi x)$ and $\lambda_j = (j^2\pi^2 + \tau^2)^{-\alpha}$, where $\tau$ represents an inverse length scale and $\alpha$ defines the regularity of the process. We note the index in~\eqref{eq:KLo} starts at $j = 1$ to ensure that  $\rr$ has zero mean when integrated over $D$. In our experiments we set the hyper-parameters to $\tau = 3$ and $\alpha = 2$.

The two data modalities are linear projections of the Gaussian process given by the noisy pointwise evaluation of $w$ at $\dx$ uniformly-spaced grid locations $(x_1,\dots,x_{\dx}) \in D$ and the first $\dy$ random elements in its  orthonormal basis. That is, 
\begin{subequations} \label{eq:Gaussian_linear_transformations}
    \begin{align}
        \x &= (\rr(x_1),\dots, \rr(x_{\dx})) + \eta_\x \in \cX:=\bR^{\dx}, \\
        \y &= \left(\frac{1}{\sqrt{\lambda_\x}} \int \rr(x)\psi_1(x) dx, \dots, \frac{1}{\sqrt{\lambda_{\dy}}} \int \rr(x)\psi_{\dy}(x) dx \right) \in 
         \cY \coloneqq \bR^{\dy},
    \end{align}
\end{subequations}
where $\eta_{\x} \sim \mathcal{N}(0, \sigma^2 I_{\dx})$. In our experiments we set $\dx = 12$ and $\dy = 5$, to emphasize an asymmetry that arises in practice between the information content of the two modalities, and employ noise standard deviation $\sigma = 0.05$. We consider the true random process based on a truncation of the
Karhunen-Lo\'{e}ve expansion \eqref{eq:KLo} to $1000$ modes. Figure~\ref{fig:Gaussian_true_conditionals} displays 400 independent realizations of the underlying true process $w$, conditional on one realization of $\y$, as well as the conditional distribution for each of the two data modalities, $\mu_{\x|\y}(\cdot|\y)$ and $\mu_{\y|\x}(\cdot|\x)$, given one realization of $\y$ and $\x,$ respectively. For the linear transformations of the Gaussian process $\rr$ in \eqref{eq:Gaussian_linear_transformations}, the conditional distributions of each modality are also Gaussian. Thus, the mean and covariance of these conditionals are computable in closed-form using the expressions in~\eqref{eq:Gaussian_conditionals} and hence the mean and confidence intervals capturing 95\% of the conditional probability mass are also plotted in Figure~\ref{fig:Gaussian_true_conditionals}. 

\begin{figure}
\centering
\includegraphics[width=\textwidth]{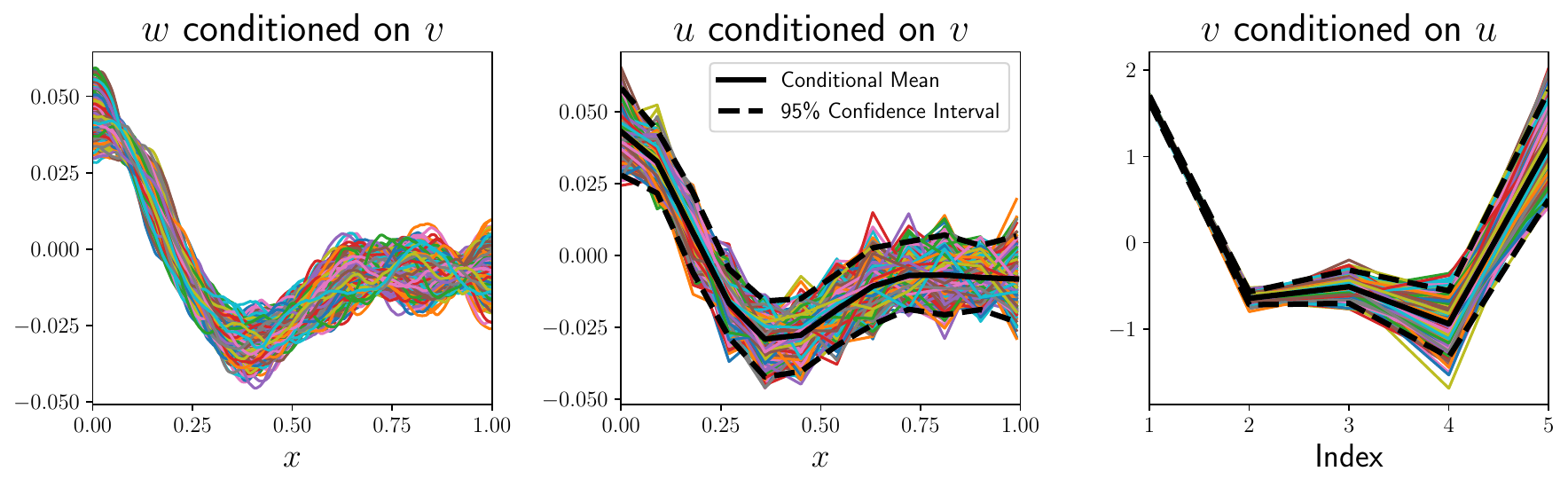}
\caption{Conditional distributions for the Gaussian process $\rr$ (\emph{left}) and the noisy observations $\x$ (\emph{middle}) conditioned on a realization of the coefficients $\y$ as well as the conditional distribution for $\y$ conditioned on a realization of the noisy observation $\x$ (\emph{right}). For each conditional, we plot 400 samples from the conditional distribution, the conditional mean and a 95\% confidence interval.} \label{fig:Gaussian_true_conditionals}
\end{figure}

To define a model for the joint distribution $\mu(\x,\y)$ as a tilting of the marginals $\mu_\x \otimes \mu_\y$, we choose linear encoders with the form in~\eqref{eq:LE}. The encoders are not normalized, in order to exactly capture the true Gaussian conditional expectations for each modality, as described in~\Cref{sec:GS}. The encoders are learned using the cosine alignment metric with the two-sided conditional loss. We approximate the loss using $N$ i.i.d.\thinspace samples from the joint distribution $\mu$, which is estimated at every step of the optimizer with a small batch of samples to build the empirical loss function as in~\Cref{r:wgwan}. The loss is minimized using the Adam optimizer with a learning rate of $10^{-4}$ over 500 epochs. In this subsection, we study the effect of the choice of embedding dimension on the approximation of the conditional expectations. We also study  various practical approximations introduced in the optimization problem including the number of data samples and the batch size.

First, we study the effect of increasing the embedding dimension of the model with $N = 10,000$ training samples and a fixed batch size of $512$. Figure~\ref{fig:Gaussian-means_vs_embedding_dim} plots the estimated conditional expectations for both data modalities $\mathbb{E}[\x|\y]$ and $\mathbb{E}[\y|\x]$ for the same realization of the conditioning variables of $\y$ and  $\x$, respectively, as in Figure~\ref{fig:Gaussian_true_conditionals}. We observe that increasing the embedding dimension up to the dimension $\de = \min(\dx,\dy) = 5$  improves the approximation to the true conditional expectations in black. As expected from Theorem~\ref{thm:Gaussian_solution_cosine}, the solution does not improve for larger embedding dimensions with a fixed total number of samples $N$ and batch size. 

\begin{figure}[!ht]
    \centering
    \includegraphics[width=0.24\textwidth]{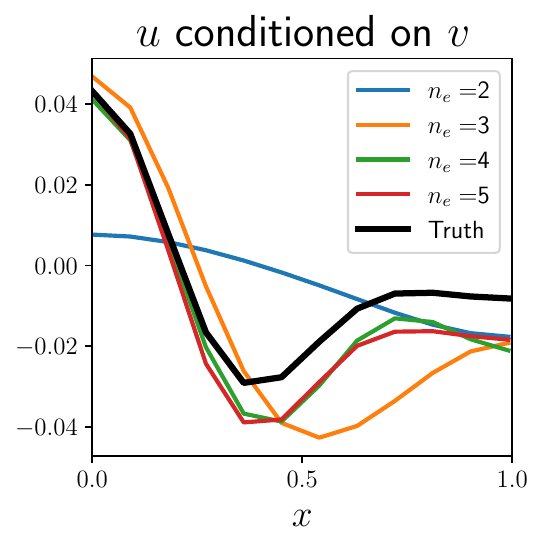}
    \includegraphics[width=0.24\textwidth]{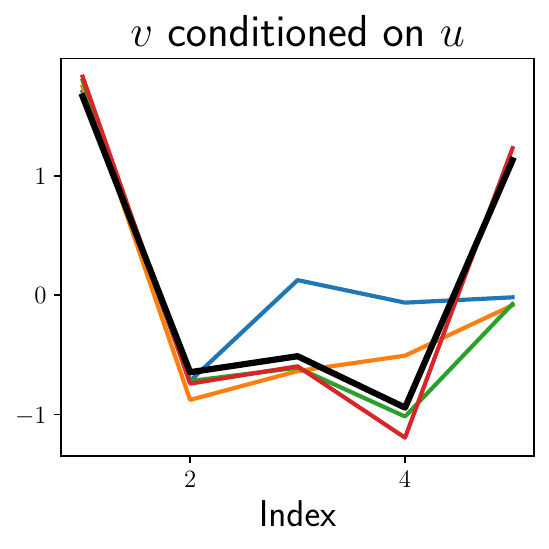}
    \includegraphics[width=0.24\textwidth]{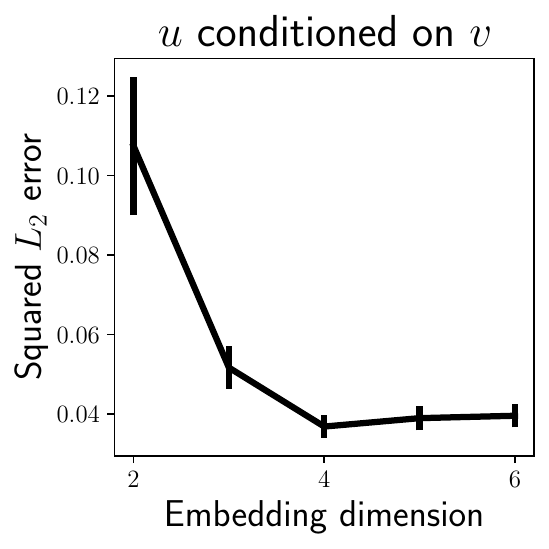}
    \includegraphics[width=0.24\textwidth]{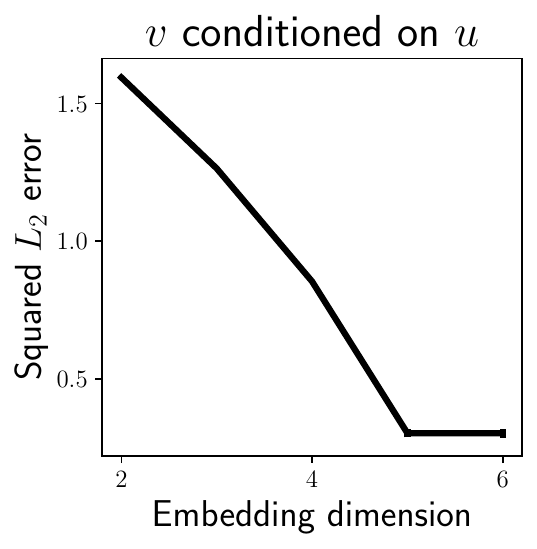}
    \vspace{-0.3cm}
    \caption{Left: Conditional expectations $\mathbb{E}[\x|\y]$ and $\mathbb{E}[\y|\x]$ with increasing embedding dimension for a fixed realization of $\y$ and $\x$, respectively. Right: Squared errors in the conditional expectations in expectation over $10^4$ realizations of the conditioning variables. We observe convergence of the conditional expectations toward the truth when increasing the embedding dimension.}
    \label{fig:Gaussian-means_vs_embedding_dim}
\end{figure}

Next, we study the effect of increasing the batch size during training with a fixed embedding dimension of $\de = 5$ and $N = 10,000$ training samples; see Figure~\ref{fig:Gaussian-MeanErr_vs_batchsize}. We observe that the errors for both conditionals converge to zero as the batch size increases, which reduces the bias in the finite sample estimator given in~\eqref{eq:LN} as an approximation of the population objective in~\eqref{eq:consistent_loss}.
Note that the minimizer of the population loss yields the true conditional means; see Theorem~\ref{thm:Gaussian_solution_cosine}.

\begin{figure}[!ht]
    \centering
    \includegraphics[width=0.24\textwidth]{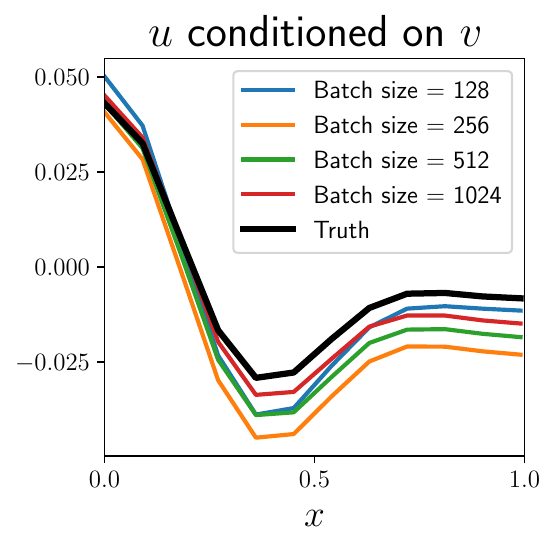}
    \includegraphics[width=0.24\textwidth]{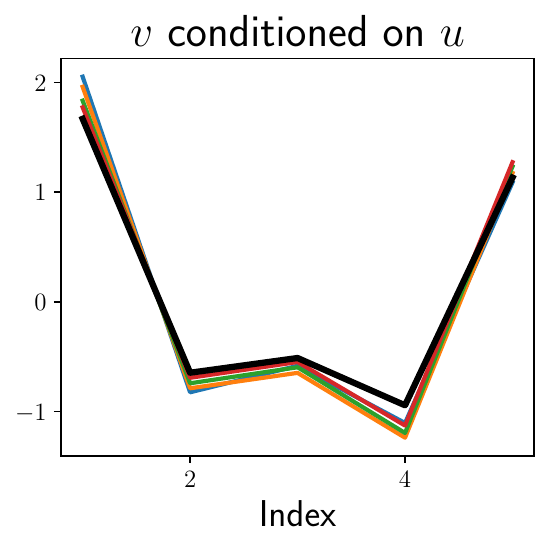}
    \includegraphics[width=0.24\textwidth]{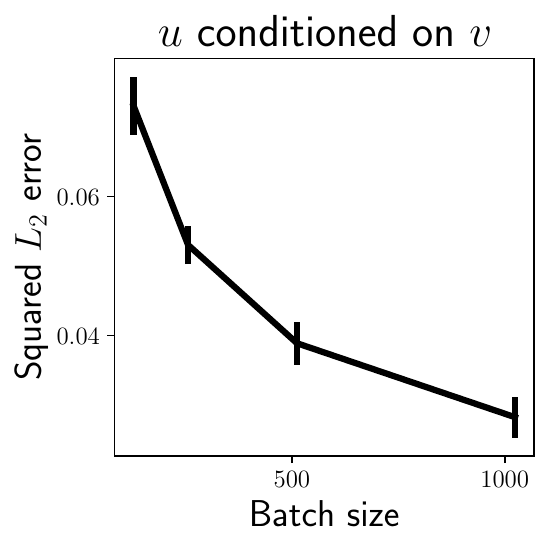}
    \includegraphics[width=0.24\textwidth]{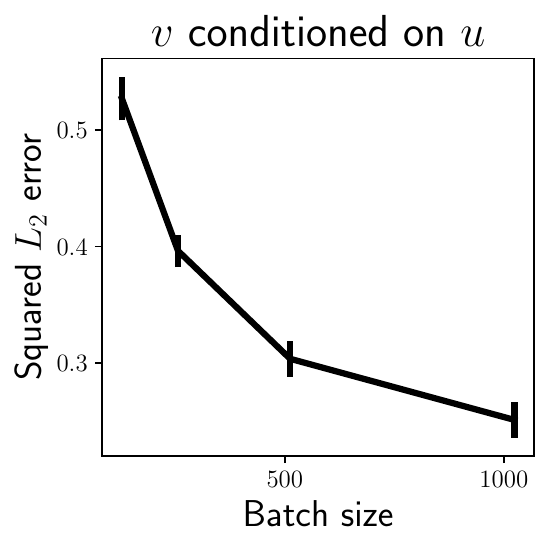}
    \vspace{-0.3cm}
    \caption{Conditional expectations with increasing batch size. We observe that the conditional expectations converge toward the truth when increasing the batch size.}
    \label{fig:Gaussian-MeanErr_vs_batchsize}
\end{figure}

Lastly, we study the effect of increasing the training sample size $N$ in Figure~\ref{fig:Gaussian-MeanErr_vs_samplesize} with a fixed batch size of 512 and embedding dimension set to $\de = 5$. We observe similar convergence of the expectations to the true expectations for $\x$ and $\y$ in expectation over $10^4$ realizations of the conditioning variables.

\begin{figure}[!ht]
    \centering
    \includegraphics[width=0.24\textwidth]{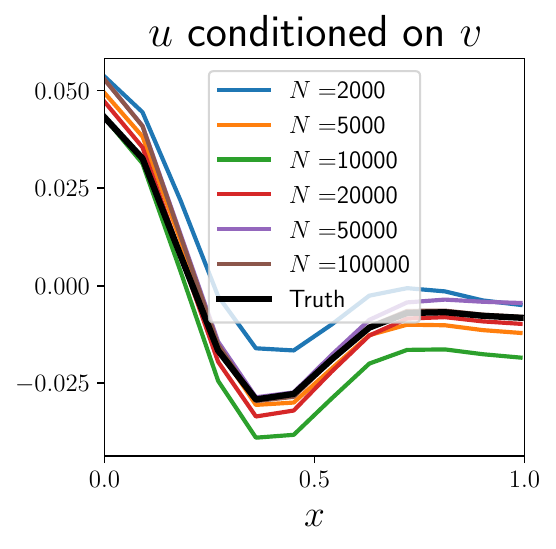}
    \includegraphics[width=0.24\textwidth]{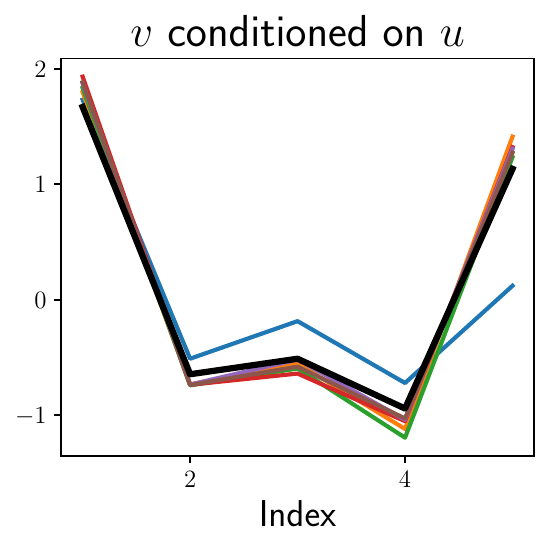}
    \includegraphics[width=0.24\textwidth]{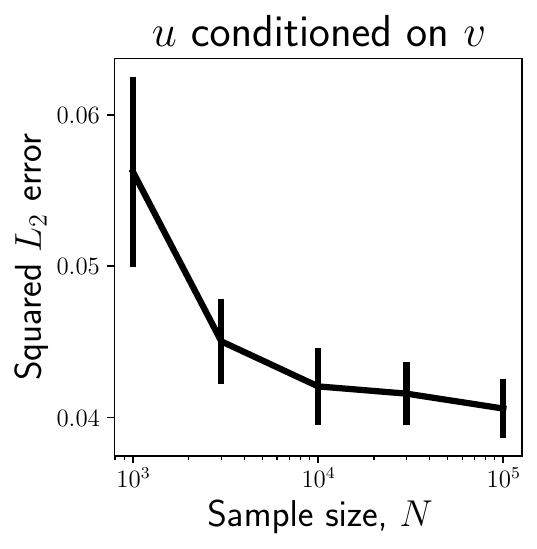}
    \includegraphics[width=0.24\textwidth]{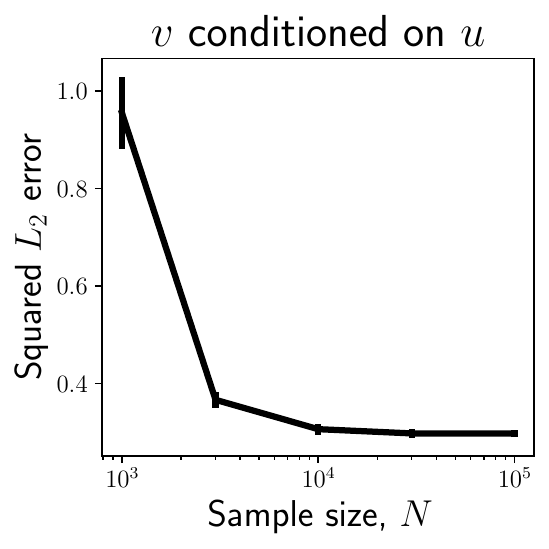}
    \vspace{-0.3cm}
    \caption{Left: Conditional expectations $\mathbb{E}[\x|\y]$ and $\mathbb{E}[\y|\x]$ with increasing sample size for a fixed realization of $\y$ and $\x$, respectively. Conditional expectations with increasing sample size. Right: Squared errors in the conditional expectations in expectation over $10^4$ realizations of the conditioning variables. We observe convergence of the conditional expectations toward the truth when increasing the sample size.}
    \label{fig:Gaussian-MeanErr_vs_samplesize}
\end{figure}

\subsection{The MNIST Data Set}
\label{ssec:lecun}

In this experiment, we consider our two data modalities to be images of MNIST digits $\x \in \mathcal{\X} = [0,1]^{28 \times 28}$ and their corresponding 
labels $\y \in \mathcal{\Y} = \{0,\dots,9\}$ determined by the prescribed labelling of the dataset. Our goal is to identify encoders $g_\x$ and $g_\y$, acting on $\cU$ and $\cV$ respectively, to jointly characterize the conditional distributions of images given labels $\mu_{\x|\y}$ and labels given images $\mu_{\y|\x}$. We note that while $\mu_{\x|\y}(\cdot|\y)$ is supported on a continuous space for each $\y$, $\mu_{\y|\x}(\cdot|\x)$ is a categorical distribution for each $\x$.  

In order to relate representation learning to a classification problem, we choose the image encoder $g_\x$ to be a feed-forward neural network whose outputs represent the log-probabilities associated to each label in $\mathcal{\Y}$; and we choose $g_{\y}(\y) = e_{\y}$ to be the one-hot encoding of the label $\y$. In this case, the embedding space dimension of the two encoders equals the number of labels, i.e., $\de = |\cY|$. In particular, we use a LeNet architecture for $g_\x \colon \mathcal{\X} \rightarrow \R^{|\cY|}$ consisting of convolution and linear layers with ReLu activation functions and set $\theta_\x$ as the weights and biases of these layers. In a classification task, the learned map $g_\x$ is used to predict a label for each input image $\x$, which is given by the output $g_\x(\x) \in \R^{|\cY|}$ with the largest log-probability. What we have just outlined is exactly the architecture used in the seminal paper \cite{lecun1998gradient}. In that paper an optimization problem is defined by minimizing the cross-entropy loss function $\LMNIST(\cdot).$

\begin{proposition} Minimizing $\LMNIST(\theta)$ with respect to $\theta$ is equivalent to minimizing the one-sided conditional loss $\JCONDD(\theta;2,0)$ using the unnormalized embedding
$g_{\x}$ and the one-hot encoding $g_{\y}$  as defined in the preceding paragraph.
\end{proposition}

Next, we show the results of learning the model parameters with three choices of loss functions between conditionals. First, we consider the standard classification methodology leading to $\LMNIST(\theta)$, which (by the preceding proposition) corresponds to minimizing the one-sided conditional loss $\JCONDD(\theta;2,0)$. This in turn is equivalent to minimizing the loss function $\LCONDD(\theta;2,0)$ given by
\begin{align*}
\LCONDD(\theta;2,0)   
    &= -\mathbb{E}_{(\x,\y) \sim \mu} \Bigl[\langle g_{\x}(\x;\theta_\x), g_{\y}(\y) \rangle \Bigr]  + \mathbb{E}_{\y \sim \mu_\y}\log \mathbb{E}_{\x \sim \mu_\x} \Bigl[\exp(\langle g_{\x}(\x;\theta_\x), g_{\y}(\y) \rangle) \Bigr],\\
     \theta^* & = \arg\min_{\theta \in \R^p} \LCONDD(\theta;2,0).
\end{align*}
This loss seeks the encoder parameters $\theta_\x$ to best match the conditional distribution for labels given a candidate image $\mu_{\y|\x}$. Secondly, we consider  minimizing the loss function $\LCONDD(\theta;0,2)$ that matches the conditional distribution $\mu_{\x|\y}$ for images given a label. And finally
we consider the loss function that aims to match both conditionals, $\LCONDD(\theta;1,1)$. In all three cases we use unnormalized encoders as defined in the
paragraph preceding the proposition. We learn the model parameters with $N = 60,000$ training samples of paired images and labels. We learn the model parameters using the Adam optimizer by training for 300 epochs with a batch size of $512$ and a learning rate of $10^{-4}$. %

First, we compare the approximate models for the conditional distribution of labels given images $\pmeas_{\y|\x}$. Figure~\ref{fig:MNIST_label_predictions} plots the results of mapping three images to a label using the three
different learned models. We observe that the one-sided conditional loss on label given image is most accurate for predicting the true label, while the two other losses yield predicted distributions with non-zero probabilities on potential labels (for example $8$ as well as $3$ for the first digit) that are consistent with the image from subjective interpretations.
\begin{figure}
\centering
\includegraphics[width=0.8\textwidth]{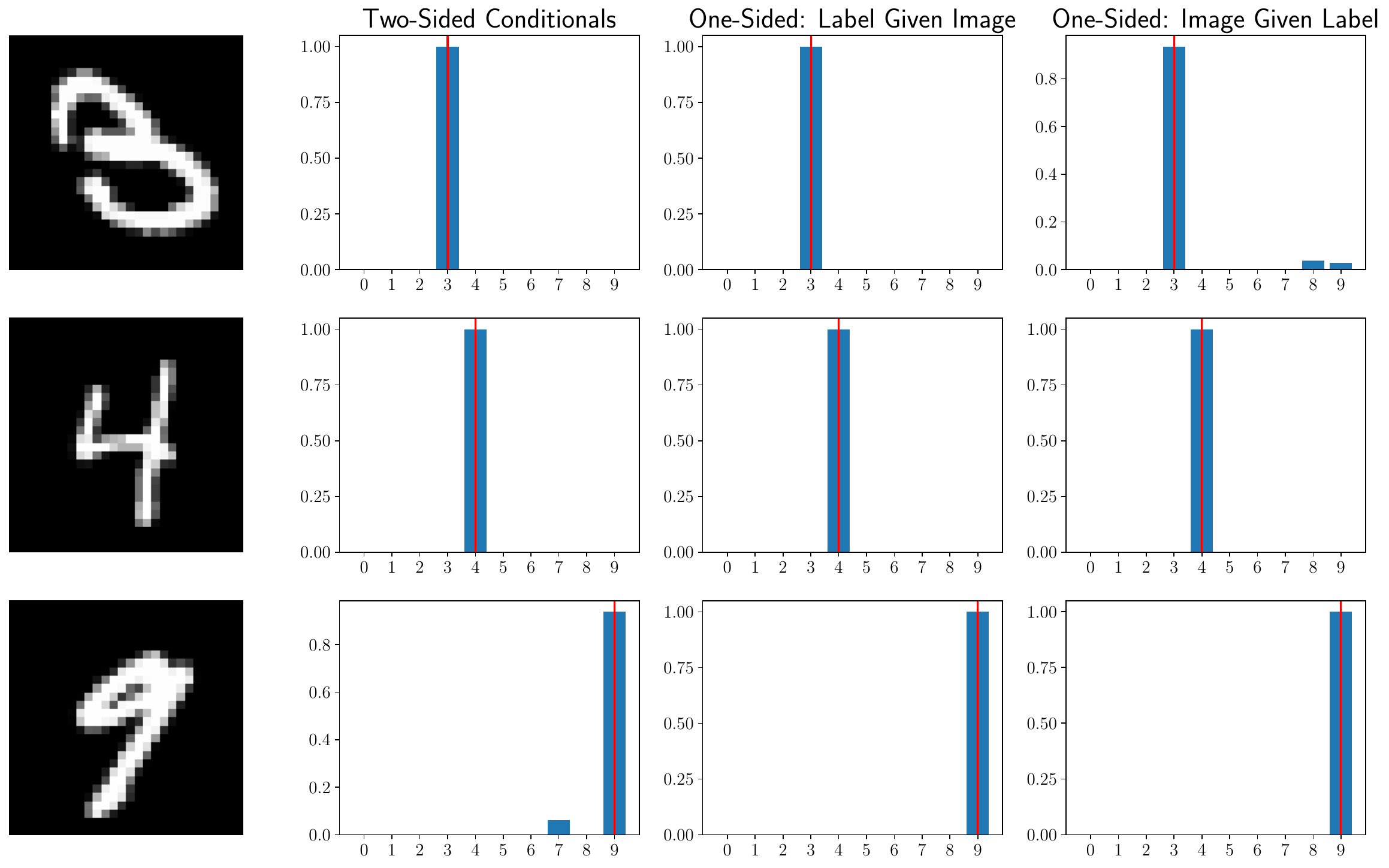}
\vspace{-0.3cm}
\caption{Predicted probabilities of the image encoder $g_\x$ relative to the true label for three candidate images (rows) from the test set using three loss functions.} \label{fig:MNIST_label_predictions}
\end{figure}

Second, we compare the approximate models for the conditional distribution of images given each label, namely $\pmeas_{\x|\y}$. Given an empirical marginal distribution over $N$ images $\mu_\x^N$, the model in~\eqref{eq:condu2} defines a weighted empirical distribution $\pmeas_{\x|\y}^N(\cdot|\y;\theta_\x)$ over the $N$ images with the weights given by the normalized exponential of the cosine similarities $\langle g_{\x}(\x;\theta_\x), g_{\y}(\y) \rangle$ for a given label $\y$.  Figure~\ref{fig:MNIST_cosines} plots 16 images sampled with replacement from the conditional probability distribution of the learned model for each loss function for the label $\y = 7$ with $\Nt =10,000$ test samples. We note that the test set consists of images for all 10 digits. While all sampled images are correctly associated with the prescribed label, we observe that only the distribution learned using the two-sided conditional loss (left plot) and the one-sided conditional loss for images given label (right plot) show a diverse set of images with multiple styles. Instead, the model learned with the one-sided loss for labels given image, that is more accurate at predicting class probabilities in Figure~\ref{fig:MNIST_label_predictions}, displays many repeated images of the digit 7. This shows that the conditional distribution, with this choice of loss function, has its probability mass concentrated on a single image; such a distribution  does not accurately capture the true data distribution consisting of different images that are consistent with the label.
\begin{figure}
\centering
\includegraphics[width=0.85\textwidth]{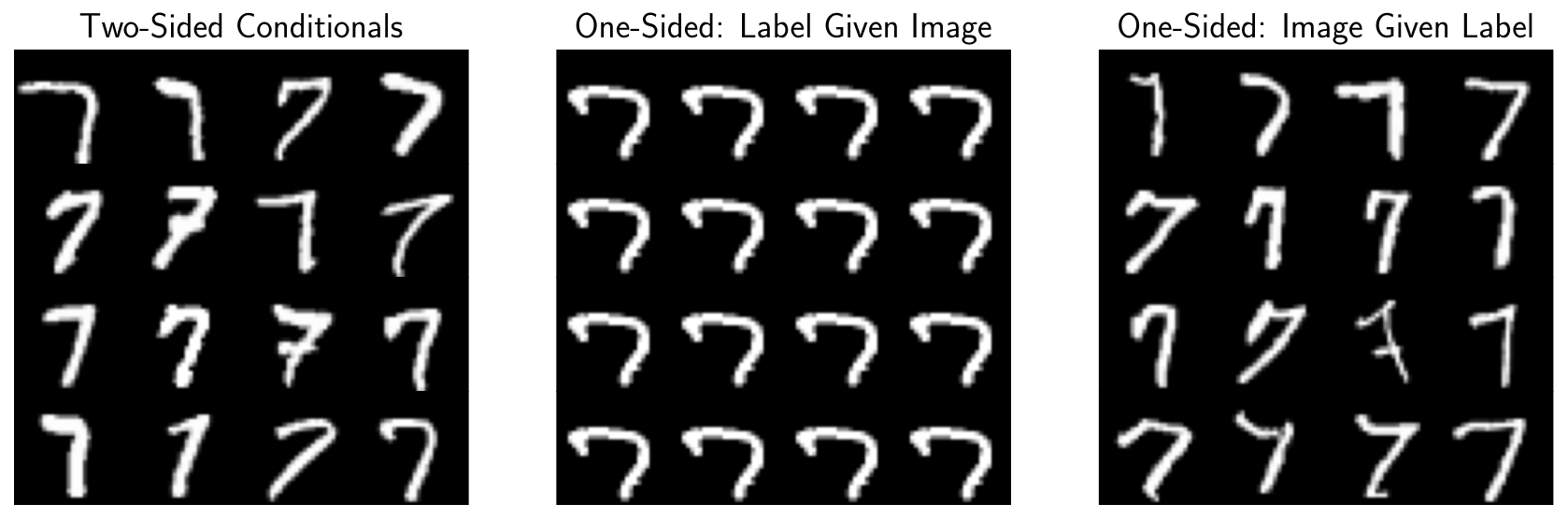}
\vspace{-0.1cm}
\caption{Sixteen images sampled from the MNIST test dataset according to the weights given by the cosine similarity between the learned image encoder for three different loss functions and the one-hot encoding of the label $\y = 7$. We observe that the encoders learned with the two-sided conditional loss (\emph{left}) and the loss for image given labels (\emph{right}) demonstrate a diverse set of images, while the encoder learned using the one-sided loss for label given image (\emph{middle}) does not capture the distribution well, instead concentrating on a single digit.  \label{fig:MNIST_cosines}}
\end{figure}

\subsection{Lagrangian Data Assimilation} \label{ssec:fluid_flow}

In this subsection we study the problem of recovering an Eulerian velocity field  from Lagrangian fluid flow data \cite{kuznetsov2003method}, introduced in Subsection \ref{ssec:IE}. We generalize the simplified setting described there in two ways: (i) we work with a time-dependent divergence free velocity, expressible as the skew-gradient of a potential (streamfunction); (ii) we assume that the Eulerian data is already encoded via the coefficients of an expansion of the potential in a set of time-oscillating, and spatially Fourier, modes.

Given potential  $\psi \colon \mathbb{T}^2 \times [0,T] \rightarrow \R$ we may define a time-dependent incompressible velocity field $\rr \in C^1( \mathbb{T}^2 \times [0,T], \R^2)$ as follows:
\begin{equation} \label{eq:velocity_from_pot}
\rr(x,t) = J\nabla \psi(x,t), \qquad J = \begin{bmatrix} 0 & -1 \\ 1 & 0 \end{bmatrix}.
\end{equation}
We work with potentials in the form
\begin{equation}
\label{eq:pot}
\psi(x,t)=\sum_{k=1}^K \psi_k \exp(i\omega_k t) e_k(x)
\end{equation}
where the $\{e_k\}_{k=1}^K$ denote a finite collection of Fourier modes and the $\{\omega_k\}_{k=1}^K$ denote a set of temporal frequencies, chosen at
random. We identify the complex numbers $\{\psi_k\}_{k=1}^K$ with a vector in $\R^{2K}$ and view this as a prescribed embedding of the Eulerian observations
of the velocity field: $g_\x(\x)=\{\psi_k\}_{k=1}^K;$ we will learn only the embedding of the Lagrangian data.

\begin{figure}
\centering
\includegraphics[width=0.4\textwidth]{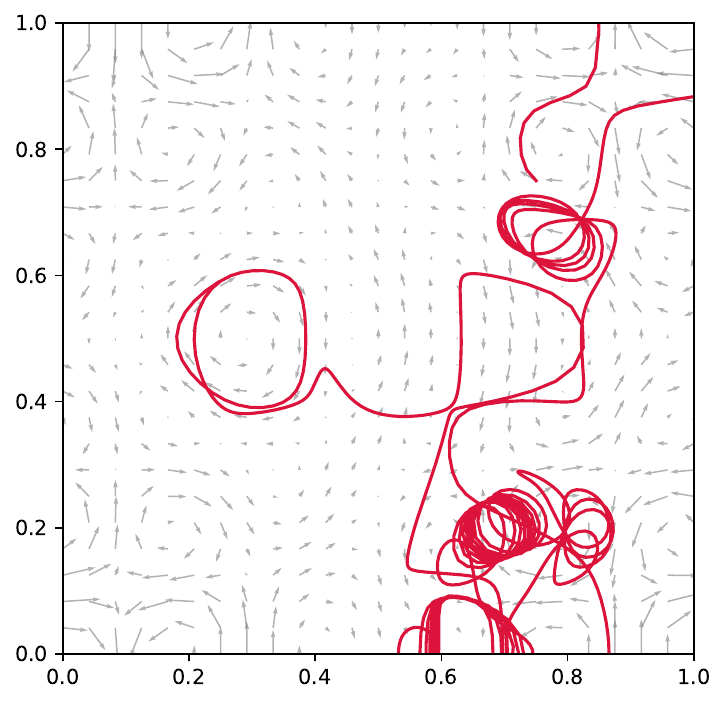}
\includegraphics[width=0.4\textwidth]{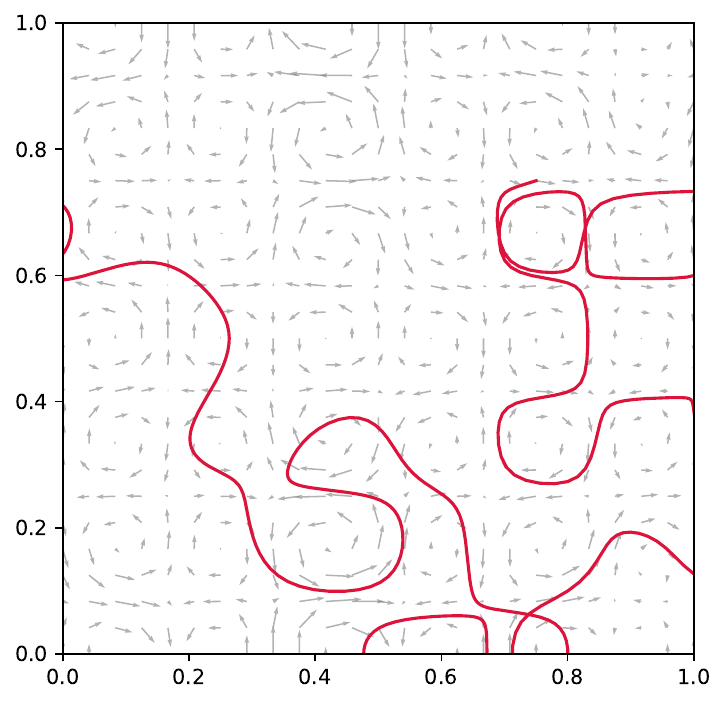}
\vspace{-0.3cm}
\caption{Visualization of two paired samples of the Eulerian potential field and Lagrangian trajectory. \label{fig:Lagrangian_paired_samples}}
\end{figure}

To define the Lagrangian data we integrate the velocity field in-time starting from the initial condition $x(0) = (0.75,0.75)$ to define a trajectory $x(t)$. In our experiment, we use a fourth-order Runge-Kutta method with time-step $\Delta t = 10^{-5}$ and integrate up to $T = 0.1$ The  trajectory position  is recorded  every $10$ time steps to produce a data sequence $\y \in \R^{2 \times J_\y}$ given by the vectors $v(j) = x(\Delta t 10 j) \in \R^2$ for $j = 1,\dots,J_\y$. In our experiment we have $J_\y = 1000$ observation times and $K = 49$ coefficients. Figure~\ref{fig:Lagrangian_paired_samples} visualizes two potential fields (frozen at time $t=0$) and their associated trajectories in the two-dimensional periodic domain $[0,1]^2$.

We use a text transformer to define the encoder for the trajectory $g_\y(\y)$. Similarly to text, the transformer follows the architecture of a typical text encoder that processes the values of the trajectory at each time-step as a separate token. More precisely, the encoder is given by the following composition
$$ g_\y(\y) = P \circ \mathsf{A}_M \circ \cdots \mathsf{A}_1 \circ L(\y),$$
where $L \colon \R^{d} \rightarrow \R^{d'}$ is an initial lifting layer acting on each sample along the trajectory of length $J_\y$, $\mathsf{A}_k \colon \R^{J_\y \times d'} \rightarrow \R^{J_\y \times d'}$ is a residual attention block (consisting of a multi-head attention, layer normalization and a multi-layer perception network), and $P \colon \R^{J_\y \times d'} \rightarrow \R^{\de}$ is a pooling operator that extracts the last element (i.e., token) of the ensemble and projects it into the embedding dimension. In this example, we select $\de = 2K$ to embed the Lagrangian data into the space of coefficients for the potential and we use $M = 12$ layers.

In our experiments, we minimize the training loss $\LCOND$ on the two conditionals using the Adam optimizer with mini-batches of size 128, a learning rate set to $10^{-4}$ and using 20 epochs that each see $N = 32,268$ training samples. %
For this study we evaluate the retrieval capabilities of the learned embeddings on a test set of the same size comprising a total of $\Nt = 32,268$ paired samples of coefficients and trajectories. 
Figure~\ref{fig:fluid_CLIP_loss} plots the retrieval accuracies (normalized to 1) on the training and test sets across the training iterations. In particular, we evaluate the accuracy of retrieving the coefficients from Lagrangian trajectories in the top row and retrieving the trajectories from the coefficients in the bottom row. For a set of samples from a given modality, the retrieval accuracy is computed by evaluating the fraction of inputs for which the true paired sample of the other modality is the top retrieved sample according to the learned cosine similarities ($R@1$ left) or within the retrieved samples with the top five similarities ($R@5$ right). We observe that the model shows an improvement in mapping potential functions to trajectories and trajectories to potential functions across training, approaching 100\% accuracy for $R@5$ retrieval. We also observe that the retrieval accuracy is generally higher for recovering trajectories from potentials, which is expected given this mapping is well-defined by~\eqref{eq:velocity_from_pot} and~\eqref{eq:pot}. %
Lastly, Figure~\ref{fig:fluid_CLIP_retrieval} shows an example of the more challenging retrieval problem of recovering a velocity field potential from a Lagrangian trajectory.

\begin{figure}[!ht]
\centering
\includegraphics[width=0.9\textwidth]{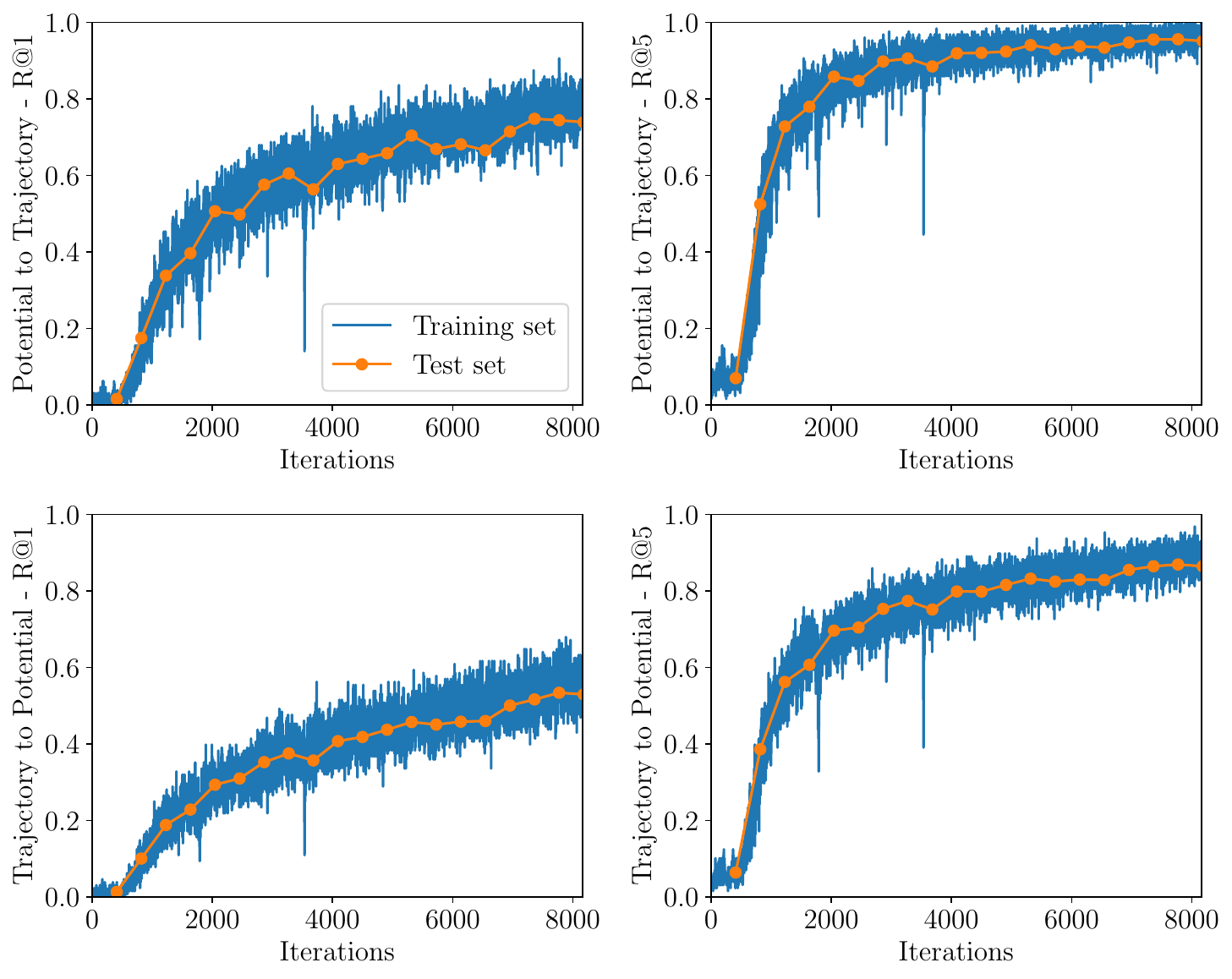}
\vspace{-0.3cm}
\caption{Improvement in retrieval accuracy during training of the encoders for the Lagrangian data assimilation problem. We compute the retrieval accuracies from coefficients to trajectories (\emph{top}) and trajectories to coefficients (\emph{bottom}) for both top retrieval (\emph{left}) and top 5 retrieved items (\emph{right}), which shows improvement on both the training and test data. \label{fig:fluid_CLIP_loss}}
\end{figure}

\begin{figure}[!ht]
\centering
\includegraphics[width=0.8\textwidth]{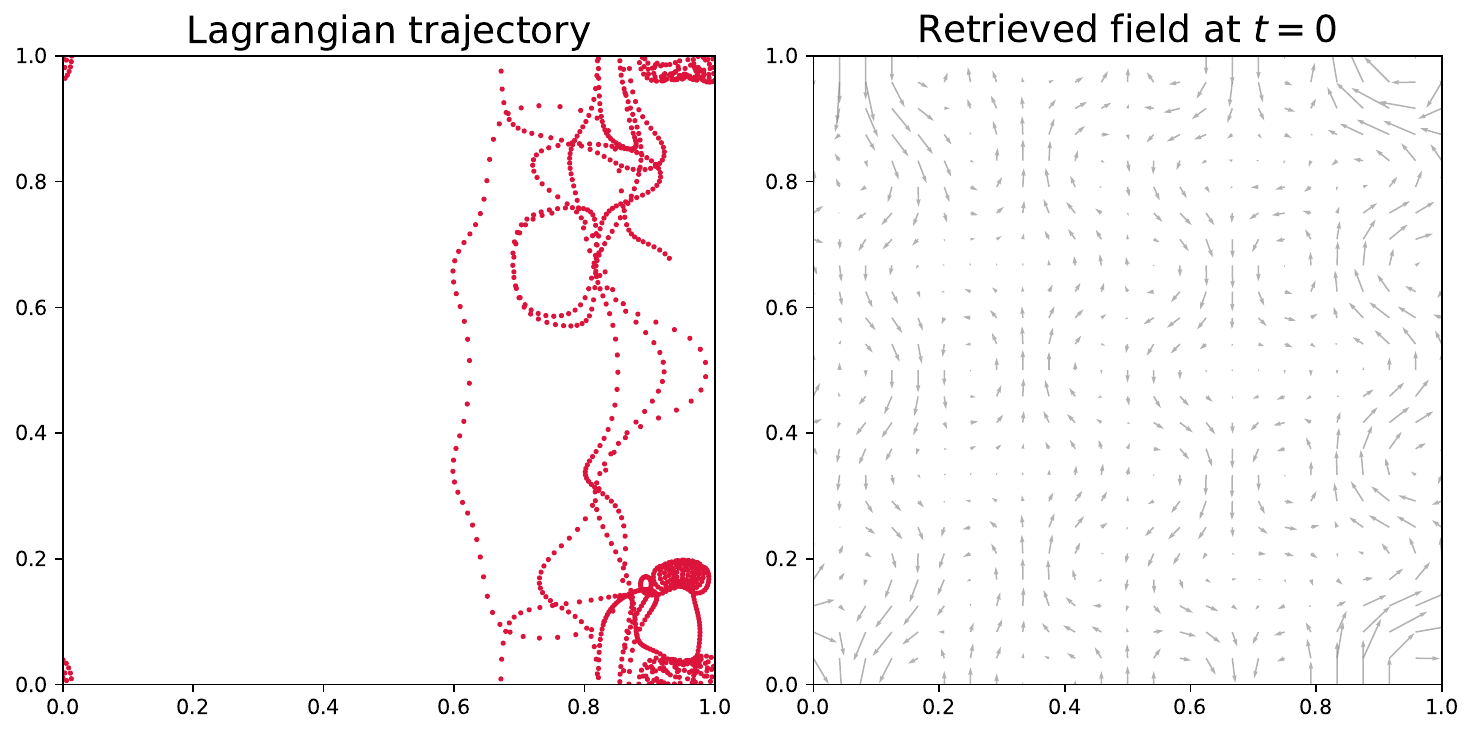}
\vspace{-0.2cm}
\caption{Retrieval of a Eulerian potential from a Lagrangian trajectory. \label{fig:fluid_CLIP_retrieval}}
\end{figure}

\section{Discussion and Outlook}

This work presents a mathematical framework for contrastive learning, introducing several generalizations of the standard approach through novel alignment metrics and probabilistic loss functions. By analyzing Gaussian models, we develop theory to elucidate the relative merits of the different variants on the standard methodology. Our numerical experiments support and extend the theoretical findings, demonstrating their relevance beyond the Gaussian setting. These experiments also demonstrate connections to image classification tasks and suggest broader applicability in scientific and engineering domains. Future directions include extending the framework to handle problems involving more than two modalities, leveraging the flexibility of the proposed methodological variants, utilizing the probabilistic insights emphasized in our formulation, and further developing applications in science and engineering.

\section*{Acknowledgments}
The work of AMS on Lagrangian data assimilation is supported by a Department of Defense (DoD) Vannevar Bush Faculty Fellowship (award N00014-22-1-2790). The authors thank Niall Siegenheim for his assistance in creating the Lagrangian dataset used in this study.

\def\bibfont{\small}
\bibliographystyle{imsart-number}
\bibliography{references}

\appendix

\section{Proofs of Optimization Results}
\label{app:A}

\begin{proof}[Proof of Theorem~\ref{thm:clip_minimization}] Let $\nu$ and $\mu$ be measures, both absolutely continuous with
respect to a common reference measure $\lambda$. The Kullback-Leibler divergence between $\mu$ and $\nu$ is
$$\dkl(\mu||\nu) = \mathbb{E}_{\mu}\left[\log\frac{d\mu/d\lambda}{d\nu/d\lambda}\right].$$
Using this identity with reference measures $\mu_\x$ and $\mu_\y$, respectively, we obtain
\begin{align*}
\mathbb{E}_{\y \sim \mu_\y}\left[\dkl(\mu_{\x|\y}(\cdot|\y)||\pmeas_{\x|\y}(\cdot|\y;\theta)\right] &= \mathbb{E}_{\y \sim \mu_\y}\mathbb{E}_{\mu_{\x|\y}(\cdot|\y)}\left[\log \frac{d\mu_{\x|\y}(\cdot|\y)}{d\mu_\x}\right] - \mathbb{E}_{\y \sim \mu_\y}\mathbb{E}_{\x \sim \mu_{\x|\y}(\cdot|\y)}[\log\pmodel(\x|\y;\theta)],\\ 
\mathbb{E}_{\x \sim \mu_\x}\left[\dkl(\mu_{\y|\x}(\cdot|\x)||\pmeas_{\y|\x}(\cdot|\x;\theta)\right] &= \mathbb{E}_{\x \sim \mu_\x}\mathbb{E}_{\mu_{\y|\x}(\cdot|\x)}\left[\log \frac{d\mu_{\y|\x}(\cdot|\x)}{d\mu_\y}\right] - \mathbb{E}_{\x \sim \mu_\x}\mathbb{E}_{\y \sim \mu_{\y|\x}(\cdot|\x)}[\log\pmodel(\y|\x;\theta)]. 
\end{align*}
Under the finite-entropy assumption, the objective $\JCOND$ can be written as
\begin{align*}
\JCOND(\theta) &= \frac{1}{2}\mathbb{E}_{\y \sim \mu_\y}\left[\dkl(\mu_{\x|\y}(\cdot|\y)||\pmeas_{\x|\y}(\cdot|\y;\theta)\right] + \frac12 \mathbb{E}_{\x \sim \mu_\x}\left[\dkl(\mu_{\y|\x}(\cdot|\x)||\pmeas_{\y|\x}(\cdot|\x;\theta))\right] \\
&= C - \frac12 \mathbb{E}_{\y \sim \mu_\y}\mathbb{E}_{\x \sim \mu_{\x|\y}(\cdot|\y)}[\log\pmodel(\x|\y;\theta)] - \frac12 \mathbb{E}_{\x \sim \mu_\x}\mathbb{E}_{\y \sim \mu_{\y|\x}(\cdot|\x)}[\log\pmodel(\y|\x;\theta)] \\
&= C + \LCOND(\theta),
\end{align*}
where $C$ is a constant that only depends on $\mu$, and hence not on the parameters $\theta$, given by
\begin{align*}
C &\coloneqq \frac{1}{2}\mathbb{E}_{\y \sim \mu_\y}\mathbb{E}_{\mu_{\x|\y}(\cdot|\y)}\left[\log \frac{d\mu_{\x|\y}(\cdot|\y)}{d\mu_\x}\right] + \frac{1}{2}\mathbb{E}_{\x \sim \mu_\x}\mathbb{E}_{\mu_{\y|\x}(\cdot|\x)}\left[\log \frac{d\mu_{\y|\x}(\cdot|\x)}{d\mu_\y}\right] \\
&= \frac{1}{2}\mathbb{E}_{\y \sim \mu_\y}\left[\dkl(\mu_{\x|\y}(\cdot|\y)||\mu_{\x})\right] + \frac{1}{2}\mathbb{E}_{\x \sim \mu_\x}\left[\dkl(\mu_{\y|\x}(\cdot|\x)||\mu_{\y})\right].
\end{align*}
From the assumption, $C < \infty.$ Thus, a minimizer of $\JCOND$ is also a minimizer of $\LCOND$.
\end{proof}

\begin{proof}[Proof of Theorem~\ref{thm:equivalent_objectives}]
For the lower bound, we use that $x \mapsto -\log(x)$ is a convex function on $(0,\infty)$ so that by Jensen's inequality we have
\begin{align*}
\mathbb{E}_{\x \sim \mu_\x}\left[-\log \mathbb{E}_{\y \sim \mu_\y}\exp \bigl( \langle \normg_\x(\x;\theta_\x), \normg_\y(\y; \theta_\y) \rangle / \tau \bigr)\right] &\geq  -\log\mathbb{E}_{(\x,\y) \sim \mu_{\x}\otimes \mu_{\y}} \exp \bigl( \langle \normg_\x(\x;\theta_\x), \normg_\y(\y; \theta_\y) \rangle / \tau \bigr) \\
\mathbb{E}_{\y \sim \mu_\y}\left[-\log \mathbb{E}_{\x \sim \mu_\x}\exp \bigl( \langle \normg_\x(\x;\theta_\x), \normg_\y(\y; \theta_\y) \rangle / \tau \bigr)\right] &\geq  -\log\mathbb{E}_{(\x,\y) \sim \mu_{\x}\otimes \mu_{\y}} \exp \bigl( \langle \normg_\x(\x;\theta_\x), \normg_\y(\y; \theta_\y) \rangle / \tau \bigr).
\end{align*}
Substituting these expressions in the loss function $\LCOND$ gives us
\begin{align*}
-\LCOND(\theta) &\geq \mathbb{E}_{(\x,\y) \sim \mu}\left[\langle \normg_\x(\x;\theta_\x), \normg_\y(\y; \theta_\y) \rangle / \tau \right] - \log \mathbb{E}_{(\x,\y) \sim \mu_{\x}\otimes \mu_{\y}} \left[\langle \normg_\x(\x;\theta_\x), \normg_\y(\y; \theta_\y) \rangle / \tau \right].\\
&= -\LJOINT(\theta),
\end{align*}
for all $\theta \in \R^p$. Multiplying both signs by a negative number gives us the final result. 
\end{proof}

\begin{proof}[Proof of Theorem~\ref{thm:retrieval}] For an empirical distribution, we have that the maximum is achieved at one of the elements in the dataset. Moreover,
$$\argmax_{\x \in \mathcal{\X}} \pmodel^N(\x|\y;\theta) = \argmax_{\x^i: i \in [N]} \omega_{i}(\y).$$
Given that the un-normalized weights are monotonic functions of the cosine similarity and the normalization constant does not affect the minimum, we have the result.
\end{proof}

\section{Use of Maximum Mean Discrepancy} \label{app:mmd}

\begin{example} \label{ex:mmd_conditionals} One metric that is readily computable between empirical measures is the maximum mean discrepancy (MMD). Assume
we are given a space $U$ and kernel function $k \colon U \times U \rightarrow \R_+$ measuring the closeness of points in $U$; for example we might
use the Gaussian kernel or polynomial kernel. The MMD between measures $\pi$ and $\pi'$ supported on $U$ is 
\begin{equation} \label{eq:MMD}
\Div_{\mathsf{mmd}}(\pi,\pi') \coloneqq \mathbb{E}_{(x,x') \sim \nu \otimes \nu} k(x,x') - 2\mathbb{E}_{(x,y) \sim \nu \otimes \nu'} k(x,y) + \mathbb{E}_{(y,y') \sim \nu' \otimes \nu'} k(y,y').
\end{equation}
For instance, if $\pi = \frac{1}{N}\sum_{i=1}^N \delta_{x^i}$ and $\pi' = \frac{1}{N}\sum_{i=1}^N \delta_{y^i},$ then the metric has the form
$$\Div_{\mathsf{mmd}}(\pi,\pi') \coloneqq \frac{1}{N(N-1)}\sum_{i \neq j} k(x^i,x^j) - 2\frac{1}{N(N-1)}\sum_{i \neq j} k(x^i,y^j) + \frac{1}{N(N-1)}\sum_{i \neq j} k(y^i,y^j).$$

Now set $U=\cU$ and choose a kernel $k_\x$; and then set $U=\cV$ and choose a kernel $k_\y.$
Using the two resulting MMDs for the divergence $\Div$ in Optimization Problem~\ref{do:J1} we have the objective function (noting that the MMD distance
will typically be defined with different kernels on the two different spaces $\cU,\cV$) given as:
\begin{align*}
\JCONDM(\theta;\lambda_\x,\lambda_\y) &= \frac{\lambda_\x}{2}\mathbb{E}_{\y \sim \mu_\y}\left[\Div_{\mathsf{mmd}}(\mu_{\x|\y}(\cdot|\y)||\pmeas_{\x|\y}(\cdot|\y;\theta))\right]\\&\quad\quad\quad\quad\quad\quad + \frac{\lambda_\y}{2}\mathbb{E}_{\x \sim \mu_\x}\left[\Div_{\mathsf{mmd}}(\mu_{\y|\x}(\cdot|\x)||\pmeas_{\y|\x}(\cdot|\x;\theta))\right].
\end{align*}
Only a subset of terms in this objective depend on the parameters $\theta$. Thus to
minimize this objective with respect to $\theta$, we may identify a loss function comprising only those terms that depend on $\theta$.
The objective is a sum of two terms for matching the $\x|\y$ and $\y|\x$ conditionals:
$$\JCONDM(\theta;1,1) = \lambda_\x \JCONDM(\theta;1,0) + \lambda_\y \JCONDM(\theta;0,1).$$ 
We first identify the $\theta$-dependent terms in the contribution to this objective defined by the  $\x|\y$ conditional. The term defined by the $\y|\x$ conditional is handled similarly, by symmetry. From the definition in~\eqref{eq:MMD}, the expected MMD between the true and model $\x|\y$ conditionals is given by 
\begin{align*}
    \JCONDM(\theta;1,0) &= \frac{1}{2} \mathbb{E}_{\y \sim \mu_{\y}}\Bigl[\mathbb{E}_{(\x,\x') \sim \mu_{\x|\y}(\cdot|\y) \otimes \mu_{\x|\y}(\cdot|\y)} k_\x(\x,\x')\\&  \qquad\qquad\qquad- 2 \mathbb{E}_{(\x,\x') \sim \mu_{\x|\y}(\cdot|\y) \otimes \nu_{\x|\y}(\cdot|\y;\theta)} k_\x(\x,\x') \\
    &\qquad\qquad\qquad \qquad\qquad\qquad +  \mathbb{E}_{(\x,\x') \sim \nu_{\x|\y}(\cdot|\y;\theta) \otimes \nu_{\x|\y}(\cdot|\y;\theta)} k_\x(\x,\x')\Bigr] \\
    &= \frac{1}{2} \mathbb{E}_{\y \sim \mu_\y} \Bigl[\mathbb{E}_{(\x,\x') \sim \mu_{\x}\otimes \mu_\x} k_\x(\x,\x')r(\x|\y)r(\x'|\y)\\
    &\qquad\qquad\qquad - 2\mathbb{E}_{(\x,\x') \sim \mu_{\x} \otimes \mu_{\x}} k_\x(\x,\x')r(\x|\y)\rho(\x'|\y;\theta) \\
    &\qquad\qquad\qquad  \qquad\qquad\qquad+ \mathbb{E}_{(\x,\x') \sim \mu_{\x}\otimes \mu_\x} k_\x(\x,\x')\rho(\x|\y;\theta)\rho(\x'|\y;\theta)\Bigr] \\
    &= \frac{1}{2} \mathbb{E}_{(\y,\x,\x') \sim \mu_{\y} \otimes \mu_{\x}\otimes \mu_\x} k_\x(\x,\x')r(\x|\y)r(\x'|\y) + \LCONDM(\theta;1,0),
\end{align*}
where the loss is defined as
\begin{align*}\LCONDM(\theta;1,0)& = -\frac{1}{2}\mathbb{E}_{(\y,\x,\x') \sim \mu_{\y}\otimes \mu_{\x}\otimes \mu_\x} k_\x(\x,\x')\rho(\x|\y;\theta)\rho(\x'|\y;\theta)\\ &\qquad\qquad\qquad\qquad\qquad\qquad+\mathbb{E}_{((\y,\x),\x') \sim \mu \otimes \mu_{\x}} k_\x(\x,\x')\rho(\x'|\y;\theta).
\end{align*}
We then minimize
$$\LCONDM(\theta;1,1) = \lambda_\x \LCONDM(\theta;1,0) + \lambda_\y \LCONDM(\theta;0,1).$$ 
This may be evaluated using only the empirical measures $\mu^N, \mu_\x^N, \mu_\y^N$ in~\eqref{eq:empirical} by noting that
$$\LCONDM^N(\theta;1,0) = \frac{1}{2}\sum_{i} \sum_{j \neq k} k_\x(\x^j,\x^k)\rho(\x^j|\y^i;\theta)\rho(\x^k|\y^i;\theta)
+\sum_{i \neq k}  k_\x(\x^i,\x^k)\rho(\x^k|\y^i;\theta),$$
together with the analogous expression for the $\y|\x$ conditional.
\end{example}

\section{Retrieval and Classification at the Population Level} \label{app:RCP}

\subsection{Retrieval: Population Level} \label{sec:retrieval-population} 
For simplicity of exposition we assume that the the marginal of the data distribution over $\x$ is absolutely continuous with respect to the Lebesgue measure: there exists potential function $\phi_\x \colon \mathcal{\X} \rightarrow \R$ such that
$$\mu_\x(d\x) =  \exp(\phi_\x(\x))d\x.$$
Then, the conditional measure for $\x|\y$ in~\eqref{eq:pmeas1} also has Lebesgue density
$$\pmeas_{\x|\y}(d\x|\y;\theta) =\Big(\frac{\exp\bigl(\langle \normg_\x(\x;\theta_\x), \normg_\y(\y;\theta_\y) \rangle / \tau + \phi_\x(\x)\bigr)}{\int_{\cX} \exp(\langle \normg_\x(\x';\theta_\x), \normg_\y(\y;\theta_\y) \rangle / \tau + \phi_\x(\x')) d\x'}\Bigr)\, du.$$
We may view this formulation in the context of a Bayesian inverse problem for $\x$ given $\y.$ The preceding identity defines the posterior $\pmeas_{\x|\y}$ from the prior $\mu_\x$ and likelihood proportional to $\rho(\cdot|\y)$ defined in (\ref{eq:continuum_probability}a).
To identify the input $\x \in \mathcal{\X}$ that is most related to the given point $v \in \mathcal{\Y}$, it is natural to seek the point in $\mathcal{\X}$ that maximizes the posterior probability. That is,
\begin{equation}
\x^\ast(\y) \in \argmax_{\x \in \mathcal{\X}} \pmeas_{\x|\y}(\x|\y;\theta) = \argmax_{\x \in \mathcal{\X}} \, \Bigl(\langle \normg_\x(\x;\theta_\x), \normg_\y(\y;\theta_\y) \rangle/\tau + \phi_\x(\x)\Bigr).
\end{equation}
In the context of Bayesian inverse problems, $\x^\ast$ is commonly referred to as the \emph{mode} or \emph{maximum a-posteriori} (MAP) point. 

\subsection{Classification: Population Level}  \label{sec:classification-population}

The contrastive learning problem defines the learned joint distribution $\nu$, given by \eqref{eq:joint_probability_model} and 
\eqref{eq:rz}, as change of measure $\rho$ from reference measure $\mu_\x \otimes \mu_\y.$ In the Bayesian interpretation we may view 
the resulting conditional distribution $\nu_{\y|\x}$ from  \eqref{eq:pmeas2} as being defined by the
likelihood proportional to $\rho(\y|\x;\theta)$ and the prior $\mu_\y$. 
In classification we use the same likelihood, possibly with a modified parameter $\theta$ found using fine-tuning, 
but develop algorithms that generalize to different priors.

To this end, we introduce a measure $\munew$ on $\mathcal{\X} \times \mathcal{\Y}$ and modify \eqref{eq:joint_probability_model}, \eqref{eq:rz} to obtain $\nunew(\cdot;\theta)$, a probability measure for the joint random variable $(\x, \y) \in \mathcal{\X} \times \mathcal{\Y}$, defined by
\begin{subequations}
\label{eq:rzM}
    \begin{align}
    \nunew(d\x,d\y;\theta) &=  \pmodel(\x,\y;\theta)  \munew_\x(d\x)\munew_\y(d\y),\\
    \pmodel(\x,\y;\theta) &=  \frac{1}{Z} \exp\Bigl(\langle \normg_\x(\x;\theta_\x), \normg_\y(\y;\theta_\y) \rangle / \tau \Bigr),\\
        Z &= \int_{\mathcal{\X} \times \mathcal{\Y}} \exp(\langle \normg_\x(\x;\theta_\x),  \normg_\y(\y;\theta_\y) \rangle / \tau) \munew_\x(d\x)\munew_\y(d\y),
    \end{align}
\end{subequations}
where $\munew_\x(d\x), \munew_\y(d\y)$ are the marginals of $\munew.$ 

For expository purposes we assume that, for $\phinew_\y \colon \mathcal{\Y} \rightarrow \R$ and $d\y$ denoting Lebesgue measure, 
$$\munew_\y(d\y) =  \exp\bigl(\phinew_\y(\y)\bigr)d\y.$$ 
The conditional measure for $\y|\x$  then also has a Lebesgue density and is given by
\begin{equation}
    \label{eq:asin}
\nunew_{\y|\x}(d\y|\x;\theta) = \left( \frac{\exp\bigl(\langle \normg_\x(\x;\theta_\x), \normg_\y(\y;\theta_\y) \rangle / \tau + \phinew_\y(\y)\bigr)}{\int_{\cY} \exp\bigl(\langle \normg_\x(\x;\theta_\x), \normg_\y(\y';\theta_\y) \rangle / \tau + \phinew_\y(\y')\bigr) d\y'}\right) d\y.
\end{equation}
For now we fix the parameters $(\theta_\x, \theta_\y)$ at the solution found during contrastive learning based on measure $\mu.$
To assign one label $\y \in \mathcal{\Y}$ to an input $\x \in \mathcal{\X}$, it is natural to identify the mode that maximizes the conditional measure over labels. That is,
\begin{equation} \label{eq:mode_labels}
\y^\ast(\x) \in \argmax_{\y \in \mathcal{\Y}} \nunew_{\y|\x}(\y|\x;\theta) = \argmax_{\y \in \mathcal{\Y}} \Bigl( \langle \normg_\x(\x;\theta_\x), \normg_\y(\y;\theta_\y) \rangle/\tau + \phinew_\y(\y) \Bigr).
\end{equation}
In the Bayesian interpretation $\nunew_{\y|\x}$ differs from $\pmeas_{\y|\x}$ only in the change of prior; the likelihood has been fixed.
The point $\y^\ast(\x)$ is the MAP estimator for the posterior $\nunew_{\y|\x}$.

\begin{example} For a set $\mathcal{C} \subseteq \mathcal{\Y}$, one choice for the marginal distribution $\munew_\y$ is 
$$\munew_\y(d\y) \propto \mu_\y(d\y)1_{\{\y \in \mathcal{C}\}}.$$ Assuming that $\mu_\y$ is absolutely continuous with respect to the Lebesgue measure so that there exists a potential $\phi_\y$ such that $\mu_\y(d\y) = \exp(\phi_\y(\y))d\y$, the optimization problem in~\eqref{eq:mode_labels} to identify the mode corresponds to solving the constrained problem
$$\y^\ast(\x) \in  \argmax_{\y \in \mathcal{C}} \Bigl( \langle \normg_\x(\x;\theta_\x), \normg_\y(\y;\theta_\y) \rangle/\tau + \phi_\y(\y) \Bigr).$$ 
\end{example}

During pretraining, the parameters $(\theta_\x, \theta_\y)$ 
defining the encoders for both modalities are learned using data $\mu.$ To improve the classification accuracy of the model it may be advantageous, during a fine-tuning phase, to adjust the conditional distribution whose mode defines
the classifier so that it  better matches the dataset $\mufinetuning$. We will use fine-tuning to find
an improved prior and to modify the parameter $\theta_\y$ which defines the likelihood. 
To this end we first introduce a family of marginal measures $\munef_\y(\cdot;\theta_\phi)$ depending on parameters $\theta_\phi$ that are defined via their density with respect to $\munew_\y$:
$$\munef_\y(d\y;\theta_\phi) =  \exp\bigl(\phinef_\y(\y;\theta_\phi)\bigr)\munew_\y(d\y)=\exp\bigl(\phinef_\y(\y;\theta_\phi)+\phinew_\y(\y)\bigr)d\y.$$
Now define $\vartheta=(\theta_\y,\theta_\phi).$ Using the same form for the change of measure as in \eqref{eq:asin}, but using $\munef_\y$ as
reference measure rather than $\munew_\y$, we obtain 
\begin{equation}
    \label{eq:asinA}
\nunef_{\y|\x}(d\y|\x;\vartheta) = \left( \frac{\exp\bigl(\langle \normg_\x(\x;\theta_\x), \normg_\y(\y;\theta_\y) \rangle / \tau + \phinef_\y(\y;\theta_\phi)+\phinew_\y(\y)\bigr)}{\mathbb{E}_{\y' \sim \munew_\y} \exp\bigl(\langle \normg_\x(\x;\theta_\x), \normg_\y(\y';\theta_\y) \rangle / \tau + \phinef_\y(\y';\theta_\phi)+\phinew_\y(\y')\bigr)}\right) d\y.
\end{equation}
We consider the learning of $\vartheta$ by minimizing the following objective function for the one-sided conditional distribution:
\begin{equation}
    \JFINE(\vartheta) = \mathbb{E}_{\x \sim \mufinetuning_\x} \dkl(\mufinetuning_{\y|\x}(\cdot|\x)||\nunef_{\y|\x}).
\end{equation}
In solving this optimization problem, and adopting the Bayesian perspective, we choose the prior (via optimal choice of parameter
$\theta_\phi$) and the likelihood  (via optimal choice of parameter $\theta_\y)$ so that the resulting conditional distribution
$\nunef_{\y|\x}(\cdot|\x;\vartheta)$ is well-suited to classification tasks based on data from, or closely related to, $\munew.$
By Theorem~\ref{thm:clip_minimization} and Remark \ref{rem:J2L}, the relevant optimization problem for fine-tuning parameters $\vartheta$ is
\begin{definitiono} \label{op:finetuning}
\begin{align*}
    -\LFINE(\vartheta)
    &= \mathbb{E}_{(\x,\y) \sim \mufinetuning}[\langle \normg_\x(\x;\theta_\x), \normg_\y(\y;\theta_\y) \rangle / \tau + \phinef_\y(\y;\theta_\phi)] \\
    &\qquad\qquad - \mathbb{E}_{\x \sim \mufinetuning_\x} \log \mathbb{E}_{\y' \sim \mufinetuning_\y} \exp [\langle \normg_\x(\x;\theta_\x), \normg_\y (\y';\theta_\y) \rangle / \tau + \phinef_\y(\y';\theta_\phi)]. \\
    \tfine & \in \argmin_{\vartheta \in \R^{p'}} \, \LFINE(\vartheta).
\end{align*}
\end{definitiono}
The resulting model with the optimal parameters $\tfine$ is used to classify typical inputs $\x$ from the marginal distribution $\mufinetuning_\x$ by finding the mode, analogously to~\eqref{eq:mode_labels}:
\begin{equation} \label{eq:mode_labelsA}
\y^\ast(\x) \in \argmax_{\y \in \mathcal{\Y}} \nunef_{\y|\x}(\y|\x;\tfine). 
\end{equation}

\section{Gaussian Proofs and Remarks} \label{app:GaussianProofs}

\begin{proof}[Proof of Theorem~\ref{thm:Gaussian_solution_cosine}] 
The loss function in~\eqref{eq:objective_linear_encoder} can be written as 
\begin{align*}
\LCOND(A) = -\mathbb{E}_{(\x,\y) \sim \mu} \langle \x, A\y \rangle + \frac12 \mathbb{E}_{\x \sim \mu_\x} \log \mathbb{E}_{\y \sim \mu_\y} \exp \langle u, Av \rangle + \frac12 \mathbb{E}_{\y \sim \mu_\y} \log \mathbb{E}_{\x \sim \mu_\x} \exp \langle \x, A\y \rangle
\end{align*}
Using Lemma~\ref{lem:Gaussian_expintegral} with $B = 0$ and $c = A\y$ or $c = A^\top \x$ we have that
\begin{align*}
    \mathbb{E}_{\x \sim \mu_\x} \exp \langle \x, A\y \rangle &= \exp\left(\frac12 \y^\top A^\top \cov_{\x\x} A \y\right), \text{ or } \\
    \mathbb{E}_{\y \sim \mu_\y} \exp \langle \x, A\y \rangle &= \exp\left(\frac12\x^\top A \cov_{\y\y} A^\top \x\right),
\end{align*}
respectively. Substituting these forms in $\LCOND$ results in the objective 
\begin{align*}
\LCOND(A) &= -\mathbb{E}_{(\x,\y) \sim \mu} \Tr(A\y\x^\top) + \frac12 \mathbb{E}_{\x \sim \mu_\x}\left(\frac12\x^\top A \cov_{\y\y} A^\top \x\right) + \frac12 \mathbb{E}_{\y \sim \mu_\y}\left(\frac12 \y^\top A^\top \cov_{\x\x} A \y\right) \\
&= -\mathbb{E}_{(\x,\y) \sim \mu} \Tr(A\y\x^\top) + \frac12 \mathbb{E}_{\x \sim \mu_\x} \Tr(A^\top \cov_{\y\y}A \x\x^\top) + \frac12 \mathbb{E}_{\y \sim \mu_\y} \Tr(A^\top \cov_{\x\x}A \y\y^\top)  \\
&= -\Tr(A\cov_{\y\x}) + \Tr(A^\top \cov_{\x\x}A \cov_{\y\y}).
\end{align*}
By adding a constant that is not dependent on $A$ to the loss function we have
\begin{align}
\LCOND(A) + \Tr(\cov_{\x\x}^{-1}\cov_{\x\y}\cov_{\y\y}^{-1}\cov_{\y\x}) &= \Tr((A^\top\cov_{\x\x} - \cov_{\y\y}^{-1}\cov_{\y\x})(A\cov_{\y\y} - \cov_{\x\x}^{-1}\cov_{\x\y})) \notag \\
&= \Tr(\cov_{\y\y}(A - \cov_{\x\x}^{-1}\cov_{\x\y}\cov_{\y\y}^{-1})^\top\cov_{\x\x}(A - \cov_{\x\x}^{-1}\cov_{\x\y}\cov_{\y\y}^{-1})) \notag \\
&= \|\cov_{\y\y}^{1/2}(A - A^*)\cov_{\x\x}^{1/2}\|_F^2 \label{eq:Frobenius_norm},
\end{align}
where we define $A^* \coloneqq \cov_{\y\y}^{-1}\cov_{\y\x}\cov_{\x\x}^{-1}$. Given that minimizing $\Loss$ corresponds to minimizing~\eqref{eq:Frobenius_norm}, which is bounded from below, we have that the minimum over all matrices $A$ is given by $A^*$. 

With the rank constraint, from~\citep[Theorem 2.1]{friedland2007generalized} we have that
$$A_r^* = \argmin_{A \in \mathcal{A}_r} \|\cov_{\y\y}^{1/2}(A - A^*)\cov_{\x\x}^{1/2}\|_F^2,$$
has a unique solution given by computing the SVD of $\cov_{\y\y}^{1/2}A^*\cov_{\x\x}^{1/2} = \cov_{\y\y}^{-1/2}\cov_{\y\x}\cov_{\x\x}^{-1/2}$ and then pre-multiplying by the square roots of the marginal covariances to get $A_r^*$.
\end{proof}

\begin{proof}[Proof of Corollary~\ref{cor:matching_means}]
If the parameter in the learnable model $\nu$~\eqref{eq:joint_probability_modelI} is chosen to be $A = A^*(r)$ for any $0 < r \leq \min(\dx,\dy)$, the form of the approximate conditional distributions is given by:
\begin{subequations}
    \begin{align}
    \pmeas_{\x|\y}(\x|\y;A^\ast(r)) &= \mathcal{N}(\cov_{\x\x}^{1/2}(\cov_{\x\x}^{-1/2}\cov_{\x\y}\cov_{\y\y}^{-1/2})_r\cov_{\y\y}^{-1/2}\y, \cov_{\x\x}) \\
    \pmeas_{\y|\x}(\y|\x;A^\ast(r)) &= \mathcal{N}(\cov_{\y\y}^{1/2}(\cov_{\x\x}^{-1/2}\cov_{\x\y}\cov_{\y\y}^{-1/2})_r^\top\cov_{\x\x}^{-1/2}\x,  \cov_{\y\y}).
    \end{align}
\end{subequations} 
Without the rank constraint, i.e., $r = \min(\dx,\dy)$, we have $A^*(r) = A^*$ and hence the conditional measures $\pmeas_{\x|\y}$ and $\pmeas_{\y|\x}$ are given by~\eqref{eq:cosine_optconditional_yx_unconstrained} and~\eqref{eq:cosine_optconditional_xy_unconstrained}, respectively.
\end{proof}

\begin{proof}[Proof of Theorem~\ref{thm:Gaussian_quadraticform_onesided}] 
Using Lemma~\ref{lem:KLdivergence_Gaussians} for the KL Divergence between the multivariate Gaussians $\mu_{\x|\y}$ and $\pmeas_{\x|\y}$ in expectation over $\mu_\y$, the objective for $A,B$ is given by 
\begin{align}
    \JCOND(A,B;2,0) &= \mathbb{E}_{\y \sim \mu_\y} \left[\dkl(\mu_{\x|\y}(\cdot|\y)||\pmeas_{\x|\y}(\cdot|\y))\right] \nonumber \\
    &=\frac{1}{2}\left[\Tr((B + \cov_{\x\x}^{-1})\cov_{\x|\y}) - d + \log|(B + \cov_{\x\x}^{-1})\cov_{\x|\y}|\right] + \Delta(A,B), \label{eq:KLonesided_loss_proof}
\end{align}
where the last term is expressed as
\begin{align*}
\Delta(A,B) &\coloneqq \mathbb{E}_{\y \sim \mu_\y} \|(B + \cov_{\x\x}^{-1})^{-1}Av - \cov_{\x\y}\cov_{\y\y}^{-1}\y\|_{(B + \cov_{\x\x}^{-1})^{-1}}^2 \\
&= \Tr(((B + \cov_{\x\x}^{-1})^{-1}A - \cov_{\x\y}\cov_{\y\y}^{-1})^\top (B + \cov_{\x\x}^{-1})((B + \cov_{\x\x}^{-1})^{-1}A - \cov_{\x\y}\cov_{\y\y}^{-1})\cov_{\y\y}) \\
&= \left\|\left((B + \cov_{\x\x}^{-1})^{-1/2}A - (B + \cov_{\x\x}^{-1})^{1/2}\cov_{\x\y}\cov_{\y\y}^{-1}\right)\cov_{\y\y}^{1/2}\right\|_F^2. 
\end{align*}
Only the last term in~\eqref{eq:KLonesided_loss_proof} depends on $A$. %
Minimizing the last term over unconstrained matrices $A$ for each $B$ results in
$$A^\ast = (B + \cov_{\x\x}^{-1})\cov_{\x\y}\cov_{\y\y}^{-1},$$
where the minimum satisfies $\Delta(A^\ast, B) = 0$. Then, the minimizer of $\LCOND(A^\ast,B;2,0)$ over unconstrained matrices $B$ is given by
$$B^\ast = \cov_{\x|\y}^{-1} - \cov_{\x\x}^{-1} = \cov_{\x\x}^{-1}\cov_{\x\y}\cov_{\y|\x}^{-1}\cov_{\y\x}\cov_{\x\x}^{-1}.$$
Minimizing the last term in~\eqref{eq:KLonesided_loss_proof} over the constrained set $\mathcal{A}_r$ for each $B$ gives us
$$A^\ast(r) = (B + \cov_{\x\x}^{-1})^{1/2}((B + \cov_{\x\x}^{-1})^{1/2}\cov_{\x\y}\cov_{\y\y}^{-1/2})_r\cov_{\y\y}^{-1/2},$$
where the minimum of the third term is then given by
\begin{align*}
\Delta(A^\ast(r),B) &= \left\|\left((B + \cov_{\x\x}^{-1})^{1/2}\cov_{\x\y}\cov_{\y\y}^{-1/2}\right)_r - \left((B + \cov_{\x\x}^{-1})^{1/2}\cov_{\x\y}\cov_{\y\y}^{-1/2}\right)\right\|_{F}^2.
\end{align*}
Substituting the minimum value of $\Delta(A^\ast(r),B)$ into the loss yields the  objective in~\eqref{eq:opt_onesided_Bproblem} for finding the optimal $B^*(r)$.
\end{proof}

\begin{proof}[Proof of Corollary~\ref{cor:matching_meansandcov}]
If the parameters in the learnable model $\pmeas$ in~\eqref{eq:Gaussian_joint_quadratic} are chosen to be $A = A^*(r)$ and $B = B^*(r)$ from Theorem~\ref{thm:Gaussian_quadraticform_onesided}, the approximate conditional distribution $\pmeas_{\x|\y}$ in~\eqref{eq:Gaussian_conditionals_quadratic} is given by 
\begin{align*}
    \pmeas_{\x|\y}(\x|\y;A^\ast(r),B^*(r)) &= \mathcal{N}\Bigl((B^*(r) + \cov_{\x\x})^{-1/2}\bigl((B^*(r) + \cov_{\x\x})^{1/2}\cov_{\x\y}\cov_{\y\y}^{-1/2}\bigr)_r\cov_{\y\y}^{-1/2}\y, (B^*(r) + \cov_{\x\x}^{-1})^{-1}\Bigr),
\end{align*}
where $B^*(r)$ solves the optimization problem in~\eqref{eq:opt_onesided_Bproblem}. %

Without a rank constraint, i.e., $r = \dx$, %
we have $B^*(r) = B^*$ in~\eqref{eq:optB_onesided}. 
Then, using the Sherman-Woodbury matrix identity, the conditional covariance of $\x|\y$ for the learnable model is 
$$(B^* + \cov_{\x\x}^{-1})^{-1} = (\cov_{\x\x}^{-1}\cov_{\x\y}\cov_{\y|\x}^{-1}\cov_{\y\x}\cov_{\x\x}^{-1} + \cov_{\x\x}^{-1})^{-1} = \cov_{\x\x} - \cov_{\x\y}\cov_{\y\y}^{-1}\cov_{\y\x}.$$
Thus, the conditional distribution for $\x|\y$ is given by~\eqref{eq:cosine_optconditional_xy_unconstrained},  
whose conditional expectations match the conditional expectations for $\x|\y$ in~\eqref{eq:true_conditional_xy}. 

The one-sided objective $\LCOND(A,B;2,0)$ does not depend on the matrix $C=H^\top H$ appearing in~\eqref{eq:Gaussian_conditionals_quadratic}, which defines the approximate conditional distribution for $\y|\x$. For this reason, minimizing the one-sided objective only matches the conditional distribution for $\x|\y$.
\end{proof}

\begin{proof}[Proof of Theorem~\ref{thm:Gaussian_joint}] 
Using Lemma~\ref{lem:Gaussian_expintegral} with $\z = (\x,\y) \sim \mu_{\x} \otimes \mu_\y$ we have
$$\mathbb{E}_{(\x,\y) \sim \mu_{\x} \otimes \mu_\y}[\exp(\x^\top A \y)] = \mathbb{E}_{\z \sim \mathcal{N}(0,\Lambda)}\left[\exp\left(\frac{1}{2}(\x,\y)^\top B (\x,\y)\right)\right] = \frac{1}{\sqrt{|I_{\dx + \dy} - \Lambda B|}},$$
where
$$\Lambda = \begin{bmatrix} \cov_{\x\x} & 0 \\ 0 & \cov_{\y\y} \end{bmatrix}, \qquad B = \begin{bmatrix} 0 & A \\ A^\top & 0 \end{bmatrix}.$$
Thus, the objective function in~\eqref{eq:JointLossGaussian} is given by
\begin{align*}
\LJOINT(A) &= -\mathbb{E}_{(\x,\y) \sim \mu}[\Tr(A\y\x^\top)] + \log \mathbb{E}_{(u,v) \sim \mu_{\x} \otimes \mu_\y}[\exp(\x^\top A \y)] \\
&= -\Tr(A\cov_{\y\x}) - \frac{1}{2}\log|I_{\dy} - \cov_{\y\y}A^\top \cov_{\x\x}A| \\
&= -\Tr((\cov_{\x\x}^{1/2}A\cov_{\y\y}^{1/2})(\cov_{\y\y}^{-1/2}\cov_{\y\x}\cov_{\x\x}^{1/2})) - \frac{1}{2}\log|I_{\dy} - \cov_{\y\y}^{1/2}(\cov_{\y\y}^{1/2}A^\top\cov_{\x\x}^{1/2})(\cov_{\x\x}^{1/2}A\cov_{\y\y}^{1/2})\cov_{\y\y}^{-1/2}|.
\end{align*}
Let $\widehat{A} \coloneqq \cov_{\x\x}^{1/2}A\cov_{\y\y}^{1/2} \in \R^{\dx \times \dy}$ and $\Omega = \cov_{\y\y}^{-1/2}\cov_{\y\x}\cov_{\x\x}^{-1/2} \in \R^{\dy \times \dx}$. Then, an equivalent objective function that is maximized to identify $\widehat{A}$ is given by
$$\widehat{\mathsf{L}}_{\textsf{joint}}(\widehat{A}) \coloneqq \Tr(\widehat{A}\Omega) +  \frac{1}{2}\log|I_{\dy} - \widehat{A}^\top\widehat{A}|.$$
Let the matrix $\Omega$ have a singular value decomposition $U\Sigma V^\top$. Each matrix $\widehat{A}$ has a singular value decomposition $\widehat{V}D \widehat{U}^\top$. By the von Neumann Trace inequality (see~\cite{carlsson2021neumann}), we have 
$$\max_{\widehat{U},\widehat{V}}\, \Tr(\widehat{V}D\widehat{U}^\top U \Sigma V^T) \leq \sum_{i=1}^{\min(\dx,\dy)} D_{ii}\Sigma_{ii}$$ 
where equality is attained at $\widehat{U} = U$ and $\widehat{V} = V$. By the invariance of the log-determinant to the singular vectors of $\widehat{A}^\top\widehat{A}$, the objective satisfies
$$\widehat{\mathsf{L}}_{\textsf{joint}}(\widehat{A}) \leq \sum_{i=1}^{\min(\dx,\dy)} D_{ii}\Sigma_{ii} + \frac{1}{2}\log(1 - D_{ii}^2),$$
which is a sum of concave functions for each singular value $D_{ii}$. Thus, the objective is maximized at the solutions of the equation $\Sigma_{ii} - D_{ii}/(1 - D_{ii}^2) = 0$ for $i = 1,\dots,\min(\dx,\dy)$. Rearranging into a quadratic equation, the objective is maximized at the positive root of the equation
$$D_{ii}^2\Sigma_{ii} + D_{ii} - \Sigma_{ii} = 0,$$
whose solution is given by $D_{ii} = h(\Sigma_{ii})$ for the function $h$ in Definition~\ref{def:sv_shrinkage}.

In the rank-constrained setting, the same result follows with $D_{ii} = 0$ for $i > r$.
\end{proof}

\begin{proof}[Proof of Corollary~\ref{corr:marginals_joint_loss}]
For any parameter $A$, the marginal distribution of the parameterized model $\pmeas$ in~\eqref{eq:joint_probability_modelI} is given by
\begin{align*} 
\pmeas_{\x}(d\x;A) &= \mathcal{N}(0, \cov_{\x\x} + \cov_{\x\x} A(\cov_{\y\y}^{-1} - A^\top\cov_{\x\y}A)^{-1} A^\top \cov_{\x\x}),
\end{align*}
where the covariance follows from the Schur complement of the inverse covariance of $\pmeas$ for choices of $A$ so that $\cov_{\y\y}^{-1} - A^\top\cov_{\x\y}A$ is invertible. In particular, for any parameter of the form 
\begin{equation} \label{eq:Aparameter_factorization}
A = \cov_{\x\x}^{-1/2} U D V^\top \cov_{\y\y}^{-1/2}
\end{equation}
where $U$ and $V$ are unitary matrices and $D$ is a diagonal matrix, the marginal distribution of $\pmeas_\x$ is
\begin{align} \label{eq:marginal_dist_approx_factorization}
\pmeas_{\x}(d\x;A) = \mathcal{N}(0, \cov_{\x\x}^{1/2}U(I_{\dx} - D^2)^{-1}U^\top\cov_{\x\x}^{1/2}),
\end{align}
where the entries of $D$ must satisfy $|D_{ii}| < 1$ for the covariance to be well defined.

The optimal parameters $A^*_{\mathsf{cond}}$ in~\eqref{eq:Gaussian_cond_cosine_minimizer_unconstrained} and $A^*_{\mathsf{joint}}$ in~\eqref{eq:JointMin} that minimize the conditional and joint loss functions, respectively, both have the form in~\eqref{eq:Aparameter_factorization}. In particular, letting $U\Sigma V$ denote the SVD of $\cov_{\x\x}^{-1/2}\cov_{\x\y}\cov_{\y\y}^{-1/2}$, the optimal parameters are given by
\begin{align*}
A^*_{\mathsf{cond}} &= \cov_{\x\x}^{-1/2} U \Sigma V^\top \cov_{\y\y}^{-1/2} \\
A^*_{\mathsf{joint}} &= \cov_{\x\x}^{-1/2} U h(\Sigma) V^\top \cov_{\y\y}^{-1/2}.
\end{align*}
Substituting $D = \Sigma$ and $D = h(\Sigma)$ in~\eqref{eq:marginal_dist_approx_factorization} yields the  marginal distributions in~\eqref{eq:marginal_dist_approx_cond} and~\eqref{eq:marginal_dist_approx_joint} for the parameters minimizing the conditional and joint losses, respectively. Moreover, we note that $D = 0$ or equivalently $A = 0$ yields the marginal distribution of the data distribution $\mu_\x(d\x) = \mathcal{N}(0, \cov_{\x\x})$.

Lastly, using the property $0 \leq h(\sigma) < \sigma$ for the function $h$ in Definition~\ref{def:sv_shrinkage}, we have that $I \prec (I - h(\Sigma)^2)^{-1} \prec (I - \Sigma^2)^{-1}$. Thus, it follows that
$$\cov_{\x\x} = \cov_{\x\x}^{1/2}U U^\top\cov_{\x\x}^{1/2} \prec \cov_{\x\x}^{1/2}U(I_{\dx} - \h(\Sigma)^2)^{-1}U^\top\cov_{\x\x}^{1/2} \prec \cov_{\x\x}^{1/2}U(I_{\dx} - \Sigma^2)^{-1}U^\top\cov_{\x\x}^{1/2}.$$
That is, the marginal covariance corresponding to $A^*_{\mathsf{joint}}$ is strictly closer to the true marginal covariance of $\mu_{\x}$ than the marginal covariance corresponding to $A^*_{\mathsf{cond}}$ in the cone of positive definite matrices.
\end{proof}

\begin{remark} \label{rem:G1} The loss function~\eqref{eq:objective_linear_encoder} can also be derived by noting that KL divergence between two multivariate Gaussians with the same covariance, which we denote by $\JCOND(A)$, is given by squared and weighted $L^2$ norms of the errors in the mean for both the $\x|\y$ and $\y|\x$ conditionals; see Lemma~\ref{lem:KLdivergence_Gaussians}.  Then, we have
\begin{align*}
\JCOND(A) &= \frac{1}{4}\mathbb{E}_{\y \sim \mu_\y} \left|\cov_{\x\x} A\y - \cov_{\x\y}\cov_{\y\y}^{-1}\y\right|_{\cov_{\x\x}}^2 + \frac{1}{4}\mathbb{E}_{\x \sim \mu_\x}\left|\cov_{\y\y} A^\top \x - \cov_{\y\x}\cov_{\x\x}^{-1}\x\right|_{\cov_{\y\y}}^2 \\
&= \frac{1}{4}\Tr\bigl((\cov_{\x\x}A - \cov_{\x\y}\cov_{\y\y}^{-1})^\top \cov_{\x\x}^{-1}(\cov_{\x\x}A - \cov_{\x\y}\cov_{\y\y}^{-1})\cov_{\y\y}\bigr) \\
&\quad + \frac{1}{4}\Tr\bigl((\cov_{\y\y}A^\top - \cov_{\y\x}\cov_{\x\x}^{-1})^\top \cov_{\y\y}^{-1}(\cov_{\y\y}A^\top - \cov_{\y\x}\cov_{\x\x}^{-1})\cov_{\x\x}\bigr) \\
&= \frac{1}{2}\Tr(\cov_{\y\y}^{-1}\cov_{\y\x}\cov_{\x\x}^{-1}\cov_{\x\y}) + \frac{1}{2}\LCOND(A),
\end{align*}
where $\LCOND(A)$ 
corresponds to the loss function in~\eqref{eq:objective_linear_encoder}, which has the equivalent form 
\begin{equation*} %
\LCOND(A) \coloneqq  \Tr(A^\top\cov_{\x\x}A\cov_{\y\y}) -\Tr(A\cov_{\x\y}) - \Tr(A^\top\cov_{\y\x}).
\end{equation*}

Similarly, the generalized loss function in Optimization Problem~\ref{do:J1} for the KL divergence between the $\x|\y$ and $\y|\x$ conditionals given weighting parameters $\lambda_\x,\lambda_\y \in \R_+$ has the form
    $$\LCOND(A;\lambda_\x,\lambda_\y) = \frac{\lambda_\x + \lambda_\y}{2}\left[\Tr(A^\top\cov_{\x\x}A\cov_{\y\y}) - \Tr(A\cov_{\x\y}) - \Tr(A^\top\cov_{\y\x})\right].$$
    Using general weighting $\lambda_\x$ and $\lambda_\y$  only results in a scaling of the  loss arising in Theorem~\ref{thm:Gaussian_solution_cosine} when $\lambda_\x = \lambda_\y = 1$. Thus, for any values $\lambda_\x,\lambda_\y$, $\LCOND(A;\lambda_\x,\lambda_\y)$ has the minimizer $A^\ast$ in~\eqref{eq:Gaussian_cond_cosine_minimizer_unconstrained} over all matrices of size $\dx \times \dy$ and the minimizer $A_r^\ast$ in~\eqref{eq:Gaussian_cond_cosine_minimizer_constrained} over rank-r matrices.
\end{remark}

\begin{remark}  \label{rem:G2} In the empirical setting defined by \eqref{eq:empirical}, the loss for matrix $A$ is given by
\begin{align*}
\LCONDN(A) &= \Tr(A^\top \widehat{\cov}_{\x\x}A\widehat{\cov}_{\y\y}) - \Tr(A\widehat{\cov}_{\x\y}) - \Tr(A^\top \widehat{\cov}_{\y\x}), \text{ where} \\
\widehat{\cov} &\coloneqq \frac{1}{N} \sum_{i=1}^N (\x^i,\y^i) \otimes (\x^i,\y^i).
\end{align*}
This has the same form as the population-level loss with the covariances replaced by their empirical counterparts. 
Moreover, assuming that the product of empirical covariances $\widehat\cov_{\x\x}^{-1}\widehat\cov_{\x\y}\widehat\cov_{\y\y}^{-1}$ is invertible, 
then the solution that minimizes $\LCONDN$ is also given by Theorem~\ref{thm:Gaussian_solution_cosine} with the covariances replaced by their empirical counterparts. 
\end{remark}

\begin{remark} \label{rem:G3}
After identifying $A,B$, the original encoders $G,H$ are identifiable without the rank constraint assuming that $\dx \leq \dy$. That is, $G \in \R^{\dx \times \dx}$ can be computed from a square root of $G^\top G = B^\ast$ given that $B^\ast$ is strictly positive definite. Moreover, given that $G$ is invertible, we can solve the equation $G^\top H = A^\ast$ for $H$. We note that this procedure does not define a unique solution for $G,H$ unless we seek the positive-definite square root of $B^\ast$.
\end{remark}

\section{Useful Identities}
\label{app:B}

\begin{lemma}[\cite{pardo2018statistical}] \label{lem:KLdivergence_Gaussians} 
Let $C_1, C_2 \in \R^{d \times d}$ be symmetric positive definite matrices and $m_1,m_2 \in \R^d$ be mean vectors. The Kullback-Leibler divergence from $\mathcal{N}(m_2,C_2)$ to $\mathcal{N}(m_1,C_1)$ is 
\begin{align*}
    \dkl \bigl(\mathcal{N}(m_1,C_1)||\mathcal{N}(m_2,C_2) \bigr) &= \frac{1}{2}\left[\Tr(C_2^{-1}C_1) - d + \log\frac{|C_2|}{|C_1|} + \|m_2 - m_1\|_{C_2}^2 \right]. %
\end{align*} 
\end{lemma}

\begin{lemma} \label{lem:Gaussian_expintegral} Let $\z \in \R^d$ follow a multivariate Gaussian distribution $\z \sim \pmeas = \mathcal{N}(m,\Lambda).$ Then, for any matrix $B \in \R^{d \times d}$ and vector $c \in \R^d$, we have
$$\mathbb{E}_{z \sim \nu}\left[\exp\left(\frac{1}{2} z^\top Bz + c^\top z \right)\right] = \frac{1}{\sqrt{|I - \Lambda B|}} \exp \left(\frac{1}{2}(c + \Lambda^{-1}m)^\top (\Lambda^{-1} - B)^{-1}(c + \Lambda^{-1}m) - m^\top \Lambda^{-1}m \right).$$
\end{lemma}

\begin{proof} Completing the square we have
\begin{align*}
\mathbb{E}_{\z \sim \nu}\left[\exp\left(\frac{1}{2} z^\top Bz + c^\top z \right)\right] &= \int \frac{1}{\sqrt{(2\pi)^d|\Lambda|}}\exp\left(-\frac{1}{2}(z - m)^\top \Lambda^{-1}(z - m)\right) \exp\left(\frac{1}{2} z^\top Bz + c^\top z \right) dz \\
&= \int \frac{1}{\sqrt{(2\pi)^d|\Lambda|}}\exp\left(-\frac{1}{2}(z - m^*)^\top (\Lambda^{-1} - B)(z - m^*)\right) \times \\
&\qquad \exp\left(\frac{1}{2} (m^*)^\top (\Lambda^{-1} - B)m^* - m^\top \Lambda^{-1} m \right) dz,
\end{align*}
where $m^*= (\Lambda^{-1} - B)^{-1}(c + \Lambda^{-1}m)$. Computing the normalizing constant in closed form and using the property for the matrix-determinant $|\Lambda| = 1/|\Lambda^{-1}|$ we have
\begin{align*}
\mathbb{E}_{\z \sim \nu}\left[\exp\left(\frac{1}{2} z^\top Bz + c^\top z \right)\right] &= \sqrt{\frac{|\Lambda^{-1}|}{|\Lambda^{-1} - B|}} \exp\left(\frac{1}{2} (m^*)^\top (\Lambda^{-1} - B)m^* - m^\top \Lambda^{-1} m \right). %
\end{align*}
Substituting the form for $m^*$ gives us the final result.
\end{proof}

\end{document}